 \documentclass[final,3p,times]{elsarticle}
\usepackage{amssymb}
\usepackage{amsmath}
\usepackage{amsthm}
\usepackage{graphicx}

\usepackage{algorithm}  
\usepackage{algpseudocode}  
\usepackage{algorithmicx} 
\usepackage{makecell}
\usepackage{booktabs}
\usepackage{multirow}
\usepackage{diagbox}
\theoremstyle{definition}
\newtheorem{theorem}{Theorem}
\newtheorem{definition}{Definition}
\newtheorem{lemma}{lemma}
\usepackage[colorlinks=true, allcolors=blue]{hyperref}
\usepackage{caption}
\usepackage{subcaption}
\usepackage{ulem}
\usepackage{indentfirst}
\usepackage{enumerate}

\usepackage[capitalize]{cleveref}
\usepackage[T1]{fontenc}
\usepackage{longtable}
\usepackage{color}

\usepackage{setspace}
\begin{document}

\begin{frontmatter}

\title{$L_{2,1}$-Norm Regularized Quaternion Matrix Completion Using Sparse Representation and Quaternion QR Decomposition}

\author[a]{Juan Han}
\author[a]{Kit Ian Kou\corref{Cor1}}
\ead{kikou@umac.mo}
\author[b]{Jifei Miao}
\author[c]{Lizhi Liu}
\author[c]{Haojiang Li}
\cortext[Cor1]{Corresponding author.}

\address[a]{{Department of Mathematics, Faculty of
		Science and Technology, University of Macau, Macau, 999078, China}}
\address[b]{School of Mathematics and Statistics, Yunnan University, Kunming, Yunnan, 650091, China}	
\address[c]{State Key Laboratory of Oncology in South China, Sun Yat-sen University Cancer Center, Guangzhou, Guangdong, 510060, China.}

\begin{abstract}
	Color image completion is a challenging problem in computer vision, but recent research has shown that quaternion representations of color images perform well in many areas. These representations consider the entire color image and effectively utilize coupling information between the three color channels. Consequently, low-rank quaternion matrix completion (LRQMC) algorithms have gained significant attention.  We propose a method based on quaternion Qatar Riyal decomposition (QQR) and quaternion $L_{2,1}$-norm called QLNM-QQR. This new approach reduces computational complexity by avoiding the need to calculate the QSVD of large quaternion matrices. We also present two improvements to the QLNM-QQR method: an enhanced version called IRQLNM-QQR that uses iteratively reweighted quaternion $L_{2,1}$-norm minimization and a method called QLNM-QQR-SR that integrates sparse regularization. Our experiments on natural color images and color medical images show that IRQLNM-QQR outperforms QLNM-QQR and that the proposed QLNM-QQR-SR method is superior to several state-of-the-art methods.
	
\end{abstract}


\end{frontmatter}

 \section{Introduction}
 Image completion, which endeavors to restore missing pixel values in an image from limited available data, has a broad spectrum of applications in computer vision and has garnered considerable research attention in recent years \cite{liu2015truncated, gu2017weighted,yang2020feature, miao2021color}. Of all the techniques for completing images, those based on low-rank matrix completion (LRMC) have attained significant success. The majority of LRMC-based models fall into matrix rank minimization. As for matrix rank minimization, the low-rank property of matrices is used by LRMC-based models to primarily formulate a problem for minimizing matrices with low rank. Hence, the problem is usually converted into solving a constrained rank optimization problem by using the rank function. 

 Despite its effectiveness, the rank function poses a challenge due to its discontinuity and non-convexity, rendering it NP-hard. However, the nuclear norm has been demonstrated to be the most rigorous convex relaxation of the rank minimization problem \cite{Candes2009Exact}. Consequently, the nuclear norm, serving as a convex surrogate of the rank function, has been extensively employed to tackle image completion predicaments. Even so, the prevailing nuclear norm (NN) minimization methods, as noted by the authors in \cite{hu2012fast}, tend to minimize all singular values simultaneously, leading to inadequate rank approximations. To remedy this, they proposed the truncated nuclear norm (TNN) regularization, which furnishes a more precise approximation of the rank function than the nuclear norm. Other comparable alternatives to TNN include the weighted nuclear norm (WNNM) \cite{Gu2014CVPR, gu2017weighted}, the weighted Schatten $p$-norm \cite{xie2016weighted}, and the log-determinant penalty \cite{kang2015logdet}. These methods also focus on optimizing the rank approximation, which entails processing the complete singular value decomposition (SVD). However, this can be computationally expensive and presents limitations for high-dimensional or big data applications. When it comes to matrix factorization, the general approach is to break down the original, larger matrix into at least two smaller matrices, often resulting in quicker numerical optimization \cite{wen2012solving, shang2017bilinear}. However, this factorization method can become enmeshed in local minima.

Besides, in the processing of color images using the aforementioned LRMC-based models, the RGB's three color channels are typically processed separately and then combined to obtain the final restoration outcome. However, this approach can result in a loss of coupling information between the three channels due to the disregard of the inter-channel relationships. Models that can more effectively use the connection between three channels are thus worth researching. 

In recent years, the use of quaternion representation for color images has received significant attention from researchers. Various studies have demonstrated that quaternions can effectively describe color images, and approaches built upon quaternion representation have demonstrated competitiveness in addressing a range of image processing issues, such as foreground/background separation of color videos \cite{han2022DMD}, color image denoising \cite{yu2019quaternion}, edge detection \cite{Hu2018Phase}, face recognition \cite{zou2016quaternion}, and completion \cite{miao2021color}.
Specifically, the three color channels of a color image precisely correspond to the three imaginary sections of a quaternion matrix. We can express a pixel of a color image as a pure quaternion in the following manner:
\[ \dot p = 0+p_R \,  i+p_G \,  j +p_B \,   k,\]
where $p_R$, $p_G$, and $p_B$ correspond to the pixel values of the three channels (RGB) of a color pixel, respectively, and $i$, $j$, and $k$ stand for the three imaginary units of a quaternion.
Utilizing the pure quaternion matrix to represent the three color channels of a color image can better exploit the relationship between the channels in image processing. More recently, there have been proposed approaches that utilize quaternion-based LRMC for color image completion. 
A general approach is proposed in \cite{chen2019low} for low-rank quaternion matrix completion (LRQMC) that employs the quaternion nuclear norm (QNN) and three nonconvex rank surrogates. These surrogates rely on the Laplace function, Geman function, and weighted Schatten norm. These methods have demonstrated superior performance compared to some popular LRMC-based methods. However, they still require solving the singular values of a large quaternion matrix, which can significantly increase the algorithm's computational complexity. Additionally, a logarithmic norm-based quaternion completion algorithm is proposed \cite{yang2022logarithmic}. \cite{miao2020quaternion} proposed three minimization models based on quaternion-based bilinear decomposition. These approaches only require dealing with smaller-sized two-factor quaternion matrices, which reduces the computational complexity of solving quaternion singular value decomposition (QSVD). However, this factorization could result in getting trapped in a local minimum.

Indeed, the singular values and vectors can be obtained using the Qatar Riyal (QR) decomposition, which has lower algorithmic complexity than SVD \cite{liu2018fast}. Similarly, a method based on the quaternion QR (QQR) decomposition to approximate the QSVD of quaternion matrices (CQSVD-QQR) is also proposed \cite{han2022low}. Additionally, the $L_{2,1}$-norm has recently shown success in various applications such as feature selection \cite{hou2013joint} and low-rank representation \cite{tang2014structure, xiao2015robust, liu2018fast}. The study in \cite{nie2014optimal} shows that the $L_{2,1}$-norm provides better robustness against outliers.
In low-rank representation, the $L_{2,1}$-norm can be used to remove outliers by solving a minimization problem, which does not require the use of SVD to obtain the optimal solution \cite{liu2018fast}. However, currently, there is no quaternion $L_{2,1}$-norm-based quaternion matrix completion method available.
This inspired us to establish a quaternion-based color image completion model by introducing the quaternion $L_{2,1}$-norm based on the quaternion Tri-Factorization method (CQSVD-QQR) so that the calculation of QSVD is not required during the model solving process. 

In order to improve image completion results, it is essential to consider more information beyond just the low-rank qualities of images \cite{yang2010image, liu2018fast}. One crucial factor is the sparsity of images in a particular domain, such as transform domains, where many signals have a naturally sparse structure. To address this, researchers have proposed combining low-rank and sparse priors. In \cite{dong2018low}, to formulate the sparse prior, the authors use the $l_1$-norm regularizer, while the truncated nuclear norm is selected as the surrogate for the rank function. To extend this method to the quaternion system, the authors in \cite{yang2022quaternion} adapted it to work with quaternion matrices. Although combining low-rank and sparse prior has proven effective for improving image completion, these methods still depend on the SVD/QSVD calculation of large matrices, making them computationally expensive. We were inspired to improve the quaternion Tri-Factorization method (CQSVD-QQR)-based quaternion $L_{2,1}$-norm minimization model by introducing the sparse regularization, which led to enhancing the precision of image completion results. The following points summarize the key contributions of this paper:
\begin{itemize}
	\item A novel approach called QLNM-QQR, based on quaternion $L_{2,1}$-norm and CQSVD-QQR, is introduced for completing quaternion data. The computational complexity of QLNM-QQR is reduced compared to methods that use QSVD by replacing QSVD with QQR and the quaternion nuclear norm with the quaternion $L_{2,1}$-norm.	
	\item An improved method for quaternion matrix completion called IRQLNM-QQR is proposed by introducing iteratively reweighted quaternion $L_{2,1}$-norm minimization. The proposed method enhances the accuracy of QLNM-QQR. Theoretical analysis demonstrates that IRQLNM-QQR achieves the same optimal solution as an LRQA-based method using the weighted Schatten function as the nonconvex rank surrogate, which outperforms the traditional QQR decomposition-based methods in terms of accuracy.
	\item To enhance the precision of color image recovery, we have integrated sparsity into the QLNM-QQR method, resulting in an improved method called QLNM-QQR-SR.
	\item We have proved that the quaternion $L_{2,1}$-norm of a quaternion matrix serves as an upper bound for its quaternion nuclear norm. As a result, the methods proposed in this study can be extended to enhance the performance of low-rank representation and quaternion matrix completion methods based on the quaternion nuclear norm.
\end{itemize}

\indent
The structure of this paper is outlined as follows: Section 2 reviews the related models for completing matrices/quaternion matrices based on the low-rank property of data. Section 3 presents commonly used mathematical notations and provides a brief introduction to quaternion algebra. Section 4 gives a brief overview of the CQSVD-QQR technique, followed by the introduction of three proposed quaternion matrix-based completion algorithms. The computational complexities of the proposed models are also discussed. In Section 5, we provide experimental results and compare our proposed methods with several state-of-the-art approaches. Finally, Section 6 presents our conclusions.

\section{Related work}
A natural image has recurring patterns that may be utilized to estimate the missing values, which is the foundation for almost all of the existing image completion techniques \cite{zhang2019nonlocal, fan2017matrix}. Consider an incomplete matrix $\mathbf{M} \in \mathbb{R}^{M \times N}, \ M \geq N >0$. Following are the formulation for the conventional LRMC-based technique:  
 \begin{equation}
 	\mathop{\text{min}}\limits_{\mathbf{X}}\text{rank}(\mathbf{X}),   \: \text{s.t.}, \ {P}_{\Omega}(\mathbf{X}-\mathbf{M})=\mathbf{0}, 
 	\label{LRMC_TRA}
 \end{equation}
 where rank$(\cdot)$ signifies the rank function, $\mathbf{X}\in \mathbb{R}^{M \times N}$ represents a restored matrix, and $\Omega$ represents the set of observed entries.
 The definition of ${P}_{\Omega}(\mathbf{X})$ is given by
 \[({P}_{\Omega}(\mathbf{X}))_{mn}=\begin{cases}
 	\mathbf{X}_{mn}, \ (m,n)\in \Omega,\\
 	0, \  \    \quad   \text{otherwise}.
 \end{cases}
 \]
An effective method for solving the combinatorial optimization problem (\ref{LRMC_TRA}) is to optimize a convex substitute for the rank function. However, it should be noted that problem (\ref{LRMC_TRA})-centric optimization algorithms mainly deal directly with two-dimensional data sets, specifically grayscale images. The processing of color images generally requires the RGB channels to be separated before the model in problem (\ref{LRMC_TRA}) can operate on them, while the LRQMC-based models can directly work with the assembled RGB channels. Analogous to problem (\ref{LRMC_TRA}) mentioned above, the prevalent low-rank quaternion matrix-based completion algorithm (LRQMC) can be expressed as follows:
\begin{equation}
	\mathop{\text{min}}\limits_{\dot{\mathbf{X}}} \text{rank}(\dot{\mathbf{X}}),   \: \text{s.t.}, \ {P}_{\Omega}(\dot{\mathbf{X}}-\dot{\mathbf{M}})=\mathbf{0}, 
	\label{LRQMC}
\end{equation}
where $\dot{\mathbf{X}}\in \mathbb{H}^{M \times N}$ denotes the resulting quaternion matrix that has been completed, $\dot{\mathbf{M}}\in \mathbb{H}^{M \times N}$ is used to signify the observed quaternion matrix.
The definition of ${P}_{\Omega}(\dot{\mathbf{X}})$ is given by
\[({P}_{\Omega}(\dot{\mathbf{X}}))_{mn}=\begin{cases}
	\dot{\mathbf{X}}_{mn}, \ (m,n)\in \Omega,\\
	0, \  \    \quad  \text{otherwise}.
\end{cases}
\]
A commonly used technique for resolving the above minimization problem (\ref{LRQMC}) is to adopt the QNN as a convex surrogate for the rank function. Subsequently, a minimization method can be formulated for this surrogate function as
\begin{equation}
	\mathop{\text{min}}\limits_{\dot{\mathbf{X}}} \|\dot{\mathbf{X}}\|_{\ast},   \: \text{s.t.}, \ {P}_{\Omega}(\dot{\mathbf{X}}-\dot{\mathbf{M}})=\mathbf{0}, 
	\label{LRQMC_QNN}
\end{equation}
where $\|\dot{\mathbf{X}}\|_{\ast}$ is the QNN of $\dot{\mathbf{X}}$.\\
\indent
Inspired by the favorable outcomes of nonconvex surrogates in LRMC, Chen $\it{et}$ $ \it{al}$. \cite{chen2019low} present a universal LRQMC model as follows:
\begin{equation}
	\mathop{\text{min}}\limits_{\dot{\mathbf{X}}} \sum_{l}\phi(\sigma_{l}(\dot{\mathbf{X}}),\gamma),  \: \text{s.t.}, \ {P}_{\Omega}(\dot{\mathbf{X}}-\dot{\mathbf{M}})=\mathbf{0}, 
	\label{LRQMC_NNM}
\end{equation}
where the nonnegative $\gamma$ is associated with the particular function $\phi(\cdot)$. Furthermore, their proposed approaches are based on three nonconvex rank surrogates that utilize the Laplace, Geman, and Weighted Schatten-$\gamma$ functions, respectively. And 
\begin{equation}
\|\dot{\mathbf{X}}\|_{L,\gamma}=\sum_{l}(1-e^{\frac{-\sigma_{l}(\dot{\mathbf{X}})}{\gamma}});
\end{equation}
\begin{equation}
	\|\dot{\mathbf{X}}\|_{G,\gamma}=\sum_{l}(1-e^{\frac{(1+\gamma)\sigma_{l}(\dot{\mathbf{X}})}{\gamma+\sigma_{l}(\dot{\mathbf{X}})}});
\end{equation}
\begin{equation}
	\|\dot{\mathbf{X}}\|_{W,\gamma}=\sum_{l}\omega_{l}\sigma_{l}^{\gamma}(\dot{\mathbf{X}}).
\end{equation}
As for $\|\dot{\mathbf{X}}\|_{W,\gamma}$, the contribution of the $l$-th singular value of $\dot{\mathbf{X}}$ ($\sigma_{l}(\dot{\mathbf{X}})$) to the quaternion rank is balanced using a non-negative weight scalar $\omega_{l}$.  
LRQA-N, LRQA-L, LRQA-G, and LRQA-W are the four low-rank quaternion approximation (LRQA) methods proposed in \cite{chen2019low}, employing the nuclear norm, Laplace function, Geman function, and weighted Schatten norm, respectively. To gain a deeper insight, please consult \cite{chen2019low}.

\section{Notations and preliminaries}
The primary mathematical notations used in this paper are initially introduced in this part, and then some basic concepts and theorems in the quaternion system are provided. For a more in-depth understanding of quaternion algebra, we recommend reading \cite{rodman2014topics}.

\subsection{Notations}
$\mathbb{R}$, $\mathbb{C}$, and $\mathbb{H}$ present the real space, complex space, and quaternion space, respectively. Lowercase letters, boldface lowercase letters, and boldface capital letters signify scalar, vectors, and matrices, respectively. A quaternion scalar, vector, and matrix are denoted by $\dot{a}$, $\dot{\mathbf{a}}$, $\dot{\mathbf{A}}$. $(\cdot)^{T}$, $(\cdot)^{\ast}$, and $(\cdot)^{H}$, respectively, signify the transpose, conjugation, and conjugate transpose. $|\cdot|$, $\|\cdot\|_{1}$, $\|\cdot\|_{F}$, and $\|\cdot\|_{\ast}$ correspond to the absolute value or modulus, the $l_{1}$-norm, the Frobenius norm, and the nuclear norm, respectively. The trace and rank operators are denoted by $\text{tr}\{\cdot\}$ and $\text{rank}(\cdot)$, respectively. $\mathbf{I}_{m} \in \mathbb{R}^{m \times m}$ is the identity matrix. $\mathfrak{R}(\cdot)$ signifies its real part for a quaternion scalar, vector, or matrix.
\subsection{The fundamentals of quaternion algebra}
Quaternion is invented by Hamilton in 1843 \cite{hamilton1844ii}. The definition of quaternions is 
	\begin{equation}
\mathbb{H}=\{q_0+q_1 i +q_2 j+q_3 k|q_0, q_1, q_2, q_3 \in \mathbb{R}\} \label{H}
\end{equation}
and $\mathbb{H}$ is an algebra that fulfills the associative law but not the commutative law. $i,j,k$
in (\ref{H}) represent imaginary units and satisfy $i^2 =j^2=k^2=ijk=-1$, and thus $ij=-ji=k, jk=-kj=i, ki=-ik=j$. Regarding any quaternion 
\[\dot{q} = q_0+q_1 i +q_2 j+q_3 k=\mathfrak{R}(\dot{q}) + \mathfrak{J}(\dot{q}),\]
$\mathfrak{R}(\dot{q})= q_0 \in \mathbb{R}$ is its real part and $\mathfrak{J}(\dot{q}) = q_1 i +q_2 j+q_3 k$ is its vector part. Additionally, $\dot{q}$ is referred to be a pure quaternion if $\mathfrak{R}(\dot{q}) =0$.
%$\mathbf{V}(\mathbb{H})$ stands for the collection of pure quaternions.
Obviously, in general, $\dot{p} \dot{q} \neq \dot{q} \dot{p}$ since the commutative property of multiplication does not apply to the quaternion algebra. 
Given $\dot{q}$, its conjugate is defined as 
$\dot{q}^{\ast}=q_0-q_1 i -q_2 j-q_3 k$, and its norm is calculated by
$\lvert{\dot{q}} \rvert =\sqrt{q^{\ast}q}=\sqrt{qq^{\ast}}=\sqrt{q_0^2+q_1^2+q_2^2+q_3^2}.$\\
\indent
For any quaternion matrix $\dot{\mathbf Q}=(\dot q_{mn}) \in \mathbb{H}^{M \times N}$, it can be expressed as $\dot{\mathbf Q} = \mathbf{ Q}_0 +\mathbf{ Q}_1 i +\mathbf{ Q}_2 j+\mathbf{ Q}_3 k$, where $\mathbf{ Q}_s \in  \mathbb{R}^{M \times N} (s=0, \, 1, \, 2, \, 3)$. If $\mathfrak{R}(\dot{\mathbf Q})=\mathbf{ Q}_0=\mathbf{0}$, a quaternion matrix is referred to as a pure quaternion matrix. 
The following is the definition of a quaternion matrix $\dot{\mathbf Q}$'s Frobenius norm: 
\[\left\| \dot{\mathbf {Q}} \right\|_{F}=\sqrt {\sum_{m=1}^{M}\sum_{n=1}^{N}\lvert {\dot {q}}_{mn}\rvert ^{2}} = \sqrt{\text{tr}(\dot{\mathbf {Q}}^{H}\dot{\mathbf {Q})}}.\]
 $\|\dot{\mathbf{Q}}\|_{\ast}$ denotes the nuclear norm of $\dot{\mathbf{Q}}$, and $\|\dot{\mathbf{Q}}\|_{\ast}=\sum_{s}\sigma_{s}(\dot{\mathbf{Q}})$, where $\sigma_{s}(\dot{\mathbf{Q}})$ is the $s$-th nonzero singular value of $\dot{\mathbf{Q}}$. The $l_{1}$-norm of $\dot{\mathbf{Q}}$ is defined as $\|\dot{\mathbf{Q}}\|_{1}=\sum_{m=1}^{M}\sum_{n=1}^{N}|\dot{q}_{mn}|$.

\begin{definition}[Cayley-Dickson form \cite{schafer1954algebras} and Equivalent complex matrix of a quaternion matrix \cite{ZHANG199721}] \label{Complex representation}  For any quaternion matrix $\dot{\mathbf{Q}}= \mathbf{ Q}_0 +\mathbf{ Q}_1 i +\mathbf{ Q}_2 j+\mathbf{ Q}_3 k \in \mathbb{H}^{M \times N}$, its Cayley-Dickson form is given by $\dot{\mathbf{Q}}=\mathbf{Q}_a+\mathbf{Q}_bj$, where $\mathbf{Q}_a=\mathbf{ Q}_0 +\mathbf{ Q}_1 i$, $\mathbf{Q}_b=\mathbf{ Q}_2 +\mathbf{ Q}_3 i \in \mathbb{C}^{M \times N}$. And the following definition applies to its equivalent complex matrix ${\chi_{\dot{\mathbf Q}}} \in \mathbb{C}^{2M \times 2N}$:
	\begin{equation}
		{\chi_{\dot{\mathbf Q}}} =
		\begin{bmatrix}
			\mathbf{Q_a}  & \mathbf{Q_b} \\
			-{\mathbf{Q_b}}^\ast & {\mathbf{Q_a}}^\ast
		\end{bmatrix}.
	\end{equation}
\end{definition}
There are many similarities between the quaternion matrix and its equivalent complex matrix.
To learn more, we advise reading \cite{ZHANG199721}.

\noindent
\begin{theorem}[Qauternion singular value decomposition (QSVD) \cite{ZHANG199721}] \label{QSVD}  For every quaternion matrix $\dot{\mathbf{A}} \in \mathbb{H}^{M \times N}$ of rank $r$, two unitary quaternion matrices $\dot{\mathbf{U}} \in \mathbb{H}^{M \times M}$ and $\dot{\mathbf{V}} \in \mathbb{H}^{N \times N}$ exist, such that
	\begin{equation}
		\dot{\mathbf{A}} = \dot{\mathbf{U}}\begin{bmatrix}
			\mathbf{\Sigma}_r  & \mathbf{0} \\
			\mathbf{0} & \mathbf{0}
		\end{bmatrix}\dot{\mathbf{V}}^H=\dot{\mathbf{U}}\mathbf{\Lambda} \dot{\mathbf{V}}^H,\end{equation}
	where $\mathbf{\Sigma}_r$ is a real diagonal matrix with $r$ positive singular values of $\dot{\mathbf{A}}$ on its diagonal.
\end{theorem}
\noindent
\begin{theorem}[The quaternion Qatar Riyal decomposition (QQR) \cite{Wei2018QuaMatComs}] \label{QQR} 
	Considering an arbitrary quaternion matrix $\dot{\mathbf{A}} \in \mathbb{H}^{M \times N}$ with a rank of $r$, one can find a unitary quaternion matrix  $\dot{\mathbf{Q}} \in \mathbb{H}^{M \times M}$ and a weakly upper triangular quaternion matrix $\dot{\mathbf{R}} \in \mathbb{H}^{M \times N}$, such that 
	\begin{equation}
		\dot{\mathbf{A}} = \dot{\mathbf{Q}} \dot{\mathbf{R}}.
	\end{equation}
In other words, there exists a permutation matrix $\mathbf{P}\in \mathbb{R}^{N \times N}$ such that the product $\dot{\mathbf{R}}\mathbf{P}$ forms an upper triangular quaternion matrix.
\end{theorem}

\begin{lemma}[\cite{chen2019low}] \label{QSVT}
	For any $\mu>0$, $\dot{\mathbf{Y}} \in \mathbb{H}^{M\times N}$ is a given quaternion matrix and the QSVD of $\dot{\mathbf{Y}}$ is given by $\dot{\mathbf{Y}}=\dot{\mathbf{U}}{\Sigma}\dot{\mathbf{V}}^{H}$. For the following quaternion nuclear norm minimization problem (QNNM)
	\begin{equation}
		\mathop{\text{min}}\limits_{\dot{\mathbf{X}}}\mu\|\dot{\mathbf{X}}\|_{\ast}+\frac{1}{2}\|\dot{\mathbf{Y}}-\dot{\mathbf{X}}\|_{F}^{2},  
		\label{QSVT_pro}
	\end{equation}
	the closed-form solution ${\dot{\widetilde{\mathbf{X}}}}$ is given by
		\begin{equation}
		{\dot{\widetilde{\mathbf{X}}}}=\dot{\mathbf{U}}{\mathit{S}_{\mu}(\Sigma)}\dot{\mathbf{V}}^{H},  
		\label{QSVT_solu}
	\end{equation}
	where $\mathit{S}_{\mu}(\Sigma)=\text{diag}(\text{max}\{\sigma_s(\dot{\mathbf{Y}})-\mu,0\})$ is the soft thresholding operator.
\end{lemma}

\begin{lemma}[\cite{chen2019low}] \label{LRQA-W} Given a positive value of $\mu$, $\dot{\mathbf{Y}} \in \mathbb{H}^{M\times N}$ is a known quaternion matrix, and its QSVD is expressed as $\dot{\mathbf{Y}}=\dot{\mathbf{U}}{\Sigma}\dot{\mathbf{V}}^{H}$. 
	The optimal solution for ${\dot{\widetilde{\mathbf{X}}}}$ in the following problem 
	\begin{equation}
		\mathop{\text{min}}\limits_{\dot{\mathbf{X}}}\mu\sum_{l}\omega_{l}\sigma_{l}(\dot{\mathbf{X}})+\frac{1}{2}\|\dot{\mathbf{Y}}-\dot{\mathbf{X}}\|_{F}^{2}
		\label{LRQA-W_pro}
	\end{equation}
 is given by
	\begin{equation}
		{\dot{\widetilde{\mathbf{X}}}}=\dot{\mathbf{U}}\Sigma_{\dot{\mathbf{X}}}\dot{\mathbf{V}}^{H},  
		\label{LRQA-W_solu}
	\end{equation}
	where $\Sigma_{\dot{\mathbf{X}}}=\text{diag}(\sigma^{\ast})$ and $\sigma^{\ast}$ is given by
		\begin{equation}
		\sigma^{\ast}=\mathop{\text{arg min}}\limits_{\sigma\geq0}\mu\sum_{l}\omega_{l}\sigma_{l}(\dot{\mathbf{X}})+\frac{1}{2}\|\sigma-\sigma_{\dot{\mathbf{Y}}}\|_{F}^{2}.
		\label{LRQA-W_subpro}
	\end{equation}
\end{lemma}

Presented below is a succinct overview of the Quaternion Discrete Cosine Transform.
\begin{definition}[Forward quaternion discrete cosine transform (FQDCT) \cite{feng2008quaternion}] \label{FQDCT} Given a quaternion matrix $\dot{\mathbf{A}} \in \mathbb{H}^{M \times N}$, as quaternions are non-commutative, its quaternion discrete cosine transform (QDCT) exists in two distinct types, namely, the left-handed and right-handed forms as follows:
	\begin{equation}
		\mathop {FQDCT}\nolimits ^{L} (s,t) = \psi(s)\psi (t)\sum \limits _{m =0}^{M - 1} {\sum \limits _{n = 0}^{N - 1} {\dot{q}\, \dot{\mathbf{A}} \left ({{m,n} }\right)Q\left ({{s,t,m,n} }\right)} },  
	\end{equation}
	\begin{equation}
		\mathop {FQDCT}\nolimits ^{R} (s,t) = \psi(s)\psi (t)\sum \limits _{m =0}^{M - 1} {\sum \limits _{n = 0}^{N - 1} {\dot{\mathbf{A}} \left ({{m,n} }\right)Q\left ({{s,t,m,n} }\right)\,\dot{q} } },
	\end{equation}
	where the pure quaternion $\dot{q}$, which fulfills the condition $\dot{q}^2=-1$, is called the quaternionization factor.
	The values of $\psi(s)$, $\psi(t)$, and $Q\left ({s,t,m,n} \right)$ in the QDCT are similar to those in the discrete cosine transform (DCT) in the real domain, which are given as follows:
	\begin{equation}
		\psi(s) = \begin{cases}
			\sqrt{\frac{1}{M}}\quad {\text {for}}\quad s=0\cr \sqrt{\frac{2}{M}}\quad {\text {for}}\quad s\neq 0
		\end{cases}, \quad \psi(t) = \begin{cases}
			\sqrt{\frac{1}{N}}\quad {\text {for}}\quad t=0\cr \sqrt{\frac{2}{N}}\quad {\text {for}}\quad t\neq 0
		\end{cases},
	\end{equation}
	\begin{equation}
		Q\left ({s,t,m,n} \right) = \cos\left[{\frac{\pi (2m+1)s}{2M} }\right]\cos\left[{\frac{\pi (2n+1)t}{2N} }\right].
	\end{equation}
\end{definition}
The two types of inverse QDCT (IQDCT), following the two forms of the FQDCT, are given below:
\begin{equation} \mathop {IQDCT}\nolimits ^{L}(m,n)=\sum \limits _{s=0}^{M-1}{\sum \limits _{t=0}^{N-1} {\psi (s)\psi (t)\,{\dot{q}}\, \dot{\mathbf B}\left ({{s,t}}\right)Q\left ({{s,t,m,n}}\right)} },
\end{equation}
\begin{equation} \mathop {IQDCT}\nolimits ^{R}(m,n)=\sum \limits _{s=0}^{M-1}{\sum \limits _{t=0}^{N-1} {\psi (s)\psi (t)\dot{\mathbf B}\left ({{s,t}}\right)Q\left ({{s,t,m,n}}\right)}\, {\dot{q}}},
\end{equation}
where $\dot{\mathbf{B}} \in \mathbb{H}^{M \times N}$.

\begin{theorem}[The relationship between FQDCT and IQDCT \cite{feng2008quaternion}] 
	\begin{equation}
		\begin{aligned}
			\dot{\mathbf{A}}(m,n) &= \mathop {IQDCT}\nolimits ^{L}\left[\mathop {FQDCT}\nolimits ^{L}(\dot{\mathbf{A}}(m,n))\right]\\
			&=\mathop {IQDCT}\nolimits ^{R}\left[\mathop {FQDCT}\nolimits ^{R}(\dot{\mathbf{A}}(m,n))\right].
		\end{aligned}
	\end{equation}
\end{theorem}
Since the construction of our proposed model necessitates the use of $\text{FQDCT}^{L}$, we outline the computational steps of $\text{FQDCT}^{L}$ in the following.
\begin{enumerate}
	\item Express the quaternion matrix $\dot {\mathbf{A}}(m,n) \in \mathbb{H}^{M \times N}$ in its Cayley-Dickson form, that is, 
	$\dot {\mathbf{A}}(m,n)=\mathbf{A}_{a}(m,n)+\mathbf{A}_{b}(m,n)j$, where $\mathbf{A}_{a}(m,n)$ and $\mathbf{A}_{b}(m,n)\in \mathbb{C}^{M \times N}$.
	\item Compute the DCT of $\mathbf{A}_{a}(m,n)$ and $\mathbf{A}_{b}(m,n)$, denoting the resulting matrices as $\mathop {DCT}(\mathbf{A}_{a}(m,n))$, and $\mathop {DCT}(\mathbf{A}_{b}(m,n))$, respectively. 
	\item Using $\mathop {DCT}(\mathbf{A}_{a}(m,n))$ and $\mathop {DCT}(\mathbf{A}_{b}(m,n))$, construct a quaternion matrix $\dot{\hat{{\mathbf{A}}}} (m,n)$ as follows: $\dot{\hat{{\mathbf{A}}}} (m,n)=\mathop {DCT}(\mathbf{A}_{a}(m,n))+\mathop {DCT}(\mathbf{A}_{b}(m,n)) j$.
	\item Multiply $\dot{\hat{{\mathbf{A}}}} (m,n)$ by the quaternization factor $\dot{q}$ to obtain the final result:
	\[\mathop {FQDCT}\nolimits ^{L}\left[\dot {\mathbf{A}}(m,n)\right]=\dot{q}\, \dot{\hat{{\mathbf{A}}}}(m,n).\]
\end{enumerate}

\section{Our method of quaternion completion}
After providing a brief overview of the CQSVD-QQR technique, this section introduces three novel algorithms for quaternion matrix-based completion. Furthermore, the computational complexities of these models are discussed.
\subsection{A method for quaternion matrix decomposition }
In the event that $\dot{\mathbf{X}} \in \mathbb{H}^{M \times N}$ is a quaternion matrix, the following quaternion Tri-Factorization is used to compute $\dot{\mathbf{X}}$'s approximation QSVD based on QQR decomposition (CQSVD-QQR) \cite{han2022low}: 
\begin{equation}
\dot{\mathbf{X}}=\dot{\mathbf{L}} \dot{\mathbf{D}} \dot{\mathbf{R}},
\end{equation}
where the quaternion matrices $\dot{\mathbf{L}} \in \mathbb{H}^{M \times r}$ and $\dot{\mathbf{R}}\in \mathbb{H}^{r \times N}$ satisfy $\dot{\mathbf{L}}^{H} \dot{\mathbf{L}}=\mathbf{I}_r, \ \dot{\mathbf{R}} \dot{\mathbf{R}}^{H}=\mathbf{I}_r$, the quaternion matrix $\dot{\mathbf{D}}\in \mathbb{H}^{r \times r}$ is lower triangular, and $|D_{ss}|=|\sigma_{s}(\dot{\mathbf{X}})|$. The CQSVD-QQR procedure is briefly described below. For more information, please see \cite{han2022low}. 
First, suppose that $\dot{\mathbf{L}}^{1}=\text{eye}(M,r)$, $\dot{\mathbf{D}}^{1}=\text{eye}(r,r)$, and $\dot{\mathbf{R}}^{1}=\text{eye}(r,N)$.\\
\indent
Next, in the $\tau+1$-th iteration, $\dot{\mathbf{L}}^{\tau+1}$ is determined by 
\begin{equation}
	\left[\dot{\mathbf{Q}}, \dot{\mathbf{G}}\right]=\text{qqr}(\dot{\mathbf{X}} (\dot{\mathbf{R}}^{\tau })^{H}),   \label{L_QR}
\end{equation}
\begin{equation}
	\dot{\mathbf{L}}^{\tau+1} =\dot{\mathbf{Q}} (:,1:r),   \label{L_QRr}
\end{equation}
where the operator qqr($\cdot$) calculates the quaternion QR decomposition of a given quaternion matrix, and $\dot{\mathbf{Q}} (:,1:r)$ means to extract the first $r$ columns of $\dot{\mathbf{Q}}$.
$\dot{\mathbf{R}}^{\tau+1}$ is provided by 
\begin{equation}
	\left[\dot{\mathbf{T}}, \dot{\mathbf{S}}\right]=\text{qqr}(\dot{\mathbf{X}}^{H}\mathbf{L}^{\tau+1}),   \label{R_QR}
\end{equation}
\begin{equation}
	\dot{\mathbf{R}}^{\tau+1} =(\dot{\mathbf{T}} (:,1:r))^{H}.   \label{R_QRr}
\end{equation}
The last update to $\dot{\mathbf{D}}^{\tau+1}$ is made by
\begin{equation}
	\dot{\mathbf{D}}^{\tau+1} =(\dot{\mathbf{S}}(1:r,1:r))^{H} ,   \label{D_QRr}
\end{equation}

\subsection{Proposed three quaternion $L_{2,1}$-norm-based methods for color image completion }
The formulation of the proposed model is presented in this section. 
The definition of quaternion $L_{2,1}$-norm of $\dot{\mathbf{X}} \in \mathbb{H}^{M \times N}$ is given by:
\begin{equation}
	\|\dot{\mathbf{X}}\|_{2,1}=\sum_{n=1}^{N}\sqrt{\sum_{m=1}^{M}|\dot{\mathbf{X}}_{mn}|^{2}}.
\end{equation}
We have verified that the three norm requirements are met by this valid norm, even for the triangle inequality. \\
For the quaternion $L_{2,1}$-norm, we can obtain the following theorem.
\begin{theorem}[] \label{L21bound} Using the method CQSVD-QQR, a quaternion matrix $\dot{\mathbf{X}} \in \mathbb{H}^{M \times N}$ can be decomposed into $\dot{\mathbf{X}} =\dot{\mathbf{L}}\dot{\mathbf{D}}\dot{\mathbf{R}}$. And for $\|\dot{\mathbf{D}}\|_{\ast}$ and $\|\dot{\mathbf{D}}\|_{2,1}$, it can be found that $\|\dot{\mathbf{D}}\|_{\ast} \leq \|\dot{\mathbf{D}}\|_{2,1}$.
\end{theorem}
 \begin{proof}
 	The following is a decomposition of the quaternion matrix $\dot{\mathbf{D}} \in \mathbb{H}^{r \times r}$: 
 	\begin{equation}
 		\dot{\mathbf{D}} =\sum_{l=1}^{r}\dot{\mathbf{D}}^{l},  \label{D_L21bound}
 	\end{equation}
	\begin{equation}
	\dot{\mathbf{D}}_{ts}^{l} ==\begin{cases}
			\dot{\mathbf{D}}_{tl}, \ (s=l),\\
			0,      \quad   (s\neq l),
		\end{cases}  
\end{equation}
 where $\dot{\mathbf{D}}^{l} \in \mathbb{H}^{r \times r}$, and $t$, $s=1, \dots, r$. Due to the convex nature of the quaternion nuclear norm, we can obtain that
 	\begin{equation}
 	\|\dot{\mathbf{D}}\|_{\ast} \leqslant\ \sum_{l=1}^{r}\|\dot{\mathbf{D}}^{l}\|_{\ast},  
 \end{equation}
By calculating the singular values of the equivalent complex matrix of $\dot{\mathbf{D}}^{l}$ and according to the definition of the QNN, we can get 
	\begin{equation}
	\|\dot{\mathbf{D}}^{l}\|_{\ast} = \sqrt{\sum_{t=1}^{r}|\dot{\mathbf{D}}_{tl}|^{2}}. \label{nuclear_sigular}
\end{equation}
According to (\ref{nuclear_sigular}), we have 
 		\begin{equation}
 		\sum_{l=1}^{r}\|\dot{\mathbf{D}}^{l}\|_{\ast} = \|\dot{\mathbf{D}}\|_{2,1}. \label{l21_nuclear}
 	\end{equation}
 Therefore, we proved the above conclusion.
 \end{proof}

\textit{1) The proposed CQSVD-QQR-based quaternion $L_{2,1}$-norm minimization approach}: The quaternion $L_{2,1}$-norm of a quaternion matrix is unquestionably the upper bound of its QNN, as shown by Theorem (\ref{L21bound}). This result inspires us to replace the QNN in the minimization problem of (\ref{LRQMC_QNN}) with the quaternion $L_{2,1}$-norm for quaternion matrix completion in the manner described below:
\begin{equation}
	\mathop{\text{min}}\limits_{\dot{\mathbf{D}}}\|\dot{\mathbf{D}}\|_{2,1},   \: \text{s.t.}, 
	\begin{cases}
		\dot{\mathbf{L}}^{H} \dot{\mathbf{L}}=\mathbf{I}_r, \ \dot{\mathbf{R}} \dot{\mathbf{R}}^{H}=\mathbf{I}_r,\\
		\dot{\mathbf{X}}=\dot{\mathbf{L}}\dot{\mathbf{D}}\dot{\mathbf{R}}, \ {P}_{\Omega}(\dot{\mathbf{L}}\dot{\mathbf{D}}\dot{\mathbf{R}}-\dot{\mathbf{M}})=\mathbf{0}.
	\end{cases}
	\label{QLNM_QQR}
\end{equation}
Convexity of the quaternion $L_{2,1}$-norm in the optimization function in (\ref{QLNM_QQR}) allows for the solution of the problem using the alternating direction method of multipliers (ADMM).
And we give the augmented Lagrangian function of (\ref{QLNM_QQR}):
\begin{equation}
	\begin{split}
	\text{Lag}=\|\dot{\mathbf{D}}\|_{2,1}+\mathfrak{R}(\langle \dot{\mathbf{E}}, \, \dot{\mathbf{X}}-\dot{\mathbf{L}} \dot{\mathbf{D}}\dot{\mathbf{R}}\rangle)+\frac{\mu}{2}\|\dot{\mathbf{X}}-\dot{\mathbf{L}} \dot{\mathbf{D}}\dot{\mathbf{R}}\|_{F}^{2},
	\end{split}
\end{equation}
where $\mu>0$, and the lagrange multiplier $\dot{\mathbf{E}}\in \mathbb{H}^{M \times N}$.
Essentially, there are two steps in the process of solving the entire problem. The first step is to update the variables $\dot{\mathbf{L}}$ and $\dot{\mathbf{R}}$ by solving the corresponding optimization problem based on the CQSVD-QQR method. The second step involves updating the variables $\dot{\mathbf{D}}$ and $\dot{\mathbf{X}}$.\\
\indent
The following minimization problem is solved in step 1 to update $\dot{\mathbf{L}}^{\tau+1}$ and $\dot{\mathbf{R}}^{\tau+1}$: 
\begin{equation}
	\mathop{\text{min}}\limits_{\dot{\mathbf{L}},\dot{\mathbf{R}}}\left\| (\dot{\mathbf{X}}^{\tau}+\dot{\mathbf{E}^{\tau}}/{\mu^{\tau}})-\dot{\mathbf{L}} \dot{\mathbf{D}}^{\tau} \dot{\mathbf{R}}  \right\|_{F} ^{2}.
\end{equation}
We can update $\dot{\mathbf{L}}^{\tau+1}$ and $\dot{\mathbf{R}}^{\tau+1}$ by using CQSVD-QQR, in accordance with the analysis of the CQSVD-QQR method, i.e.,
\begin{equation}
	\begin{cases}
		\left[\dot{\hat{\mathbf{L}}}^{\tau+1}, {\sim} \right]= \text{qqr}(\dot{\mathbf{X}}_{b}(\dot{\mathbf{R}}^{\tau })^{H}),\\ 
		\dot{\mathbf{L}}^{\tau+1} =\dot{\hat{\mathbf{L}}}^{\tau+1} (:,1:r),   \label{Upd_LQLNM-QQR}
	\end{cases}
\end{equation}
and 
\begin{equation}
	\begin{cases}
		\left[\dot{\hat{\mathbf{R}}}^{\tau+1},\dot{\hat{\mathbf{D}}}^{\tau} \right]= \text{qqr}(\dot{\mathbf{X}}_{b}^{H} \dot{\mathbf{L}}^{\tau+1 }),\\ 
		\dot{\mathbf{R}}^{\tau+1} =(\dot{\hat{\mathbf{R}}}^{\tau+1} (1:r,1:r))^{H},  \label{Upd_RQLNM-QQR}
	\end{cases}
\end{equation}
where $\dot{\mathbf{X}}_{b}=\dot{\mathbf{X}}^{\tau}+\dot{\mathbf{E}^{\tau}}/{\mu^{\tau}}$. If $\dot{\mathbf{L}}$ and $\dot{\mathbf{R}}$ are initialized as $\dot{\mathbf{L}}^{\tau}$ and $\dot{\mathbf{R}}^{\tau}$, the CQSVD-QQR method will converge within a limited number of iterations since the quaternion matrices $\dot{\mathbf{L}}$ and $\dot{\mathbf{R}}$ do not change significantly during the course of two iterations \cite{liu2018fast}. Therefore, the experiment's good outcome and the speed at which our algorithm iterates may both be enhanced by employing this wise initialization. \\
\indent
Step 2 involves updating $\dot{\mathbf{D}}^{\tau+1}$ and $\dot{\mathbf{X}}^{\tau+1}$. The following quaternion $L_{2,1}$-norm minimization problem is needed to solve in order to determine $\dot{\mathbf{D}}^{\tau+1}$:
\begin{equation}
	\dot{\mathbf{D}}^{\tau+1}=
	\mathop{\text{arg min}}\limits_{\dot{\mathbf{D}}}\|\dot{\mathbf{D}}\|_{2,1}+\frac{\mu^{\tau}}{2}\|\dot{\mathbf{D}}-(\dot{\mathbf{L}}^{\tau+1})^{H} \dot{\mathbf{X}}_{b}(\dot{\mathbf{R}}^{\tau+1})^{H}\|_{F}^{2}. 
	\label{Upd_D_P}
\end{equation}
The following theorem is established in order to solve the above problem (\ref{Upd_D_P}).
\begin{theorem}[] \label{L2,1mp} For the following minimization problem
\begin{equation}
	\mathop{\text{min}}\limits_{\dot{\mathbf{X}}}\mathcal{J}(\dot{\mathbf{X}})=\mathop{\text{min}}\limits_{\dot{\mathbf{X}}}\beta\|\dot{\mathbf{X}}\|_{2,1}+\frac{1}{2}\|\dot{\mathbf{X}}-\dot{\mathbf{Y}}\|_{F}^{2},  
	\label{L2,1_norm}
\end{equation}
where $\beta>0$, $\dot{\mathbf{X}}$, and $ \dot{\mathbf{Y}} \in \mathbb{H}^{M \times N}$, 
the $n$-th column of the optimal $\dot{\widetilde{\mathbf{X}}}$, i.e, $\dot{\widetilde{\mathbf{X}}}_{\cdot n}$, of (\ref{L2,1_norm}) is given by	
	\begin{equation}
		\dot{\widetilde{\mathbf{X}}}_{\cdot n}=\frac{(\|\dot{\mathbf{Y}}_{\cdot n}\|_{2}-4\beta)_{+}}{\|\dot{\mathbf{Y}}_{\cdot n}\|_{2}}\dot{\mathbf{Y}}_{\cdot n},
	\end{equation}
where $\dot{\widetilde{\mathbf{X}}}_{\cdot n}=\left[\dot{\widetilde{\mathbf{X}}}_{1n}, \dot{\widetilde{\mathbf{X}}}_{2n}, \dots, \dot{\widetilde{\mathbf{X}}}_{M n}\right]^{T}$, $\|\dot{\mathbf{Y}}_{\cdot n}\|_{2}=\sqrt{\sum_{m=1}^{M}|\dot{\mathbf{Y}}_{mn}|^{2}}$, and $(\|\dot{\mathbf{Y}}(:,n)\|_{2}-4\beta)_{+}=\text{max}\{\|\dot{\mathbf{Y}}(:,n)\|_{2}-4\beta,0\}$.
\end{theorem}
The proof of Theorem \ref{L2,1mp} can be found in the \ref{appendixA}.\\
\indent
Therefore, we can update $\dot{\mathbf{D}}^{\tau+1}$ in the problem (\ref{Upd_D_P}) by using Theorem \ref{L2,1mp} as follows:
\begin{equation}
\dot{\mathbf{D}}^{\tau+1}=\dot{\hat{\mathbf{D}}}\mathbf{C}, \label{Upd_D_QLNM_QQR1}
\end{equation}
where $\dot{\hat{\mathbf{D}}}=(\dot{\mathbf{L}}^{\tau+1})^{H} \dot{\mathbf{X}}_{b}(\dot{\mathbf{R}}^{\tau+1})^{H}$, and $\mathbf{C}=\text{diag}(c_1, \dots, c_r)$ is a diagonal matrix, where $c_s \ (s=1,\dots,r)$ is given by
 	\begin{equation}	
 		c_s=\frac{(\|{\dot{\hat{\mathbf{D}}}_{\cdot s}}\|_{2}-4\beta)_{+}}{\|{\dot{\hat{\mathbf{D}}}_{\cdot s}}\|_{2}}.
 \label{Upd_D_QLNM_QQR2}
\end{equation}
Next, we can update the variable $\dot{\mathbf{X}}_{\tau+1}$ by fixing the other variables as follows:
\begin{equation}
		\dot{\mathbf{X}}^{\tau+1}=\dot{\mathbf{L}}^{\tau+1}\dot{\mathbf{D}}^{\tau+1}\dot{\mathbf{R}}^{\tau+1}-{P}_{\Omega}(\dot{\mathbf{L}}^{\tau+1}\dot{\mathbf{D}}^{\tau+1}\dot{\mathbf{R}}^{\tau+1})+{P}_{\Omega}(\dot{\mathbf{M}}).
		\label{UPD_X_L21}
\end{equation}
Finally, we can update $\dot{\mathbf{E}}^{\tau+1}$ and $\mu$ by fixing the variables $\dot{\mathbf{L}}^{\tau+1}, \dot{\mathbf{D}}^{\tau+1}, \dot{\mathbf{R}}^{\tau+1}$, and $\dot{\mathbf{X}}^{\tau+1}$ as follows: 
\begin{equation}
	\dot{\mathbf{E}}^{\tau+1}=\dot{\mathbf{E}}^{\tau}+\mu^{\tau}(\dot{\mathbf{X}}^{\tau+1}-\dot{\mathbf{L}}^{\tau+1}\dot{\mathbf{D}}^{\tau+1}\dot{\mathbf{R}}^{\tau+1}),
	\label{UPD_Y_L21}
\end{equation}
\begin{equation}
	\mu^{\tau+1}=\rho\mu^{\tau},
	\label{UPD_mu_L21}
\end{equation}
where $\rho \geq 1$. We call the proposed CQSVD-QQR-based quaternion $L_{2,1}$-norm minimization approach QLNM-QQR. Algorithm \ref{tab_algorithmQLNM-QQR} provides a summary of the whole proposed approach's process. \\ 
\begin{algorithm}[htb]  
	\caption{The QLNM-QQR method for color image completion.}
	\label{tab_algorithmQLNM-QQR}
	\begin{algorithmic}[1]
		\Require The observed quaternion matrix data 
		$\dot{\mathbf{M}}\in\mathbb{H}^{M\times N}$ with $\mathcal{P}_{\Omega^{c}}(\dot{\mathbf{M}})=\mathbf{0}$; $\mu_{\max}$; $\rho$ and $r>0$.
		\State \textbf{Initialize} $\tau=0$; $\mu^{0}$;  $\varepsilon>0$; 
		$\text{It}_{\text{max}}>0$; $\dot{\mathbf{L}}^{0}=eye(M,r)$; $\dot{\mathbf{R}}^{0}=eye(r,N)$; $\dot{\mathbf{D}}^{0}=eye(r,r)$; $\dot{\mathbf{X}}^{0}=\dot{\mathbf{M}}$.
		\State \textbf{Repeat}
		\State \textbf{Step 1.} $\dot{\mathbf{L}}^{\tau+1}$, $\dot{\mathbf{R}}^{\tau+1}$: 
		(\ref{Upd_LQLNM-QQR}) and (\ref{Upd_RQLNM-QQR}), respectively.
		\State	\textbf{Step 2.}	
		$\dot{\mathbf{D}}^{\tau+1}$: (\ref{Upd_D_QLNM_QQR1}) and (\ref{Upd_D_QLNM_QQR2}).
		\State \textbf{Step 3.} $\dot{\mathbf{X}}^{\tau+1}=\dot{\mathbf{L}}^{\tau+1}\dot{\mathbf{D}}^{\tau+1}\dot{\mathbf{R}}^{\tau+1}-{P}_{\Omega}(\dot{\mathbf{L}}^{\tau+1}\dot{\mathbf{D}}^{\tau+1}\dot{\mathbf{R}}^{\tau+1})+{P}_{\Omega}(\dot{\mathbf{M}})$.
		\State $\dot{\mathbf{E}}^{\tau+1}=\dot{\mathbf{E}}^{\tau}+\mu^{\tau} 
		(\dot{\mathbf{X}}^{\tau+1}-\dot{\mathbf{L}}^{\tau+1}\dot{\mathbf{D}}^{\tau+1}\dot{\mathbf{R}}^{\tau+1}).$
		\State $\mu^{\tau+1}={\rm{min}}(\rho\mu^{\tau}, \mu_{max})$.
		\State \textbf{Until convergence}
		\Ensure   $\dot{\mathbf{L}}^{\tau+1}$, $\dot{\mathbf{D}}^{\tau+1}$, $\dot{\mathbf{R}}^{\tau+1}$, and $\dot{\mathbf{X}}^{\tau+1}$.
	\end{algorithmic}
\end{algorithm} 
\indent
QLNM-QQR can reach its optimal solution because the ADMM, a gradient-search-based approach, minimizes the convex optimization function of (\ref{QLNM_QQR}) and thus allows for convergence. Assume that QLNM-QQR converges after $t$ iterations. If the updating processes of the QLNM-QQR approach are carried out further, $\dot{\mathbf{X}}^{\tau}$ ($\tau\geq t$) will equal $\dot{\mathbf{X}}^{t}$. QLNM-QQR can revert to CQSVD-QQR since $\dot{\mathbf{L}}$ and $\dot{\mathbf{R}}$ from the preceding iteration served as initialization for CQSVD-QQR. Consequently, the series of $\{\dot{\mathbf{D}}^{\tau}\}$ generated by QLNM-QQR will converge to a diagonal matrix $\dot{\mathbf{D}}$ and
\begin{equation}
 |\dot{\mathbf{D}}_{ss}|=\sigma_{s}(\dot{\mathbf{X}}^{t}). \label{DTAUCOV}
 \end{equation} 
 Because of this, the QLNM-QQR algorithm's quaternion $L_{2,1}$-norm minimization function can reach the QNN of $\dot{\mathbf{D}}$. It inspires us to extend QLNM-QQR by introducing a quaternion $L_{2,1}$-norm minimization approach with iterative reweighting.

\textit{2) Extending the approach QLNM-QQR with iterative reweighting}: Based on (\ref{l21_nuclear}), the weighted quaternion $L_{2,1}$-norm of $\dot{\mathbf{X}}\in \mathbb{H}^{M\times N}$ is given by 
	\begin{equation}
	\|\dot{\mathbf{X}}\|_{\omega\cdot(2,1)}=\sum_{l=1}^{N}\omega_{l}\|\dot{\mathbf{X}}^{l}\|_{\ast}, \label{l21_REWI}
\end{equation}
where $\omega_{l}$ is a positive number, and $\dot{\mathbf{X}}^{l}$ and $\dot{\mathbf{D}}^{l}$ (in (\ref{D_L21bound})) both have the same definition. The following modification can be made to the QLNM-QQR minimization problem in (\ref{QLNM_QQR}): 
\begin{equation}
	\mathop{\text{min}}\limits_{\dot{\mathbf{D}}}\sum_{l=1}^{r}\partial g(\|\dot{\mathbf{D}}^{l}\|_{\ast})\|\dot{\mathbf{D}}^{l}\|_{\ast},   \: \text{s.t.}, 
	\begin{cases}
		\dot{\mathbf{L}}^{H} \dot{\mathbf{L}}=\mathbf{I}_r, \ \dot{\mathbf{R}} \dot{\mathbf{R}}^{H}=\mathbf{I}_r,\\
		\dot{\mathbf{X}}=\dot{\mathbf{L}}\dot{\mathbf{D}}\dot{\mathbf{R}}, \ {P}_{\Omega}(\dot{\mathbf{L}}\dot{\mathbf{D}}\dot{\mathbf{R}}-\dot{\mathbf{M}})=\mathbf{0}.
	\end{cases}
	\label{QLNM_QQR_wei}
\end{equation}
$\partial g(\sigma_s)$ is the gradient of $g(\cdot)$ at $\sigma_s$, and $g(\cdot)$ is continuous, concave, smooth, differentiable, and monotonically increasing on $\left[0, \ +\infty\right.)$ \citep{liu2018fast,chen2019low}. The ADMM can resolve the problem in (\ref{QLNM_QQR_wei}). We give its augmented Lagrangian function as below: 
\begin{equation}
	\begin{split}
		& \mathcal{L}(\dot{\mathbf{X}}, \dot{\mathbf{L}}, \dot{\mathbf{D}}, \dot{\mathbf{R}}, \dot{\mathbf{Y}})\\
		&=\sum_{l=1}^{r}\partial g(\|\dot{\mathbf{D}}^{l}\|_{\ast})\|\dot{\mathbf{D}}^{l}\|_{\ast}+\mathfrak{R}(\langle \dot{\mathbf{E}}, \, \dot{\mathbf{X}}-\dot{\mathbf{L}} \dot{\mathbf{D}}\dot{\mathbf{R}}\rangle)+\frac{\mu}{2}\|\dot{\mathbf{X}}-\dot{\mathbf{L}} \dot{\mathbf{D}}\dot{\mathbf{R}}\|_{F}^{2},
	\end{split}
\label{IRQLNM_PRO}
\end{equation}
In the $\tau$-th iteration, the variables $\dot{\mathbf{X}}^{\tau+1}$, $\dot{\mathbf{L}}^{\tau+1}$, $\dot{\mathbf{R}}^{\tau+1}$, $\dot{\mathbf{E}}^{\tau+1}$, and $\mu^{\tau+1}$ are updated the same as we did in the previous problem (\ref{QLNM_QQR}). We need to solve the subsequent minimization problem to update $\dot{\mathbf{D}}^{\tau+1}$:
\begin{equation}
	\mathop{\text{min}}\limits_{\dot{\mathbf{D}}}\sum_{l=1}^{r}\partial g(\|\dot{\mathbf{D}}^{l}\|_{\ast})\|\dot{\mathbf{D}}^{l}\|_{\ast}+\frac{\mu^{\tau}}{2}\|\dot{\mathbf{D}}-\dot{\hat{\mathbf{D}}}\|_{F}^{2},   
	\label{QLNM_QQR_wei_UPD}
\end{equation}
where $\dot{\hat{\mathbf{D}}}=(\dot{\mathbf{L}}^{\tau+1})^{H} \dot{\mathbf{X}}_{b}(\dot{\mathbf{R}}^{\tau+1})^{H}$. As in \cite{liu2018fast}, in the $\tau$-th iteration, we also let
\begin{equation}
	\partial g(\|\dot{\mathbf{D}}^{l}\|_{\ast})=\mu^{\tau}(1-\hat{a}_{l}) \|\dot{\hat{\mathbf{D}}}^{l}\|_{\ast} \quad (l=1, 2, \dots, r)
	\label{QLNM_QQR_wei_UPD1}
\end{equation}
where $1\geq \hat{a}_{1}\geq\hat{a}_{2}\geq \dots\geq \hat{a}_{r}>0$.
\begin{theorem}[] \label{L2,1mp2} Assume that $\dot{\mathbf{Y}} \in \mathbb{H}^{M\times M}$ and $\mu>0$, for the minimization problem below, 
		\begin{equation}
		\mathop{\text{min}}\limits_{\dot{\mathbf{X}}\in \mathbb{H}^{M\times M}}\frac{1}{\mu}\|\dot{\mathbf{X}}\|_{\omega \cdot(2,1)}+\frac{1}{2}\|\dot{\mathbf{X}}-\dot{\mathbf{Y}}\|_{F}^{2},  
		\label{L2,1_norm_wei}
	\end{equation}
	the optimal solution is given by
	\begin{equation}
	\dot{\mathbf{X}}_{\text{opt}}=\dot{\mathbf{Y}}\mathbf{A},
\end{equation}
	where $\mathbf{A}=\text{diag}(a_1,\dots,a_M)$, and
		\begin{equation}
	 a_m=\frac{({\sigma_{m}-\frac{\omega_{m}}{\mu}})_{+}}{\sigma_{m}},  \quad (m=1, \dots, M)
	\end{equation}
where $\sigma_{m}$ is the singular value of $\dot{\mathbf{Y}}^{m}$. $\dot{\mathbf{Y}}^{m}$ is defined in the same way as $\dot{\mathbf{D}}^{l}$ in (\ref{D_L21bound}). 
\end{theorem}
The proof of Theorem \ref{L2,1mp2} can be found in the \ref{appendixB}.\\
\indent
Theorem \ref{L2,1mp2} and (\ref{QLNM_QQR_wei_UPD1}) allow us to update $\dot{\mathbf{D}}^{\tau+1}$ as follows: 
\begin{equation}
	\dot{\mathbf{D}}^{\tau+1}=\dot{\hat{\mathbf{D}}}\mathbf{A}, \label{irqlnm_d_update}
\end{equation}
where $\mathbf{A}=\text{diag}( \hat{a}_{1},\dots, \hat{a}_{r})$. We abbreviate the proposed CQSVD-QQR-based iteratively reweighted quaternion $L_{2,1}$-norm minimization model for matrix completion as IRQLNM-QQR.

\begin{theorem}[] \label{L2,1cov} By using (\ref{QLNM_QQR_wei_UPD1}) to specify the weights $\partial g(\|\dot{\mathbf{D}}^{l}\|_{\ast}) \ (l=1,\cdots,r)$ in (\ref{IRQLNM_PRO}), IRQLNM-QQR can converge to the optimal solution of an LRQA-W minimization model with $\gamma=1$.
\end{theorem}
The proof of Theorem \ref{L2,1cov} can be found in the \ref{appendixC}.\\
\indent
For the experiment, we have specified the values of $\hat{a}_{l}$ ($l=1,2,\dots,r$) as follows:
	\begin{equation}
	\omega_{l}=\begin{cases}
		1, \quad \quad \quad \quad \quad \ 1\leqslant l\leqslant V, \: 1<V<r    \\
		\frac{\varsigma-1}{r-V}+\omega_{l-1}, \quad \: \: V< l\leqslant r, 
	\end{cases}
	\label{IRQLNM_QQR_wei}
\end{equation}
\begin{equation}
	\hat{a}_{l}=\frac{1}{\omega_{l}}
	\label{IRQLNM_QQR_wei1}
\end{equation}
where $\varsigma>1$, $r$ represents the number of rows in $\mathbf{A}$ from (\ref{irqlnm_d_update}).

\textit{3) QLNM-QQR with sparsity:} As we previously indicated, the performance of color image completion models will be enhanced by combining the low-rank property of color images with sparsity. As a result, we introduce the sparse regularization term into the minimization problem (\ref{QLNM_QQR}) and have the following optimization problem:
\begin{equation}
	\mathop{\text{min}}\limits_{\dot{\mathbf{D}}}\|\dot{\mathbf{D}}\|_{2,1}+\beta \|\dot{\mathbf{C}}\|_{1},   \: \text{s.t.}, 
	\begin{cases}
		\dot{\mathbf{L}}^{H} \dot{\mathbf{L}}=\mathbf{I}_r, \ \dot{\mathbf{R}} \dot{\mathbf{R}}^{H}=\mathbf{I}_r,\\
		\dot{\mathbf{X}}=\dot{\mathbf{L}}\dot{\mathbf{D}}\dot{\mathbf{R}}, \ {P}_{\Omega}(\dot{\mathbf{L}}\dot{\mathbf{D}}\dot{\mathbf{R}})={P}_{\Omega}(\dot{\mathbf{M}}), \ \mathcal{T}(\dot{\mathbf{X}})=\dot{\mathbf{C}},
	\end{cases}
	\label{M_QLNM_QQR_SP}
\end{equation}
where $\beta > 0$, $\dot{\mathbf{C}}=\mathcal{T}(\dot{\mathbf{X}})$ stands for the quaternion matrix after transformation, and $\mathcal{T}(\cdot)$ denotes a transform operator. In this section, the $\text{FQDCT}^{L}$ is employed to formulate our proposed approach. In this case, $\mathcal{T}(\dot{\mathbf{X}})$ calculates $\text{FQDCT}^{L}$ of $\dot{\mathbf{X}}$. \\
\indent
We also use the ADMM framework in the quaternion system to solve problem (\ref{M_QLNM_QQR_SP}), and its corresponding augmented Lagrangian function is shown below:
\begin{equation}
	\begin{split}
	\text{Lag}=\|\dot{\mathbf{D}}\|_{2,1}+\beta \|\dot{\mathbf{C}}\|_{1}+\mathfrak{R}(\langle \dot{\mathbf{E}}, \, \dot{\mathbf{X}}-\dot{\mathbf{L}} \dot{\mathbf{D}}\dot{\mathbf{R}}\rangle)+\frac{\mu}{2}\|\dot{\mathbf{X}}-\dot{\mathbf{L}} \dot{\mathbf{D}}\dot{\mathbf{R}}\|_{F}^{2}+\mathfrak{R}(\langle \dot{\mathbf{F}}, \, \dot{\mathbf{C}}-\mathcal{T}(\dot{\mathbf{X}})\rangle)+\frac{\mu}{2}\|\dot{\mathbf{C}}-\mathcal{T}(\dot{\mathbf{X}})\|_{F}^{2},
	\end{split}
\end{equation}
where the penalty parameter $\mu$ is a positive number, $\dot{\mathbf{E}}$ and $\dot{\mathbf{F}}$ are Lagrange multiplier. The ADMM framework allows for the alternate updating of each parameter in the optimization problem (\ref{M_QLNM_QQR_SP}) by fixing other variables. Notably, in the $\tau$-th iteration, each variable can be updated as follows:
\begin{equation}
	\begin{cases}
		\dot{\mathbf{L}}^{\tau+1}=\mathop{\text{arg min}}\limits_{\dot{\mathbf{L}}}\text{Lag}(\dot{\mathbf{X}}^{\tau}, \dot{\mathbf{L}}, \dot{\mathbf{D}}^{\tau}, \dot{\mathbf{R}}^{\tau}, \dot{\mathbf{C}}^{\tau}, \dot{\mathbf{E}}^{\tau}, \dot{\mathbf{F}}^{\tau}),\\
		\dot{\mathbf{R}}^{t+1}=\mathop{\text{arg min}}\limits_{\dot{\mathbf{R}}}\text{Lag}(\dot{\mathbf{X}}^{\tau}, \dot{\mathbf{L}}^{\tau+1}, \dot{\mathbf{D}}^{\tau}, \dot{\mathbf{R}}, \dot{\mathbf{C}}^{\tau}, \dot{\mathbf{E}}^{\tau}, \dot{\mathbf{F}}^{\tau}),\\
		\dot{\mathbf{D}}^{\tau+1}=\mathop{\text{arg min}}\limits_{\dot{\mathbf{D}}}\text{Lag}(\dot{\mathbf{X}}^{\tau}, \dot{\mathbf{L}}^{\tau+1}, \dot{\mathbf{D}}, \dot{\mathbf{R}}^{\tau+1}, \dot{\mathbf{C}}^{\tau}, \dot{\mathbf{E}}^{\tau}, \dot{\mathbf{F}}^{\tau}),\\
		\dot{\mathbf{X}}^{\tau+1}=\mathop{\text{arg min}}\limits_{\dot{\mathbf{X}}}\text{Lag}(\dot{\mathbf{X}}, \dot{\mathbf{L}}^{\tau+1}, \dot{\mathbf{D}}^{\tau+1}, \dot{\mathbf{R}}^{\tau+1}, \dot{\mathbf{C}}^{\tau}, \dot{\mathbf{E}}^{\tau}, \dot{\mathbf{F}}^{\tau}),\\
		\dot{\mathbf{C}}^{\tau+1}=\mathop{\text{arg min}}\limits_{\dot{\mathbf{C}}}\text{Lag}(\dot{\mathbf{X}}^{\tau+1}, \dot{\mathbf{L}}^{\tau+1}, \dot{\mathbf{D}}^{\tau+1}, \dot{\mathbf{R}}^{\tau+1}, \dot{\mathbf{C}}, \dot{\mathbf{E}}^{\tau}, \dot{\mathbf{F}}^{\tau}),\\
		\dot{\mathbf{E}}^{\tau+1}=\dot{\mathbf{E}}^{\tau}+\mu^{\tau}( \dot{\mathbf{X}}^{\tau+1}-\dot{\mathbf{L}}^{\tau+1} \dot{\mathbf{D}}^{\tau+1}\dot{\mathbf{R}}^{\tau+1}),\\
		\dot{\mathbf{F}}^{\tau+1}=\dot{\mathbf{F}}^{\tau}+\mu^{\tau}( \dot{\mathbf{C}}^{\tau+1}-\mathcal{T}(\dot{\mathbf{X}}^{\tau+1})).\\
	\end{cases}
	\label{sub_optimi}
\end{equation}
\indent
\textbf{Updating $\dot{\mathbf{L}}$, $\dot{\mathbf{D}}$, and $\dot{\mathbf{R}}$}: The following minimization problem is needed to solve in the $\tau+1$-th iteration to update $\dot{\mathbf{L}}^{\tau+1}$ and $\dot{\mathbf{R}}^{\tau+1}$: 
\begin{equation}
	\mathop{\text{min}}\limits_{\dot{\mathbf{L}},\dot{\mathbf{R}}}\left\| (\dot{\mathbf{X}}^{\tau}+\frac{\dot{\mathbf{E}}^{\tau}}{\mu^{\tau}})-\dot{\mathbf{L}} \dot{\mathbf{D}}^{\tau} \dot{\mathbf{R}}  \right\|_{F} ^{2}.
\end{equation}
The updates of $\dot{\mathbf{L}}^{\tau+1}$ and $\dot{\mathbf{R}}^{\tau+1}$ here are the same as in the QLNM-QQR method, i.e., $\dot{\mathbf{L}}^{\tau+1}$ and $\dot{\mathbf{R}}^{\tau+1}$ are updated according to (\ref{Upd_LQLNM-QQR}) and (\ref{Upd_RQLNM-QQR}), respectively. \\
\indent
Also, since variable $\dot{\mathbf{D}}$ does not directly rely on variable $\dot{\mathbf{C}}$, $\dot{\mathbf{D}}^{\tau+1}$ is updated by addressing the same problem, i.e., problem (\ref{Upd_D_P}), in the previous subsection. Thus, we update $\dot{\mathbf{D}}^{\tau+1}$ by using (\ref{Upd_D_QLNM_QQR1}).\\
\indent
\textbf{Updating $\dot{\mathbf{X}}$, $\dot{\mathbf{C}}$, $\dot{\mathbf{E}}$, $\dot{\mathbf{F}}$, and $\mu$}: By resolving the following problem in the $\tau$-th iteration, $\dot{\mathbf{X}}^{\tau+1}$ can be updated after updating the variables $\dot{\mathbf{L}}^{\tau+1}$, $\dot{\mathbf{D}}^{\tau+1}$, and $\dot{\mathbf{R}}^{\tau+1}$ and fixing the remaining variables: 
\begin{equation}
	%\begin{alighed}
	\begin{split}
		\dot{\mathbf{X}}^{\tau+1}&=\mathop{\text{arg min}}\limits_{\dot{\mathbf{X}}}\mathfrak{R}(\langle \dot{\mathbf{E}}^{\tau}, \, \dot{\mathbf{X}}-\dot{\mathbf{L}}^{\tau+1} \dot{\mathbf{D}}^{\tau+1} \dot{\mathbf{R}}^{\tau+1} \rangle)+\frac{\mu^{\tau}}{2}\|\dot{\mathbf{X}}-\dot{\mathbf{L}}^{\tau+1}  \dot{\mathbf{D}}^{\tau+1} \dot{\mathbf{R}}^{\tau+1} \|_{F}^{2}\\
		& +\mathfrak{R}(\langle \dot{\mathbf{F}}^{\tau} , \, \dot{\mathbf{C}}^{\tau} -\mathcal{T}(\dot{\mathbf{X}})\rangle)+\frac{\mu^{\tau}}{2}\|\dot{\mathbf{C}}^{\tau} -\mathcal{T}(\dot{\mathbf{X}})\|_{F}^{2}, \\
		&=\mathop{\text{arg min}}\limits_{\dot{\mathbf{X}}}\frac{\mu^{\tau}}{2}\|\dot{\mathbf{X}}+\frac{\dot{\mathbf{E}}^{\tau}}{\mu^{\tau}}-\dot{\mathbf{L}}^{\tau+1} \dot{\mathbf{D}}^{\tau+1}\dot{\mathbf{R}}^{\tau+1}\|_{F}^{2}+\frac{\mu^{\tau}}{2}\|\dot{\mathbf{C}}^{\tau}+\frac{\dot{\mathbf{F}}^{\tau}}{\mu^{\tau}}-\mathcal{T}(\dot{\mathbf{X}})\|_{F}^{2}.
		\label{UPD_X_QQR_SP}
	\end{split}
	% \end{alighed}
\end{equation}
Since the item $\mathcal{T}(\dot{\mathbf{X}})$ is contained in problem (\ref{UPD_X_QQR_SP}), we cannot direct separate variable $\dot{\mathbf{X}}$ from other variables. We can reformulate the problem according to the Parseval theorem in the quaternion system, and then we can isolate $\dot{\mathbf{X}}$ from $\mathcal{T}(\cdot)$. The quaternion system's equivalent of the Parseval theorem in the real domain states that following a unitary transformation like the quaternion discrete Fourier transform (QDFT) or quaternion discrete cosine transform (QDCT), the signal's overall energy stays constant \cite{bahri2008uncertainty}. As a result, by adding the appropriate inverse transform to the final component in (\ref{UPD_X_QQR_SP}), we get that 
\begin{equation}
	\|\dot{\mathbf{C}}^{\tau}+\frac{\dot{\mathbf{F}}^{\tau}}{\mu^{\tau}}-\mathcal{T}(\dot{\mathbf{X}})\|_{F}^{2}= \|\mathcal{I}(\dot{\mathbf{C}}^{\tau}+\frac{\dot{\mathbf{F}}^{\tau}}{\mu^\tau})-\dot{\mathbf{X}}\|_{F}^{2},
	\label{intro_inver}
\end{equation}
where $\mathcal{I}(\cdot)$ stands for $\mathcal{T}(\cdot)$'s inverse transform.
According to (\ref{UPD_X_QQR_SP}) and (\ref{intro_inver}), as seen below, the optimization problem that is used to update X is reformulated as:
\begin{equation}
	\dot{\mathbf{X}}^{\tau+1}=\mathop{\text{arg min}}\limits_{\dot{\mathbf{X}}}
	\|\frac{1}{2}(\dot{\mathbf{L}}^{\tau+1} \dot{\mathbf{D}}^{\tau+1}\dot{\mathbf{R}}^{\tau+1}+\mathcal{I}(\dot{\mathbf{C}}^{\tau}+\frac{\dot{\mathbf{F}}^{\tau}}{\mu^{\tau}})-\frac{\dot{\mathbf{E}}^{\tau}}{\mu^{\tau}})-\dot{\mathbf{X}}\|_{F}^{2}.
	\label{UPD_X_spa}
\end{equation}
The closed-form solution to the problem (\ref{UPD_X_spa}) is given by
\begin{equation}
	\dot{\mathbf{X}}^{\tau+1}=
	\frac{1}{2}(\dot{\mathbf{L}}^{\tau+1} \dot{\mathbf{D}}^{\tau+1}\dot{\mathbf{R}}^{\tau+1}+\mathcal{I}(\dot{\mathbf{C}}^{\tau}+\frac{\dot{\mathbf{F}}^{\tau}}{\mu^{\tau}})-\frac{\dot{\mathbf{E}}^{\tau}}{\mu^{\tau}}).
	\label{soLu_X}
\end{equation}
Considering the restriction ${P}_{\Omega}(\dot{\mathbf{L}}\dot{\mathbf{D}}\dot{\mathbf{R}})={P}_{\Omega}(\dot{\mathbf{M}})$, we get that
\begin{equation}
	\dot{\mathbf{X}}^{\tau+1}={P}_{\Omega}(\dot{\mathbf{M}})+{P}_{\Omega^{c}}(\dot{\mathbf{X}}^{\tau+1}),	
	\label{sou_X1}
\end{equation}
where $\Omega^{c}$ stands for the missing entries' indexes. \\
\indent
The next step is to update the variable $\dot{\mathbf{C}}^{\tau+1}$ by resolving the following problem: 
\begin{equation}
	\begin{aligned}
		\dot{\mathbf{C}}^{t+1}&=\mathop{\text{arg min}}\limits_{\dot{\mathbf{C}}}
		\beta \|\dot{\mathbf{C}}\|_{1}+\mathfrak{R}(\langle \dot{\mathbf{F}}^{\tau}, \, \dot{\mathbf{C}}-\mathcal{T}(\dot{\mathbf{X}}^{\tau+1})\rangle)+\frac{\mu^{\tau}}{2}\|\dot{\mathbf{C}}-\mathcal{T}(\dot{\mathbf{X}}^{\tau+1})\|_{F}^{2}\\
		&=\mathop{\text{arg min}}\limits_{\dot{\mathbf{C}}}\beta \|\dot{\mathbf{C}}\|_{1}
		+\frac{\mu^{\tau}}{2}\|\frac{\dot{\mathbf{F}}^{\tau}}{\mu^\tau}+\dot{\mathbf{C}}-\mathcal{T}(\dot{\mathbf{X}}^{\tau+1})\|_{F}^{2}. \label{UPD_W_SP}
	\end{aligned}
\end{equation}
The closed-form solution to the problem (\ref{UPD_W_SP}) is given by
\begin{equation}
	\dot{\mathbf{C}}^{\tau+1}=\mathcal{S}_{\frac{4\beta}{\mu^{\tau}}}(\mathcal{T}(\dot{\mathbf{X}}^{\tau+1})-\frac{\dot{\mathbf{F}}^{\tau}}{\mu^\tau}),
	\label{sou_W}
\end{equation}
where $\mathcal{S}_{t}(\dot{\mathbf{x}})=\frac{\dot{\mathbf{x}}}{|\dot{\mathbf{x}}|}\text{max}\{|\dot{\mathbf{x}}|-t,0\}$ stands for the element-wise soft thresholding operator \cite{yang2022quaternion}.\\
\indent
Finally, the penalty parameter $\mu^{\tau+1}$ is updated as follows:
\begin{equation}
	\mu^{\tau+1}=\rho\mu^{\tau}.
	\label{UPD_mu}
\end{equation}
We call this approach QLNM-QQR-SR because, in contrast to QLNM-QQR, it also considers a sparse regularization term. Algorithm \ref{tab_algorithm2} summarizes all of the steps involved in the proposes model.
\begin{algorithm}[htb]  
%\begin{table}\footnotesize
	\caption{The proposed QLNM-QQR-SR method for color image completion.}
%	\hrule
	\label{tab_algorithm2}
	\begin{algorithmic}[1]
		\Require The observed data $\dot{\mathbf{M}}\in\mathbb{H}^{M\times N}$ ($\mathcal{P}_{\Omega^{c}}(\dot{\mathbf{M}})=\mathbf{0}$); $\rho$; $\mu_{\max}$; $\beta$; $r$.
		\State \textbf{Initialize} $\tau=0$;  $\varepsilon>0$;  $\mu^{0}$;
		$\text{It}_{\text{max}}>0$; $\dot{\mathbf{L}}^{0}=eye(M,r)$; $\dot{\mathbf{R}}^{0}=eye(r,N)$; $\dot{\mathbf{D}}^{0}=eye(r,r)$; $\dot{\mathbf{X}}^{0}=\dot{\mathbf{M}}$; $\dot{\mathbf{C}}^{0}=\mathbf{0}$.
		\State \textbf{Repeat}
		\State \textbf{Step 1.} $\dot{\mathbf{L}}^{\tau+1}$, $\dot{\mathbf{R}}^{\tau+1}$: 
		(\ref{Upd_LQLNM-QQR}) and (\ref{Upd_RQLNM-QQR}), respectively.
		\State	\textbf{Step 2.}	
		$\dot{\mathbf{D}}^{t+1}$: (\ref{Upd_D_QLNM_QQR1}) and (\ref{Upd_D_QLNM_QQR2}).
		\State \textbf{Step 3.} $\dot{\mathbf{X}}^{\tau+1}=\mathcal{P}_{\Omega}(\dot{\mathbf{M}})+\mathcal{P}_{\Omega^{c}}\big(\frac{1} 
		{2}(\dot{\mathbf{L}}^{\tau+1} \dot{\mathbf{D}}^{\tau+1}\dot{\mathbf{R}}^{\tau+1}-\frac{\dot{\mathbf{E}}^{t}}{\mu^\tau}+\mathcal{I}(\dot{\mathbf{C}}^{\tau}+\frac{\dot{\mathbf{F}}^{\tau}}{\mu^\tau}))\big)$.
		\State \textbf{Step 4.} $\dot{\mathbf{C}}^{\tau+1}=\mathcal{S}_{\frac{4\beta}{\mu^{\tau}}}(\mathcal{T}(\dot{\mathbf{X}}^{\tau+1})-\frac{\dot{\mathbf{F}}^{\tau}}{\mu^\tau})$.
		\State $\dot{\mathbf{E}}^{\tau+1}=\dot{\mathbf{E}}^{\tau}+\mu^{\tau}( \dot{\mathbf{X}}^{\tau+1}-\dot{\mathbf{L}}^{\tau+1} \dot{\mathbf{D}}^{\tau+1}\dot{\mathbf{R}}^{\tau+1}).$
		\State $\dot{\mathbf{F}}^{\tau+1}=\dot{\mathbf{F}}^{\tau}+\mu^{\tau}( \dot{\mathbf{C}}^{\tau+1}-\mathcal{T}(\dot{\mathbf{X}}^{\tau+1}))$.
		\State $\mu^{\tau+1}={\rm{min}}(\rho\mu^{\tau}, \mu_{max})$.
		\State \textbf{Until convergence}
		\Ensure   $\dot{\mathbf{L}}^{\tau+1}$, $\dot{\mathbf{D}}^{\tau+1}$, $\dot{\mathbf{R}}^{\tau+1}$, $\dot{\mathbf{X}}^{\tau+1}$, and $\dot{\mathbf{C}}^{\tau+1}$.
	\end{algorithmic}
%	\hrule
%\end{table}
\end{algorithm}  
	
\subsection{Complexity analysis}
The computational complexities of our proposed three approaches are investigated in this subsection. According to Algorithm \ref{tab_algorithmQLNM-QQR}, the computation of the QQR decomposition of two quaternion matrices accounts for most of the computing cost of each iteration of QLNM-QQR. Additionally, these two quaternion matrices have sizes of $M\times r$ and $N\times r$, respectively, where $r<\text{min}\{M, N\}$. As a result, the complexity is approximately $\mathcal{O}(r^2(M+N)-r^{3})$. Similarly, the complexity of IRQLNM-QQR is also around $\mathcal{O}(r^2(M+N)-r^{3})$. Regarding the algorithmic complexity of the QLNM-QQR-SR technique, the transformation operator $\mathcal{T}(\cdot)$ adds significant additional algorithmic complexity on top of the QQR decomposition computation. $\mathcal{T}(\cdot)$ corresponds to a complexity of approximately $\mathcal{O}(M^2N^2+MN)$. The proposed QLNM-QQR-SR approach's overall computing complexity is thus around $\mathcal{O}(M^2N^2+MN)$ per iteration.

However, methods directly based on the QNN, such as LRQA-G \cite{chen2019low}, need to calculate the QSVD of the quaternion matrix with size $M\times N$, and the complexity is about $\mathcal{O}(\text{min}(MN^2, M^2N))$. It is clear that the computational cost of QQR decomposition of two quaternion matrices of size $M\times r$ and $N\times r$ is smaller than that of QSVD of the quaternion matrix with size $M\times N$. Additionally, the computational complexity of each iteration for the LRQMC method is with an estimated cost of $\mathcal{O}(\widetilde{r}({\widetilde{r}}^2+(M+N)\widetilde{r}+MN))$, where $\widetilde{r}$ denotes the estimated rank of the complex representation matrix of $\dot{\mathbf{X}}$ of size $2M\times 2N$. The computational cost of IRLNM-QR is around $\mathcal{O}(r^2(M+N))$. It can be found that both QLNM-QQR and IRQLNM-QQR approaches our proposed have computational complexities that are comparable to IRLNM-QR while being less than those of LRQA and LRQMC.
\section{Simulation results and discussion}
To show the effectiveness of the proposed three quaternion-based completion techniques (i.e., QLNM-QQR, IRQLNM-QQR, and QLNM-QQR-SR), we perform numerical experiments on natural color images and color medical images in this part. We also present corresponding numerical results analysis for the three quaternion-based completion methods. On a MATLAB 2019b platform with an i7-9700 CPU and 16 GB memory, we run all of the tests.    \\
\indent
\textbf{Evaluation metrics:} Peak signal-to-noise ratio (PSNR) and structural similarity (SSIM) are two extensively used metrics that we utilize to assess the effectiveness of the proposed QQR-QNN-SR, IRQLNM-QQR, and QLNM-QQR-SR. Better recovery performance is indicated by higher PSNR and SSIM.

\indent
We present the experimental results and analysis of the QLNM-QQR, IRQLNM-QQR, and QLNM-QQR-SR approaches. Specifically, we conduct numerical experiments on matrices with randomly missing entries, including those composed of quaternions. To showcase the effectiveness of our method, we consider the missing ratio (MR) of entries ranging from $50\%$ to $85\%$.\\
\indent
\textbf{ Method comparison:} Several popular completion techniques, namely, WNNM \cite{gu2017weighted}, MC-NC \cite{nie2018matrix}, IRLNM-QR \cite{liu2018fast}, TNN-SR \cite{dong2018low}, QLNF \cite{yang2022logarithmic}, TQLNA \cite{yang2022logarithmic}, LRQA-G \cite{chen2019low}, and LRQMC \cite{miao2021color} are contrasted with the proposed QLNM-QQR, IRQLNM-QQR, and QLNM-QQR-SR approaches.
Whereas WNNM, MC-NC, IRLNM-QR, and TNN-SR are LRMC-based completion algorithms, QLNF, TQLNA, LRQA-G, and LRQMC are LRQMC-based completion methods. Like the processing in \cite{dong2018low}, the approaches based on real matrix completion handle the three color channels individually first before combining their output to produce the final product. \\
\indent
\textbf{Experimental results on the eight natural color images:} In our tests, eight frequently used $256\times 256$ natural color images, as shown in Fig. \ref{fig:1}, are chosen as the test images (four from McMaster Datase \cite{zhang2011color}).
Our experiments used the following parameters for QLNM-QQR: $\mu^{0}=0.003, \rho =1.05$. And the value of $r$ is set to $65$, $90$, $105$, and $125$, corresponding to MR values of $85\%$, $75\%$, $65\%$, and $50\%$, respectively. For IRQLNM-QQR, we set $\mu^{0}=0.003$ and $\rho =1$. The value of $r$ is set to $115$, $125$, $155$, and $170$, corresponding to MR values of $85\%$, $75\%$, $65\%$, and $50\%$, respectively. To implement IRLNM-QQR, we set the values of parameters $\varsigma$ and V in (\ref{IRQLNM_QQR_wei}) to 10 and 3, respectively. As for the QLNM-QQR-SR method, $\mu^{0}$, $\beta$, and $\rho$ are set as $0.5$, $0.5$, and $1.05$, respectively. The parameter $r$ is assigned the values 60, 85, 100, and 120, which correspond to MR (missing rate) values of $85\%$, $75\%$, $65\%$, and $50\%$, respectively. As the value of MR increases, the number of missing pixels increases, and the corresponding rank decreases.

\begin{figure}[htbp]
	\centering
	\begin{minipage}[t]{0.1\textwidth}
		\centering
		\centerline{\includegraphics[width=1.7cm,height=1.7cm]{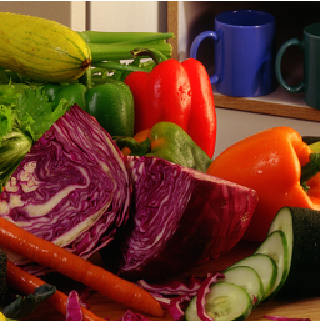}}
		%\caption*{Image (1)}
		\centerline{Image(1)}    
	\end{minipage}
	\begin{minipage}[t]{0.1\textwidth}
		\centering
		\centerline{\includegraphics[width=1.7cm,height=1.7cm]{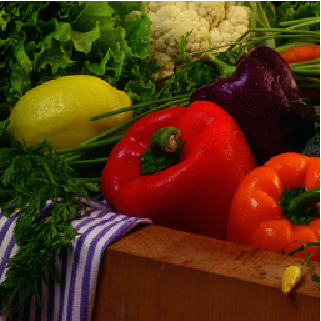}}
		%\caption*{Image (2)} 
		\centerline{Image(2)}
	\end{minipage}
	\begin{minipage}[t]{0.1\textwidth}
		\centering
		\centerline{\includegraphics[width=1.7cm,height=1.7cm]{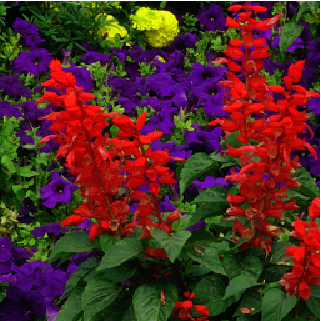}}
		\centerline{Image(3)}
	\end{minipage}
	\begin{minipage}[t]{0.1\textwidth}
		\centering
		\centerline{\includegraphics[width=1.7cm,height=1.7cm]{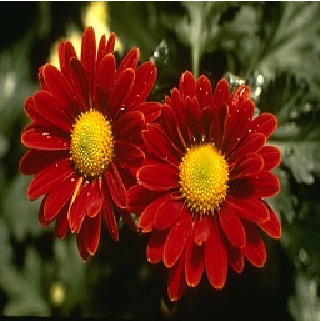}}
		\centerline{Image(4)}
	\end{minipage}
	\begin{minipage}[t]{0.1\textwidth}
		\centering
		\centerline{\includegraphics[width=1.7cm,height=1.7cm]{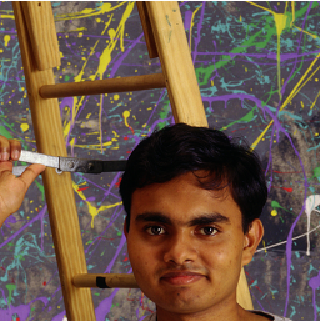}}
		\centerline{Image(5)}
	\end{minipage}
	\begin{minipage}[t]{0.1\textwidth}
		\centering
		\centerline{\includegraphics[width=1.7cm,height=1.7cm]{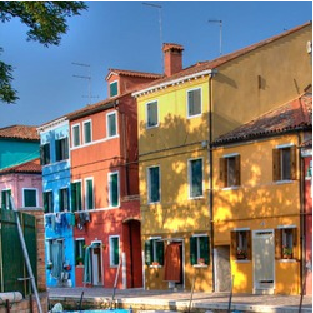}}
		\centerline{Image(6)}
	\end{minipage}
	\begin{minipage}[t]{0.1\textwidth}
		\centering
		\centerline{\includegraphics[width=1.7cm,height=1.7cm]{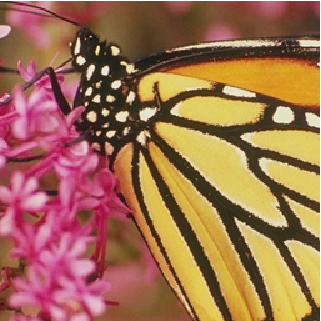}}
		\centerline{Image(7)}
	\end{minipage}
	\begin{minipage}[t]{0.1\textwidth}
		\centering
		\centerline{\includegraphics[width=1.7cm,height=1.7cm]{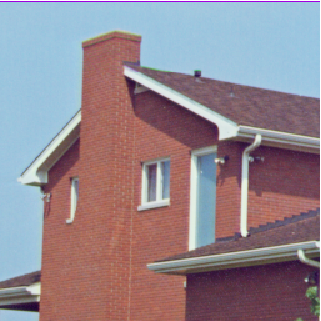}}
		\centerline{Image(8)}
	\end{minipage}
	\caption{Ground truth: Image(1)-Image(8) are eight color images, each with dimensions of $256\times256\times3$.}
	\label{fig:1}
\end{figure}

\indent
The recovery results for MR$=75\%$ are compared among different methods, as shown in \cref{fig:randommissing}.
The results of the quantitative evaluation for different methods under different missing ratios (MR) are presented in Table \ref{tablecolornature}. To illustrate the superior performance of our proposed approaches, we present in \cref{fig:IMAGE25} the visual results for image(2) and image(5) recovered using our approaches as well as several state-of-the-art approaches, all with a missing rate (MR) of $85\%$.\\
\indent
The results presented in Table \ref{tablecolornature} and \cref{fig:randommissing} indicate that both QLNM-QQR and IRQLNM-QQR exhibit superior completion performance compared to WNNM. IRQLNM-QQR achieves superior completion results, higher PSNR and SSIM values compared to IRLNM-QR, and QLNM-QQR generally outperforms IRLNM-QR. Also, when MR is large, IRQLNM-QQR outperforms MC-NC. These results demonstrate the effectiveness of quaternion representations in solving color image completion problems. Based on Theorem \ref{L2,1cov}, the IRQLNM-QQR method will approach the optimal solution of an LRQA-W minimization model. Additionally, as discussed in \cite{chen2019low}, LRQA-based models using non-convex functions demonstrate similar performance in the color image completion problem. Moreover, compared to LRQA-N and LRQA-L, the LRQA-G method has been found to have the best completion results. Table \ref{tablecolornature} indicates that the IRQLNM-QQR method can achieve numerical results comparable to LRQA-G. Furthermore, when the missing rate is high, the IRQLNM-QQR method generally outperforms the LRQA-G method. Also, the experimental results suggest that IRLNM-QQR exhibits superior precision to QLNM-QQR. \\
\indent
 According to the results presented in \cref{fig:randommissing} and Table \ref{tablecolornature}, it can be observed that the QLNM-QQR-SR method proposed in this paper outperforms other methods in terms of both visual and quantitative assessments. The results indicate that incorporating sparse prior information is crucial for achieving better completion results, as both TNN-SR and QLNM-QQR-SR outperform other methods.  Furthermore, QLNM-QQR-SR performs better than TNN-SR, which can be attributed to the ability of quaternions to characterize color images better. Based on the visual results presented in \cref{fig:IMAGE25}, our proposed method outperforms other state-of-the-art methods in terms of recovering more details from the observed images.

\begin{figure*}[htbp]
	\centering	
	\begin{minipage}[h]{0.065\textwidth}
		\centering
		\begin{minipage}{1\textwidth}
			\centering
			\includegraphics[width=1.1cm,height=1.1cm]{SR_0.25_EPS/ORI_10.eps}
		\end{minipage} 
		\hfill\\
		\begin{minipage}{1\textwidth}
			\centering
			\includegraphics[width=1.1cm,height=1.1cm]{SR_0.25_EPS/ORI_11.eps}
		\end{minipage} 
		\hfill\\
		\begin{minipage}{1\textwidth}
			\centering
			\includegraphics[width=1.1cm,height=1.1cm]{SR_0.25_EPS/ORI_17.eps}
		\end{minipage} 
		\hfill\\
			\begin{minipage}{1\textwidth}
			\centering
			\includegraphics[width=1.1cm,height=1.1cm]{SR_0.25_EPS/ORI_Flower.eps}
		\end{minipage} 
		\hfill\\	
		\begin{minipage}{1\textwidth}
			\centering
			\includegraphics[width=1.1cm,height=1.1cm]{SR_0.25_EPS/ORI_6.eps}
		\end{minipage} 
		\hfill\\
		\begin{minipage}{1\textwidth}
			\centering
			\includegraphics[width=1.1cm,height=1.1cm]{SR_0.25_EPS/ORI_Burano.eps}
		\end{minipage} 
		\hfill\\
		\begin{minipage}{1\textwidth}
			\centering
			\includegraphics[width=1.1cm,height=1.1cm]{SR_0.25_EPS/ORI_monarch.eps}
		\end{minipage} 
		\hfill\\
		\begin{minipage}{1\textwidth}
			\centering
			\includegraphics[width=1.1cm,height=1.1cm]{SR_0.25_EPS/ORI_4.1.05.eps}
		\end{minipage} 
		\hfill\\
		\caption*{(a)}
	\end{minipage}
	\begin{minipage}[h]{0.065\textwidth}
		\centering
		\begin{minipage}{1\textwidth}
			\centering
			\includegraphics[width=1.1cm,height=1.1cm]{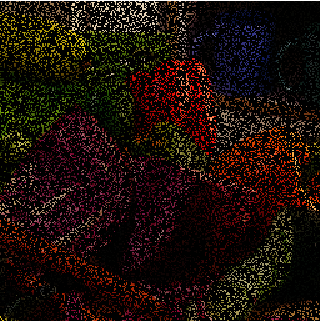}
		\end{minipage} 
		\hfill\\
		\begin{minipage}{1\textwidth}
			\centering
			\includegraphics[width=1.1cm,height=1.1cm]{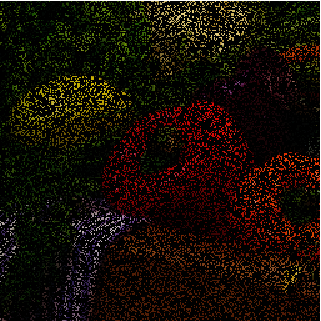}
		\end{minipage} 
		\hfill\\
		\begin{minipage}{1\textwidth}
			\centering
			\includegraphics[width=1.1cm,height=1.1cm]{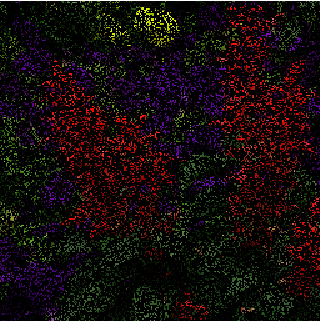}
		\end{minipage} 
		\hfill\\
			\begin{minipage}{1\textwidth}
			\centering
			\includegraphics[width=1.1cm,height=1.1cm]{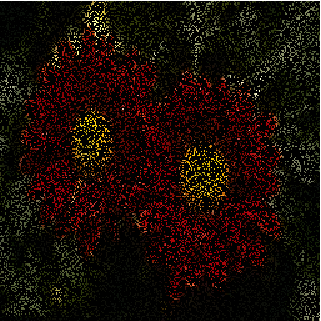}
		\end{minipage} 
		\hfill\\
		\begin{minipage}{1\textwidth}
			\centering
			\includegraphics[width=1.1cm,height=1.1cm]{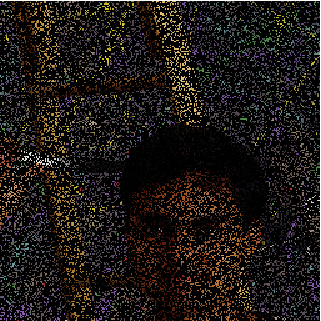}
		\end{minipage} 
		\hfill\\
		\begin{minipage}{1\textwidth}
			\centering
			\includegraphics[width=1.1cm,height=1.1cm]{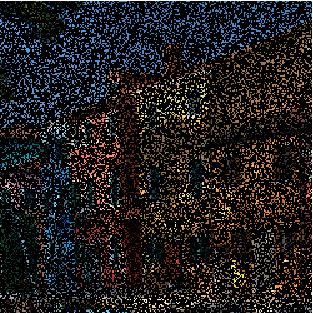}
		\end{minipage} 
		\hfill\\
		\begin{minipage}{1\textwidth}
			\centering
			\includegraphics[width=1.1cm,height=1.1cm]{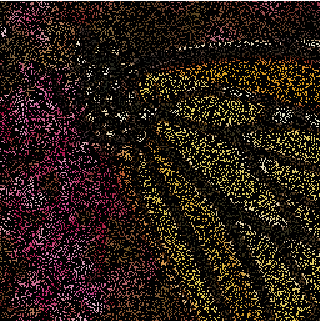}
		\end{minipage} 
		\hfill\\
			\begin{minipage}{1\textwidth}
			\centering
			\includegraphics[width=1.1cm,height=1.1cm]{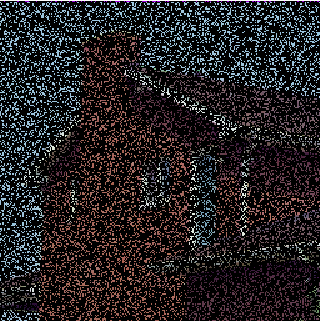}
		\end{minipage} 
		\hfill\\
		\caption*{(b)}
	\end{minipage}
\begin{minipage}[h]{0.065\textwidth}
	\centering
	\begin{minipage}{1\textwidth}
		\centering
		\includegraphics[width=1.1cm,height=1.1cm]{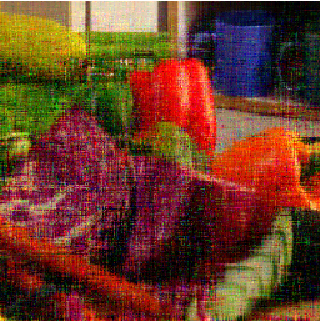}
	\end{minipage} 
	\hfill\\
	\begin{minipage}{1\textwidth}
		\centering
		\includegraphics[width=1.1cm,height=1.1cm]{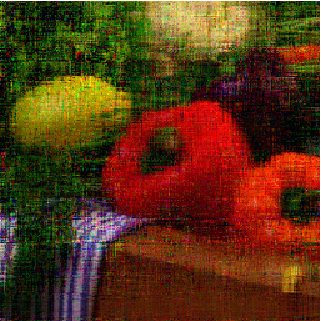}
	\end{minipage} 
	\hfill\\
	\begin{minipage}{1\textwidth}
		\centering
		\includegraphics[width=1.1cm,height=1.1cm]{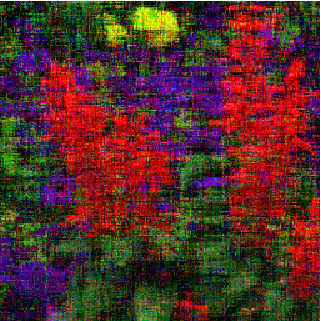}
	\end{minipage} 
	\hfill\\
		\begin{minipage}{1\textwidth}
		\centering
		\includegraphics[width=1.1cm,height=1.1cm]{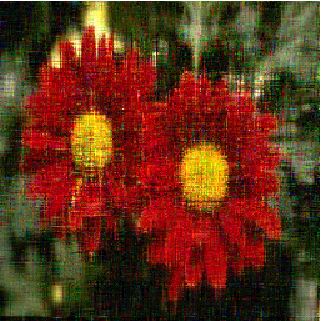}
	\end{minipage} 
	\hfill\\
	\begin{minipage}{1\textwidth}
		\centering
		\includegraphics[width=1.1cm,height=1.1cm]{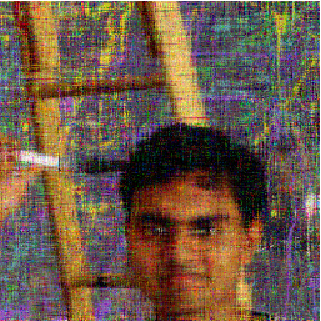}
	\end{minipage} 
	\hfill\\
	\begin{minipage}{1\textwidth}
		\centering
		\includegraphics[width=1.1cm,height=1.1cm]{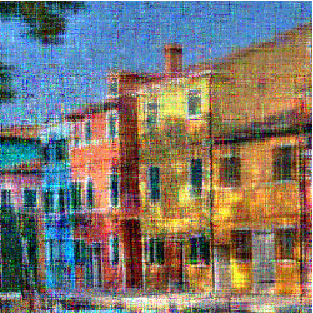}
	\end{minipage} 
	\hfill\\
	\begin{minipage}{1\textwidth}
		\centering
		\includegraphics[width=1.1cm,height=1.1cm]{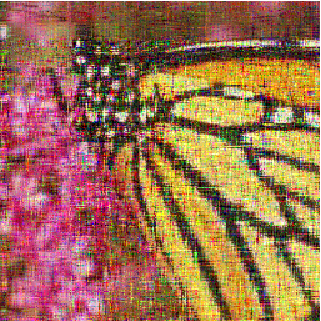}
	\end{minipage} 
	\hfill\\
		\begin{minipage}{1\textwidth}
		\centering
		\includegraphics[width=1.1cm,height=1.1cm]{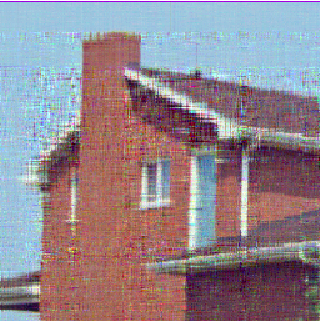}
	\end{minipage} 
	\hfill\\
	\caption*{(c)}
\end{minipage}
\begin{minipage}[h]{0.065\textwidth}
	\centering
	\begin{minipage}{1\textwidth}
		\centering
		\includegraphics[width=1.1cm,height=1.1cm]{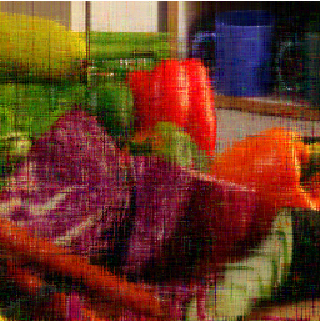}
	\end{minipage} 
	\hfill\\
	\begin{minipage}{1\textwidth}
		\centering
		\includegraphics[width=1.1cm,height=1.1cm]{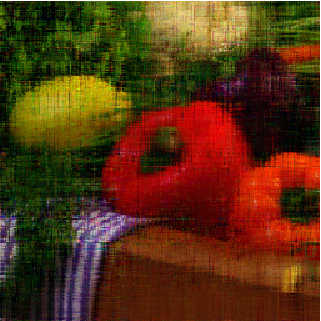}
	\end{minipage} 
	\hfill\\
	\begin{minipage}{1\textwidth}
		\centering
		\includegraphics[width=1.1cm,height=1.1cm]{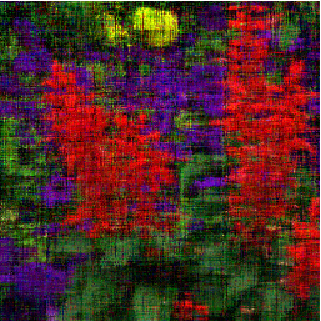}
	\end{minipage} 
	\hfill\\
		\begin{minipage}{1\textwidth}
		\centering
		\includegraphics[width=1.1cm,height=1.1cm]{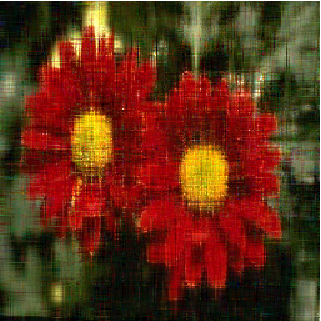}
	\end{minipage} 
	\hfill\\
	\begin{minipage}{1\textwidth}
		\centering
		\includegraphics[width=1.1cm,height=1.1cm]{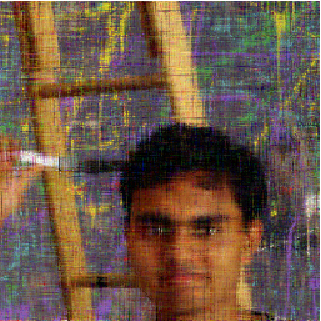}
	\end{minipage} 
	\hfill\\
	\begin{minipage}{1\textwidth}
		\centering
		\includegraphics[width=1.1cm,height=1.1cm]{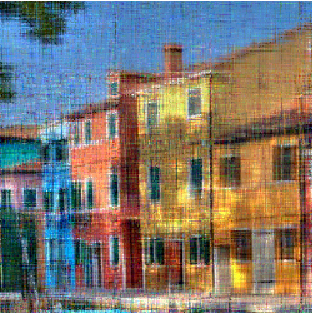}
	\end{minipage} 
	\hfill\\
	\begin{minipage}{1\textwidth}
		\centering
		\includegraphics[width=1.1cm,height=1.1cm]{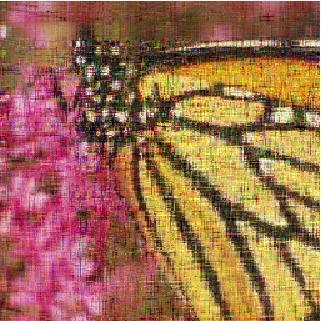}
	\end{minipage} 
	\hfill\\
		\begin{minipage}{1\textwidth}
		\centering
		\includegraphics[width=1.1cm,height=1.1cm]{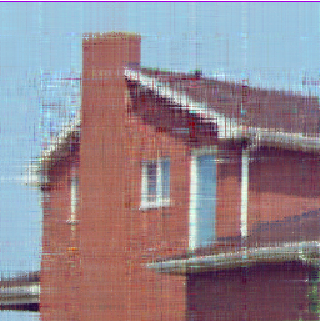}
	\end{minipage} 
	\hfill\\
	\caption*{(d)}
\end{minipage}
\begin{minipage}[h]{0.065\textwidth}
	\centering
	\begin{minipage}{1\textwidth}
		\centering
		\includegraphics[width=1.1cm,height=1.1cm]{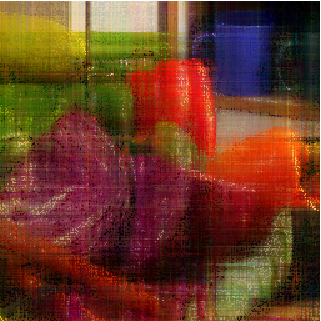}
	\end{minipage} 
	\hfill\\
	\begin{minipage}{1\textwidth}
		\centering
		\includegraphics[width=1.1cm,height=1.1cm]{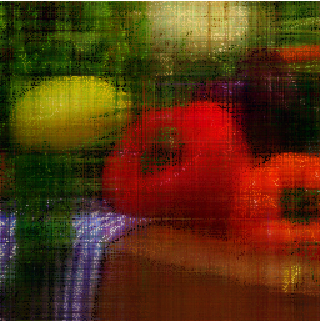}
	\end{minipage} 
	\hfill\\
	\begin{minipage}{1\textwidth}
		\centering
		\includegraphics[width=1.1cm,height=1.1cm]{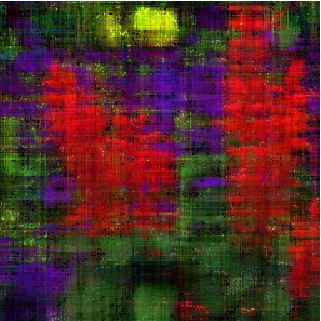}
	\end{minipage} 
	\hfill\\
	\begin{minipage}{1\textwidth}
		\centering
		\includegraphics[width=1.1cm,height=1.1cm]{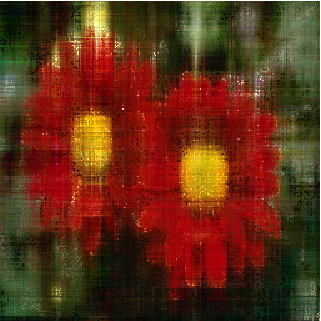}
	\end{minipage} 
	\hfill\\
	\begin{minipage}{1\textwidth}
		\centering
		\includegraphics[width=1.1cm,height=1.1cm]{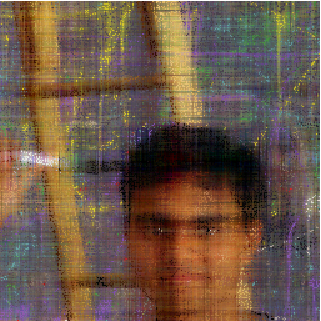}
	\end{minipage} 
	\hfill\\
	\begin{minipage}{1\textwidth}
		\centering
		\includegraphics[width=1.1cm,height=1.1cm]{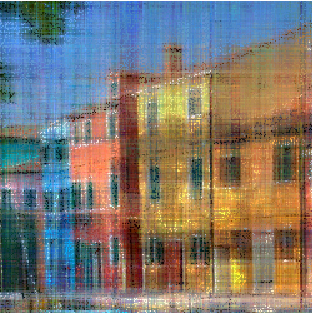}
	\end{minipage} 
	\hfill\\
	\begin{minipage}{1\textwidth}
		\centering
		\includegraphics[width=1.1cm,height=1.1cm]{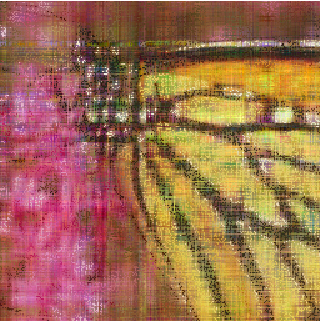}
	\end{minipage} 
	\hfill\\
	\begin{minipage}{1\textwidth}
		\centering
		\includegraphics[width=1.1cm,height=1.1cm]{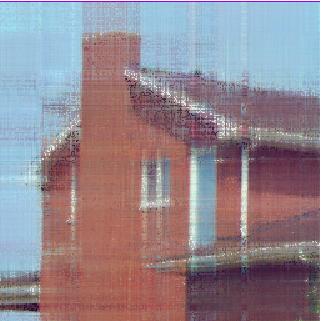}
	\end{minipage} 
	\hfill\\
	\caption*{(e)}
\end{minipage}
	\begin{minipage}[h]{0.065\textwidth}
		\centering
		\begin{minipage}{1\textwidth}
			\centering
			\includegraphics[width=1.1cm,height=1.1cm]{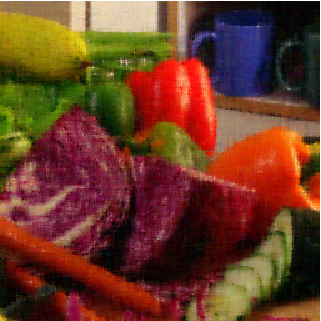}
		\end{minipage} 
		\hfill\\
		\begin{minipage}{1\textwidth}
			\centering
			\includegraphics[width=1.1cm,height=1.1cm]{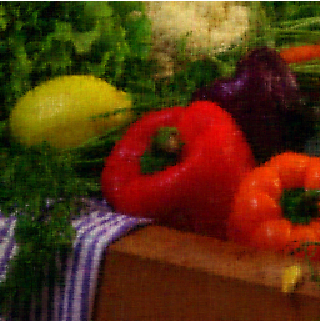}
		\end{minipage} 
		\hfill\\
		\begin{minipage}{1\textwidth}
			\centering
			\includegraphics[width=1.1cm,height=1.1cm]{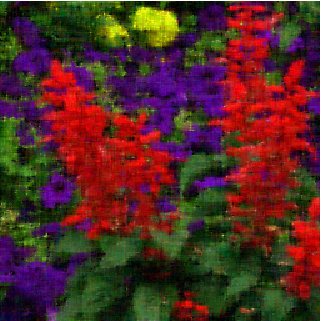}
		\end{minipage} 
		\hfill\\
		\begin{minipage}{1\textwidth}
			\centering
			\includegraphics[width=1.1cm,height=1.1cm]{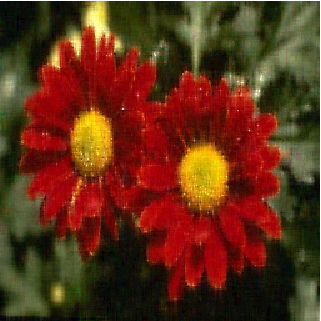}
		\end{minipage} 
		\hfill\\
		\begin{minipage}{1\textwidth}
			\centering
			\includegraphics[width=1.1cm,height=1.1cm]{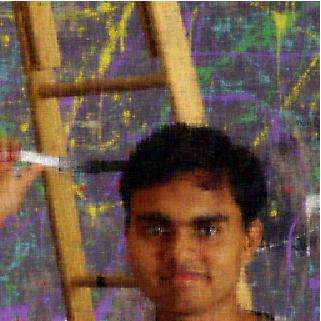}
		\end{minipage} 
		\hfill\\
		\begin{minipage}{1\textwidth}
			\centering
			\includegraphics[width=1.1cm,height=1.1cm]{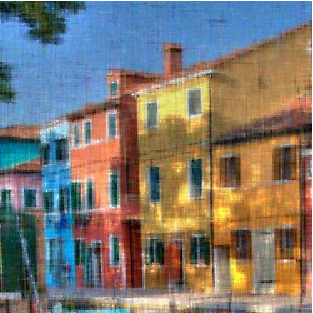}
		\end{minipage} 
		\hfill\\
		\begin{minipage}{1\textwidth}
			\centering
			\includegraphics[width=1.1cm,height=1.1cm]{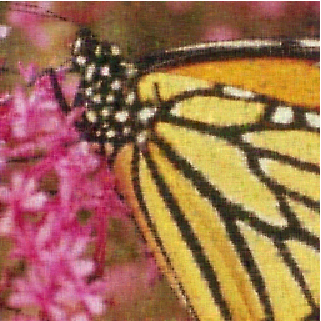}
		\end{minipage} 
		\hfill\\
		\begin{minipage}{1\textwidth}
			\centering
			\includegraphics[width=1.1cm,height=1.1cm]{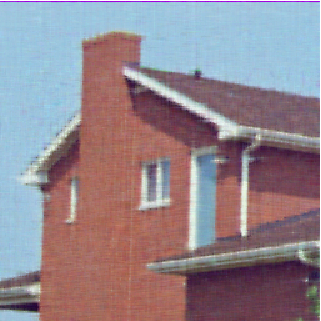}
		\end{minipage} 
		\hfill\\
		\caption*{(f)}
	\end{minipage}
		\begin{minipage}[h]{0.065\textwidth}
		\centering
		\begin{minipage}{1\textwidth}
			\centering
			\includegraphics[width=1.1cm,height=1.1cm]{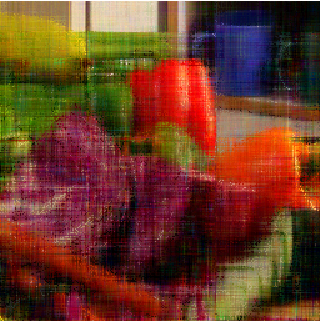}
		\end{minipage} 
		\hfill\\
		\begin{minipage}{1\textwidth}
			\centering
			\includegraphics[width=1.1cm,height=1.1cm]{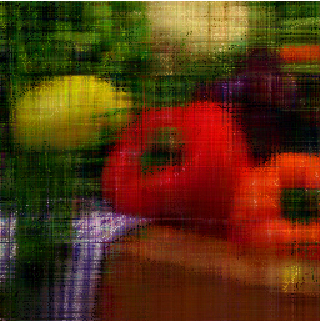}
		\end{minipage} 
		\hfill\\
		\begin{minipage}{1\textwidth}
			\centering
			\includegraphics[width=1.1cm,height=1.1cm]{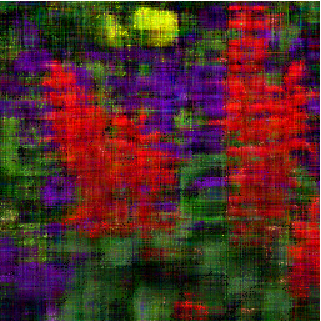}
		\end{minipage} 
		\hfill\\
		\begin{minipage}{1\textwidth}
			\centering
			\includegraphics[width=1.1cm,height=1.1cm]{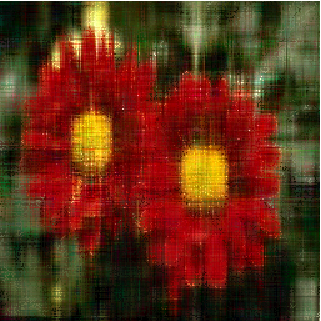}
		\end{minipage} 
		\hfill\\
		\begin{minipage}{1\textwidth}
			\centering
			\includegraphics[width=1.1cm,height=1.1cm]{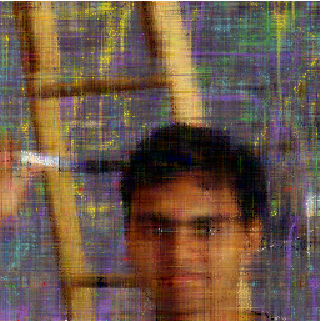}
		\end{minipage} 
		\hfill\\
		\begin{minipage}{1\textwidth}
			\centering
			\includegraphics[width=1.1cm,height=1.1cm]{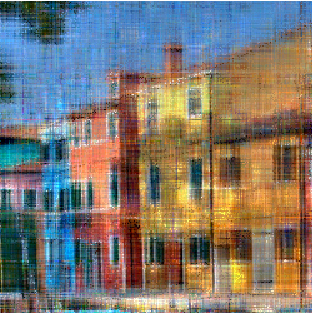}
		\end{minipage} 
		\hfill\\
		\begin{minipage}{1\textwidth}
			\centering
			\includegraphics[width=1.1cm,height=1.1cm]{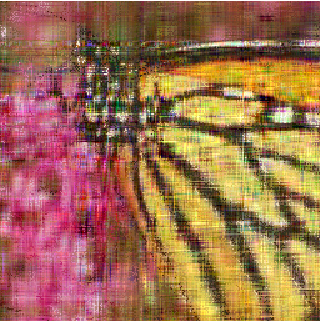}
		\end{minipage} 
		\hfill\\
		\begin{minipage}{1\textwidth}
			\centering
			\includegraphics[width=1.1cm,height=1.1cm]{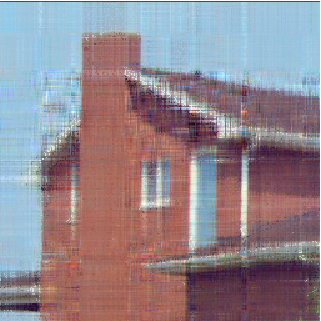}
		\end{minipage} 
		\hfill\\
		\caption*{(g)}
	\end{minipage}
	\begin{minipage}[h]{0.065\textwidth}
	\centering
	\begin{minipage}{1\textwidth}
		\centering
		\includegraphics[width=1.1cm,height=1.1cm]{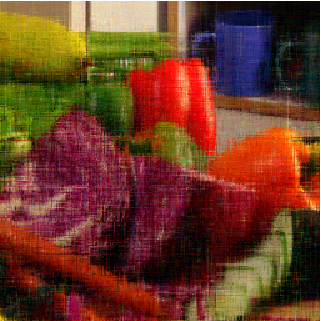}
	\end{minipage} 
	\hfill\\
	\begin{minipage}{1\textwidth}
		\centering
		\includegraphics[width=1.1cm,height=1.1cm]{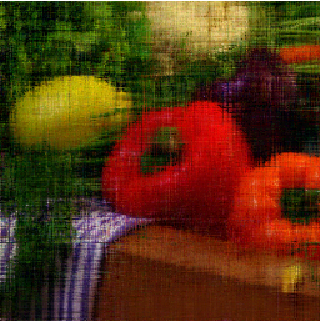}
	\end{minipage} 
	\hfill\\
	\begin{minipage}{1\textwidth}
		\centering
		\includegraphics[width=1.1cm,height=1.1cm]{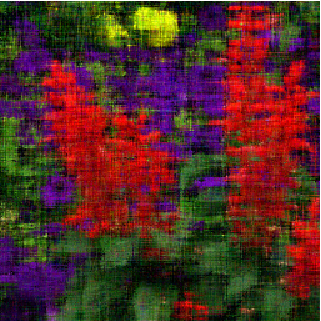}
	\end{minipage} 
	\hfill\\
	\begin{minipage}{1\textwidth}
		\centering
		\includegraphics[width=1.1cm,height=1.1cm]{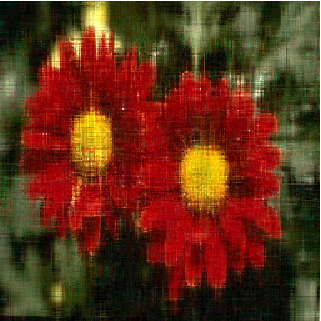}
	\end{minipage} 
	\hfill\\
	\begin{minipage}{1\textwidth}
		\centering
		\includegraphics[width=1.1cm,height=1.1cm]{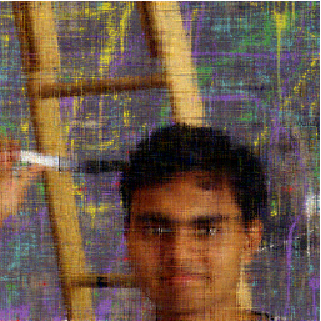}
	\end{minipage} 
	\hfill\\
	\begin{minipage}{1\textwidth}
		\centering
		\includegraphics[width=1.1cm,height=1.1cm]{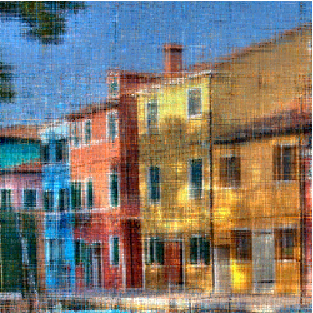}
	\end{minipage} 
	\hfill\\
	\begin{minipage}{1\textwidth}
		\centering
		\includegraphics[width=1.1cm,height=1.1cm]{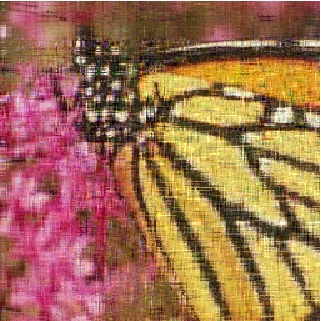}
	\end{minipage} 
	\hfill\\
	\begin{minipage}{1\textwidth}
		\centering
		\includegraphics[width=1.1cm,height=1.1cm]{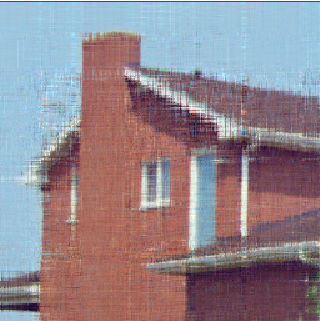}
	\end{minipage} 
	\hfill\\
	\caption*{(h)}
\end{minipage}
	\begin{minipage}[h]{0.065\textwidth}
	\centering
	\begin{minipage}{1\textwidth}
		\centering
		\includegraphics[width=1.1cm,height=1.1cm]{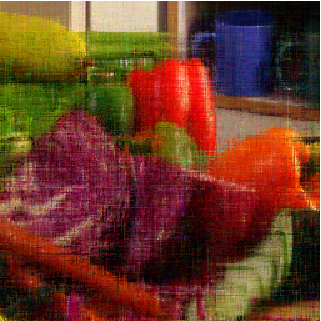}
	\end{minipage} 
	\hfill\\
	\begin{minipage}{1\textwidth}
		\centering
		\includegraphics[width=1.1cm,height=1.1cm]{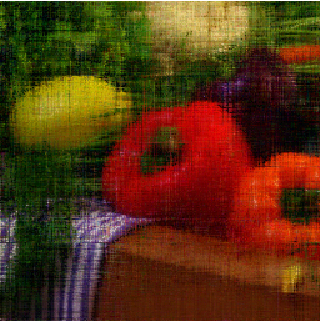}
	\end{minipage} 
	\hfill\\
	\begin{minipage}{1\textwidth}
		\centering
		\includegraphics[width=1.1cm,height=1.1cm]{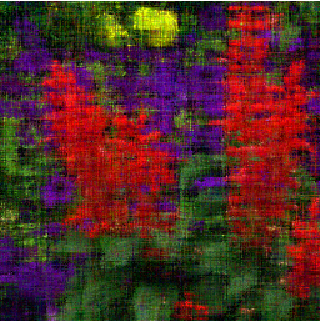}
	\end{minipage} 
	\hfill\\
	\begin{minipage}{1\textwidth}
		\centering
		\includegraphics[width=1.1cm,height=1.1cm]{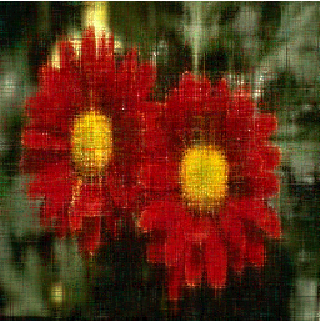}
	\end{minipage} 
	\hfill\\
	\begin{minipage}{1\textwidth}
		\centering
		\includegraphics[width=1.1cm,height=1.1cm]{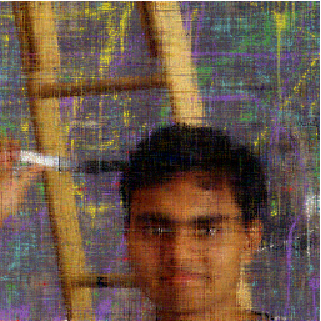}
	\end{minipage} 
	\hfill\\
	\begin{minipage}{1\textwidth}
		\centering
		\includegraphics[width=1.1cm,height=1.1cm]{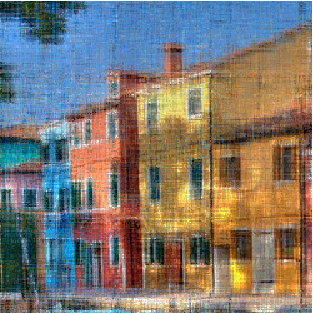}
	\end{minipage} 
	\hfill\\
	\begin{minipage}{1\textwidth}
		\centering
		\includegraphics[width=1.1cm,height=1.1cm]{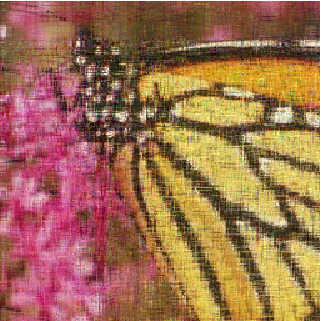}
	\end{minipage} 
	\hfill\\
	\begin{minipage}{1\textwidth}
		\centering
		\includegraphics[width=1.1cm,height=1.1cm]{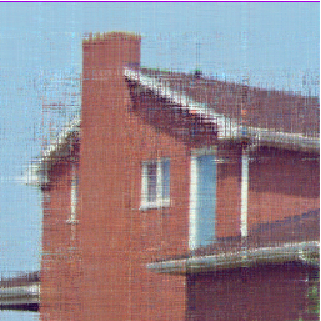}
	\end{minipage} 
	\hfill\\
	\caption*{(i)}
\end{minipage}
	\begin{minipage}[h]{0.065\textwidth}
	\centering
	\begin{minipage}{1\textwidth}
		\centering
		\includegraphics[width=1.1cm,height=1.1cm]{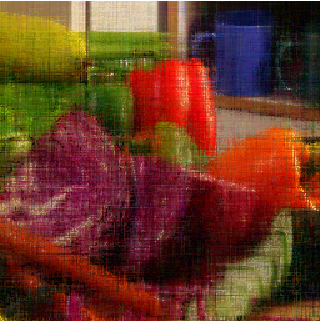}
	\end{minipage} 
	\hfill\\
	\begin{minipage}{1\textwidth}
		\centering
		\includegraphics[width=1.1cm,height=1.1cm]{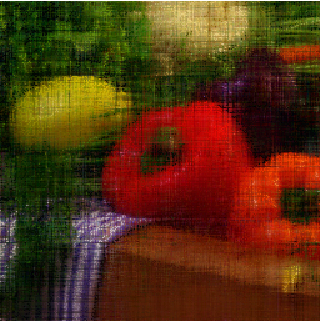}
	\end{minipage} 
	\hfill\\
	\begin{minipage}{1\textwidth}
		\centering
		\includegraphics[width=1.1cm,height=1.1cm]{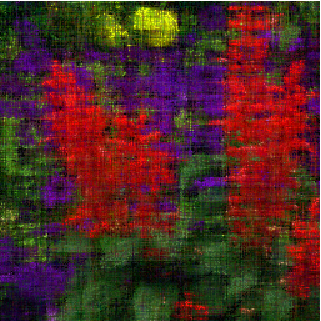}
	\end{minipage} 
	\hfill\\
	\begin{minipage}{1\textwidth}
		\centering
		\includegraphics[width=1.1cm,height=1.1cm]{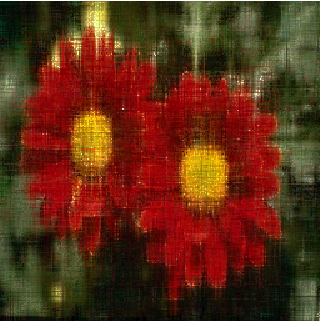}
	\end{minipage} 
	\hfill\\
	\begin{minipage}{1\textwidth}
		\centering
		\includegraphics[width=1.1cm,height=1.1cm]{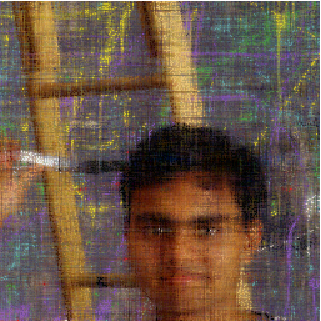}
	\end{minipage} 
	\hfill\\
	\begin{minipage}{1\textwidth}
		\centering
		\includegraphics[width=1.1cm,height=1.1cm]{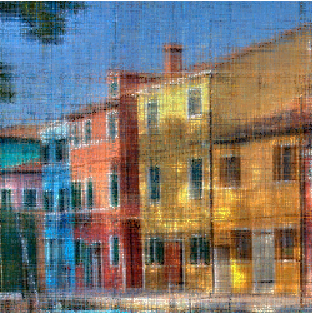}
	\end{minipage} 
	\hfill\\
	\begin{minipage}{1\textwidth}
		\centering
		\includegraphics[width=1.1cm,height=1.1cm]{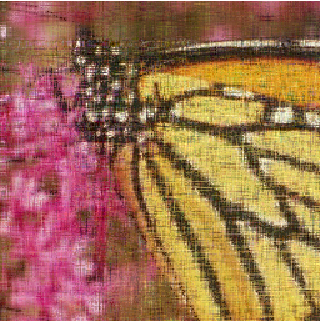}
	\end{minipage} 
	\hfill\\
	\begin{minipage}{1\textwidth}
		\centering
		\includegraphics[width=1.1cm,height=1.1cm]{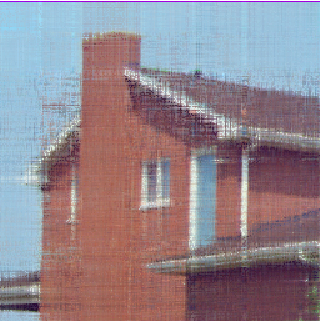}
	\end{minipage} 
	\hfill\\
	\caption*{(j)}
\end{minipage}
\begin{minipage}[h]{0.065\textwidth}
	\centering
	\begin{minipage}{1\textwidth}
		\centering
		\includegraphics[width=1.1cm,height=1.1cm]{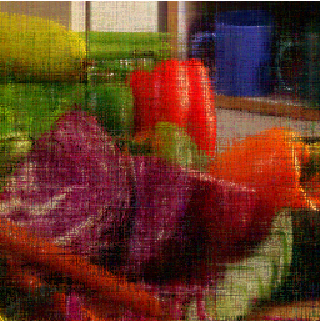}
	\end{minipage} 
	\hfill\\
	\begin{minipage}{1\textwidth}
		\centering
		\includegraphics[width=1.1cm,height=1.1cm]{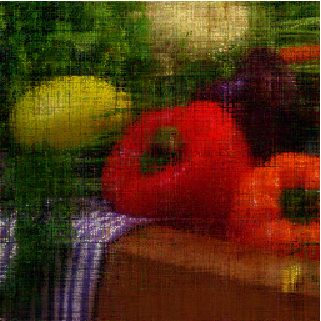}
	\end{minipage} 
	\hfill\\
	\begin{minipage}{1\textwidth}
		\centering
		\includegraphics[width=1.1cm,height=1.1cm]{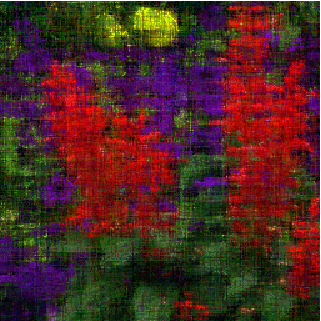}
	\end{minipage} 
	\hfill\\
	\begin{minipage}{1\textwidth}
		\centering
		\includegraphics[width=1.1cm,height=1.1cm]{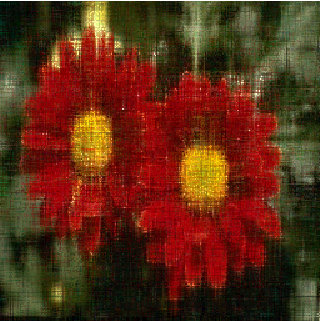}
	\end{minipage} 
	\hfill\\
	\begin{minipage}{1\textwidth}
		\centering
		\includegraphics[width=1.1cm,height=1.1cm]{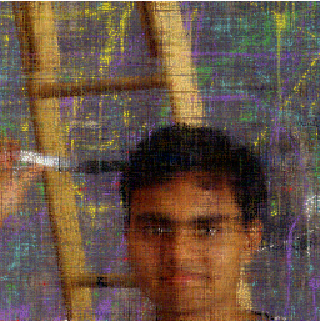}
	\end{minipage} 
	\hfill\\
	\begin{minipage}{1\textwidth}
		\centering
		\includegraphics[width=1.1cm,height=1.1cm]{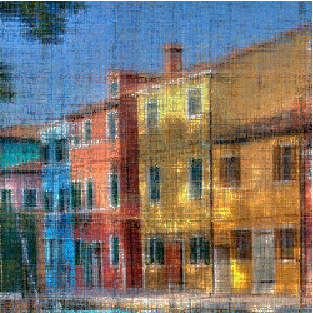}
	\end{minipage} 
	\hfill\\
	\begin{minipage}{1\textwidth}
		\centering
		\includegraphics[width=1.1cm,height=1.1cm]{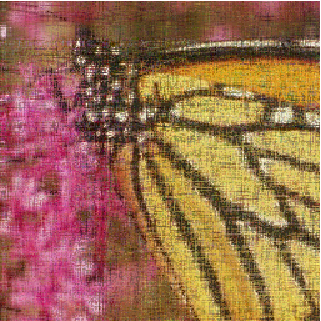}
	\end{minipage} 
	\hfill\\
	\begin{minipage}{1\textwidth}
		\centering
		\includegraphics[width=1.1cm,height=1.1cm]{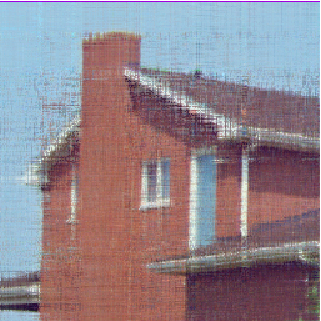}
	\end{minipage} 
	\hfill\\
	\caption*{(k)}
\end{minipage}
	\begin{minipage}[h]{0.065\textwidth}
		\centering
		\begin{minipage}{1\textwidth}
			\centering
			\includegraphics[width=1.1cm,height=1.1cm]{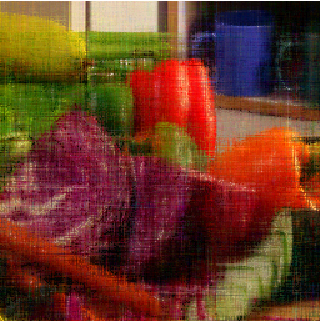}
		\end{minipage} 
		\hfill\\
		\begin{minipage}{1\textwidth}
			\centering
			\includegraphics[width=1.1cm,height=1.1cm]{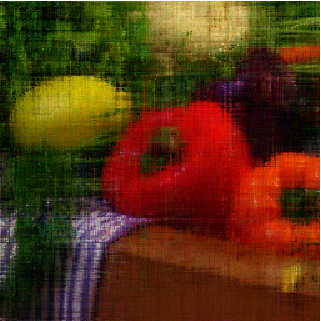}
		\end{minipage} 
		\hfill\\
		\begin{minipage}{1\textwidth}
			\centering
			\includegraphics[width=1.1cm,height=1.1cm]{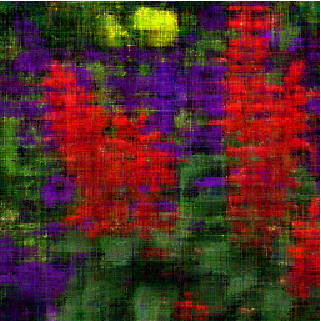}
		\end{minipage} 
		\hfill\\
		\begin{minipage}{1\textwidth}
			\centering
			\includegraphics[width=1.1cm,height=1.1cm]{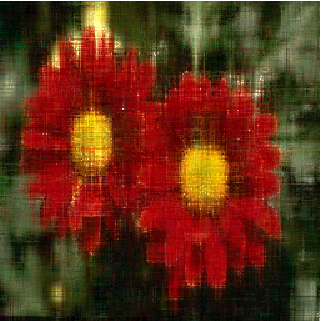}
		\end{minipage} 
		\hfill\\
		\begin{minipage}{1\textwidth}
			\centering
			\includegraphics[width=1.1cm,height=1.1cm]{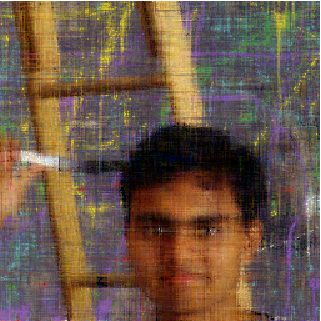}
		\end{minipage} 
		\hfill\\
		\begin{minipage}{1\textwidth}
			\centering
			\includegraphics[width=1.1cm,height=1.1cm]{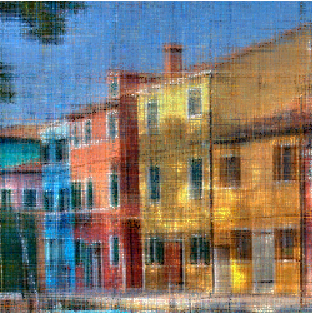}
		\end{minipage} 
		\hfill\\
		\begin{minipage}{1\textwidth}
			\centering
			\includegraphics[width=1.1cm,height=1.1cm]{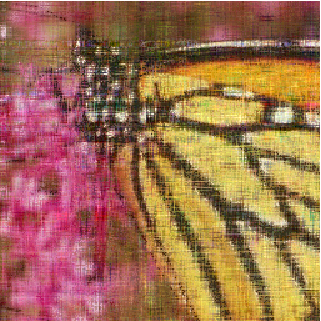}
		\end{minipage} 
		\hfill\\
		\begin{minipage}{1\textwidth}
			\centering
			\includegraphics[width=1.1cm,height=1.1cm]{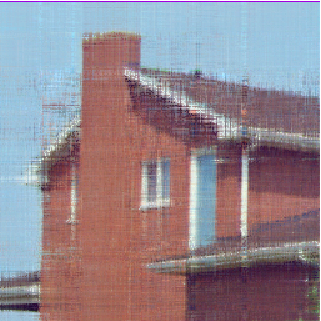}
		\end{minipage} 
		\hfill\\
		\caption*{(l)}
	\end{minipage}
\begin{minipage}[h]{0.065\textwidth}
	\centering
	\begin{minipage}{1\textwidth}
		\centering
		\includegraphics[width=1.1cm,height=1.1cm]{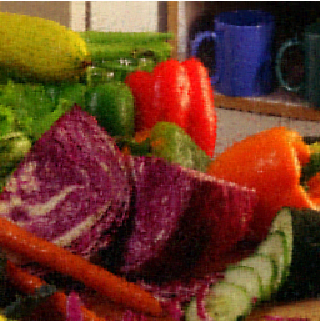}
	\end{minipage} 
	\hfill\\
	\begin{minipage}{1\textwidth}
		\centering
		\includegraphics[width=1.1cm,height=1.1cm]{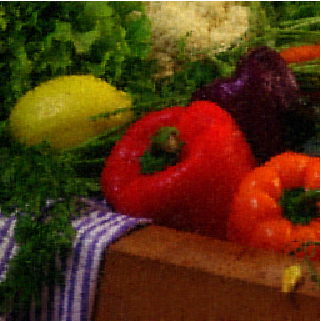}
	\end{minipage} 
	\hfill\\
	\begin{minipage}{1\textwidth}
		\centering
		\includegraphics[width=1.1cm,height=1.1cm]{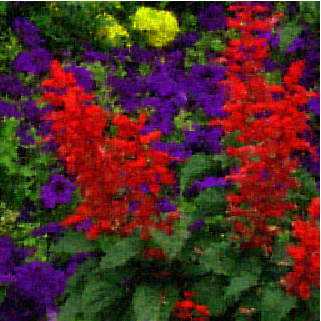}
	\end{minipage} 
	\hfill\\
	\begin{minipage}{1\textwidth}
		\centering
		\includegraphics[width=1.1cm,height=1.1cm]{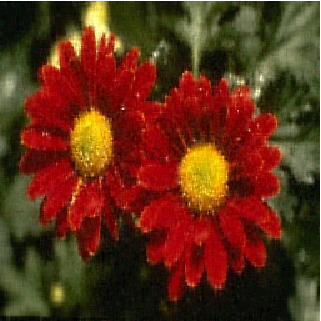}
	\end{minipage} 
	\hfill\\
	\begin{minipage}{1\textwidth}
		\centering
		\includegraphics[width=1.1cm,height=1.1cm]{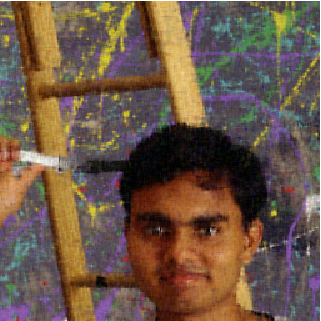}
	\end{minipage} 
	\hfill\\
	\begin{minipage}{1\textwidth}
		\centering
		\includegraphics[width=1.1cm,height=1.1cm]{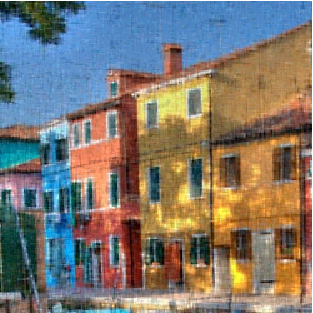}
	\end{minipage} 
	\hfill\\
	\begin{minipage}{1\textwidth}
		\centering
		\includegraphics[width=1.1cm,height=1.1cm]{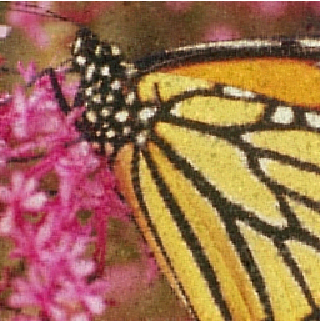}
	\end{minipage} 
	\hfill\\
	\begin{minipage}{1\textwidth}
		\centering
		\includegraphics[width=1.1cm,height=1.1cm]{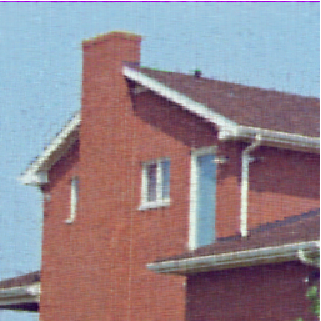}
	\end{minipage} 
	\hfill\\
	\caption*{(m)}
\end{minipage}
	\hfill\\
	\caption{(a) Ground truth. From top to bottom: Image(1)-Image(8). (b) Observation (MR=$75\%$). (c)-(m) are the restored results of WNNM, MC-NC, IRLNM-QR, TNN-SR, QLNF, TQLNA, LRQA-G, LRQMC, QLNM-QQR, IRQLNM-QQR, and QLNM-QQR-SR, respectively.}
	\label{fig:randommissing} 
\end{figure*}

\begin{table*}[t]
	\caption{A comparison of quantitative assessment indices (PSNR/SSIM) across different methods on the set of eight color images.}
	\label{tablecolornature}
	\centering
	\resizebox{\textwidth}{!}{
		\begin{tabular}{|c|c|c|c|c|c|c|c|c|c|c|c|}		
			\hline
			Methods:& WNNM  &  MC-NC  &IRLNM-QR  & TNN-SR & QLNF  &TQLNA & LRQA-G & LRQMC   & QLNM-QQR &IRQLNM-QQR & QLNM-QQR-SR	
			\\ \toprule
			\hline
			Images:  &\multicolumn{11}{c|}{${\rm{MR}}=50\%$}\\
			\hline				
			Image(1) & 23.386/0.848	&	26.939/0.935	&	24.336/0.888	&	28.657/0.955	&	24.223/0.862	&	27.020/0.909	&	26.561/0.927	&	27.007/0.936	&	26.208/0.923	&	26.557/0.930	&	\textbf{29.156}/\textbf{0.959}\\
			Image(2) & 24.837/0.861	&	28.233/0.937	&	25.767/0.899	&	30.070/0.957	&	25.312/0.872	&	28.260/0.921	&	27.748/0.928	&	28.319/0.938	&	27.442/0.926	&	27.771/	0.930	&	\textbf{30.487}/\textbf{0.960}\\
		    Image(3) &	19.995/0.759	&	23.562/0.878	&	22.155/0.841	&	25.357/0.918	&	21.659/0.814	&	23.927/0.878	&	23.376/0.873	&	23.306/0.877	&	23.057/0.869	&	23.603/0.881	&	\textbf{25.767}/\textbf{0.924}\\
			Image(4) & 24.231/0.877	&	27.586/0.950	&	25.085/0.899	&	29.205/0.965	&	25.007/0.853	&	27.497/0.918	&	27.267/0.944	&	27.499/0.950	&	26.904/0.940	&	27.223/0.945	&	\textbf{29.507}/\textbf{0.967}\\
			Image(5) & 22.795/0.731	&	26.294/0.862	&	24.214/0.811	&	28.030/0.903	&	23.840/0.749	&	26.411/0.826	&	25.985/0.850	&	26.448/0.865	&	25.693/0.846	&	26.097/0.859	&	\textbf{28.681}/\textbf{0.915}\\
			Image(6) & 22.246/0.883	&	25.686/0.943	&	22.802/0.898	&	27.046/0.958	&	22.352/0.884	&	25.734/0.941	&	25.266/0.937	&	25.658/0.942	&	24.790/0.931	&	25.081/0.936	&	\textbf{27.417}/\textbf{0.961}\\
			Image(7) & 21.406/0.877	&	24.439/0.938	&	21.018/0.886	&	27.303/0.968	&	20.249/0.861	&	24.358/0.933	&	23.969/0.931	&	24.319/0.937	&	23.433/0.926	&	23.406/0.926	&	\textbf{27.933}/\textbf{0.971}\\
			Image(8) & 27.822/0.920	&	29.482/0.956	&	26.564/0.925	&	33.175/0.978	&	26.448/0.913	&	30.684/0.959	&	29.925/0.957	&	30.434/0.962	&	29.493/0.954	&	29.284/0.954	&	\textbf{33.635}/\textbf{0.980}\\
				\hline
			Aver. &  23.340 	&	26.528	&	23.993	&	28.605	&	23.636	&	26.736	&	26.262	&	26.624	&	25.878	&	26.128	&	\textbf{29.073}	\\ \toprule
			\hline
			Images:  &\multicolumn{11}{c|}{${\rm{MR}}=65\%$}\\
			\hline
		    Image(1) & 20.815/0.771	&	24.155/0.884	&	22.029/0.834	&	26.489/0.931	&	22.677/0.823	&	24.292/0.859	&	23.823/0.874	&	24.103/0.886	&	23.323/0.866	&	23.841/0.878	&	\textbf{26.996}/\textbf{0.938}\\
		    Image(2) & 22.196/0.784	&	25.546/0.889	&	23.600/0.848	&	27.936/0.932	&	23.867/0.831	&	25.613/0.871	&	25.180/0.879	&	25.652/0.892	&	24.644/0.873	&	25.233/0.882	&	\textbf{28.422}/\textbf{0.937}\\
		    Image(3) & 17.533/0.638	&	20.605/0.777	&	20.154/0.758	&	23.277/0.870	&	20.153/0.743	&	21.465/0.798	&	20.907/0.786	&	21.249/0.806	&	20.521/0.776	&	21.193/0.801	&	\textbf{23.752}/\textbf{0.883}\\
		    Image(4) & 21.791/0.786	&	24.960/0.900	&	22.829/0.844	&	27.225/0.944	&	23.498/0.814	&	24.980/0.858	&	24.668/0.892	&	24.783/0.901	&	24.114/0.883	&	24.670/0.896	&	\textbf{27.611}/\textbf{0.950}\\
		    Image(5) & 20.130/0.607	&	23.535/0.763	&	22.181/0.719	&	25.765/0.844	&	22.342/0.669	&	23.759/0.728	&	23.405/0.754	&	23.609/0.770	&	22.957/0.743	&	23.620/0.769	&	\textbf{26.385}/\textbf{0.862}\\
		    Image(6) & 19.226/0.791	&	22.618/0.891	&	20.442/0.835	&	24.415/0.925	&	20.705/0.836	&	22.532/0.884	&	22.154/0.878	&	21.797/0.872	&	21.540/0.865	&	22.122/0.880	&	\textbf{24.854}/\textbf{0.932}\\
		    Image(7) & 18.091/0.776	&	21.101/0.877	&	18.901/0.834	&	24.607/0.943	&	18.651/0.813	&	21.090/0.874	&	20.749/0.871	&	20.333/0.868	&	20.123/0.861	&	20.501/0.870	&	\textbf{25.185}/\textbf{0.948}\\
		    Image(8) & 24.594/0.856	&	27.095/0.923	&	24.384/0.883	&	30.836/0.963	&	25.028/0.879	&	27.380/0.917	&	26.893/0.917	&	26.272/0.914	&	26.237/0.909	&	26.525/0.917	&	\textbf{31.369}/\textbf{0.967}\\
			\hline
			Aver. & 20.547	&	23.702	&	21.815	&	26.319	&	22.115	&	23.889	&	23.472	&	23.475	&	22.932	&	23.463	&	\textbf{26.822} \\ \toprule
			\hline	
			Images  &\multicolumn{11}{c|}{${\rm{MR}}=75\%$}\\
			\hline		
			Image(1) & 18.794/0.689	&	21.882/0.818	&	20.580/0.788	&	25.004/0.907	&	21.574/0.784	&	22.470/0.810	&	22.024/0.822	&	21.788/0.826	&	21.269/0.806	&	22.072/0.830	&	\textbf{25.455}/\textbf{0.914}\\
			Image(2) & 19.965/0.704	&	23.122/0.825	&	22.028/0.801	&	26.621/0.909	&	22.547/0.787	&	23.565/0.815	&	23.189/0.825	&	22.931/0.828	&	22.537/0.816	&	23.297/0.830	&	\textbf{27.090}/\textbf{0.917}\\
			Image(3) & 15.707/0.529	&	18.252/0.652	&	18.884/0.683	&	21.926/0.822	&	19.082/0.678	&	19.815/0.719	&	19.287/0.701	&	19.290/0.708	&	18.870/0.684	&	19.577/0.719	&	\textbf{22.431}/\textbf{0.842}\\
			Image(4) & 19.676/0.689	&	22.652/0.831	&	21.256/0.795	&	25.811/0.926	&	22.054/0.772	&	22.977/0.800	&	22.731/0.836	&	22.207/0.835	&	21.962/0.820	&	22.694/0.842	&	\textbf{26.234}/\textbf{0.937}\\
			Image(5) & 18.105/0.495	&	21.118/0.646	&	20.654/0.637	&	24.263/0.789	&	21.185/0.600	&	21.800/0.630	&	21.398/0.653	&	21.153/0.657	&	20.848/0.637	&	21.703/0.674	&	\textbf{24.798}/\textbf{0.812}\\
			Image(6) & 16.971/0.689	&	19.952/0.812	&	18.926/0.776	&	22.889/0.895	&	19.517/0.791	&	20.480/0.823	&	20.153/0.816	&	20.074/0.817	&	19.389/0.793	&	20.183/0.820	&	\textbf{23.342}/\textbf{0.906}\\
			Image(7) & 15.317/0.652	&	17.749/0.778	&	17.387/0.787	&	22.758/0.918	&	17.411/0.768	&	18.731/0.810	&	18.528/0.810	&	18.697/0.821	&	17.850/0.794	&	18.480/0.814	&	\textbf{23.406}/\textbf{0.928}\\
			Image(8) & 21.725/0.769	&	24.816/0.881	&	22.782/0.843	&	28.964/0.946	&	23.637/0.844	&	24.673/0.866	&	24.580/0.873	&	23.961/0.868	&	23.839/0.860	&	24.386/0.875	&	\textbf{29.373}/\textbf{0.951}\\	
			\hline
			Aver. & 18.283	&	21.193	&	20.312	&	24.780	&	20.876	&	21.814	&	21.486	&21.263	&	20.821	&	21.549	&	\textbf{25.266} \\ \toprule	
			\hline
			Images  &\multicolumn{11}{c|}{${\rm{MR}}=85\%$}\\
			\hline
		    Image(1) & 16.018/0.562	&	17.777/0.644	&	18.637/0.713	&	23.413/0.871	&	19.658/0.713	&	19.949/0.726	&	19.720/0.738	&	19.042/0.731	&	18.771/0.710	&	19.758/0.748	&	\textbf{23.836}/\textbf{0.883}\\
		    Image(2) & 17.665/0.603	&	19.079/0.684	&	19.914/0.727	&	24.561/0.866	&	20.541/0.711	&	21.201/0.737	&	20.895/0.747	&	20.059/0.743	&	19.784/0.725	&	20.901/0.750	&	\textbf{25.064}/\textbf{0.878}\\
		    Image(3) & 13.367/0.377	&	15.852/0.485	&	17.084/0.560	&	20.328/0.754	&	17.103/0.549	&	17.801/0.598	&	17.361/0.577	&	17.320/0.578	&	16.890/0.546	&	17.583/0.596	&	\textbf{20.768}/\textbf{0.779}\\
		    Image(4) & 17.216/0.569	&	18.991/0.682	&	19.428/0.719	&	24.307/0.896	&	20.149/0.699	&	20.550/0.711	&	20.520/0.749	&	19.889/0.751	&	19.572/0.725	&	20.461/0.757	&	\textbf{24.744}/\textbf{0.910}\\
		    Image(5) & 15.578/0.349	&	17.502/0.438	&	18.862/0.519	&	22.712/0.708	&	19.465/0.490	&	19.703/0.502	&	19.399/0.531	&	18.945/0.527	&	18.587/0.503	&	19.683/0.551	&	\textbf{23.152}/\textbf{0.734}\\
		    Image(6) & 14.105/0.523	&	15.674/0.599	&	16.885/0.672	&	20.992/0.842	&	17.322/0.685	&	17.635/0.703	&	17.687/0.710	&	17.711/0.713	&	16.929/0.674	&	17.772/0.718	&	\textbf{21.384}/\textbf{0.858}\\
		    Image(7) & 12.230/0.471	&	14.049/0.593	&	15.408/0.705	&	20.556/0.874	&	14.970/0.653	&	15.751/0.690	&	16.034/0.714	&	15.744/0.721	&	15.507/0.699	&	16.055/0.723	&	\textbf{21.184}/\textbf{0.889}\\
		    Image(8) & 18.320/0.628	&	20.356/0.738	&	19.730/0.730	&	26.484/0.916	&	21.107/0.762	&	21.554/0.776	&	21.573/0.788	&	21.586/0.798	&	20.655/0.760	&	21.603/0.797	&	\textbf{26.825}/\textbf{0.922}\\
			\hline
			Aver. & 15.562	&	17.410 	&	18.244	&	22.919	&	18.789	&	19.268	&	19.149	&	18.787	&	18.337	&	19.227	&	\textbf{23.370}	\\ \toprule
	\end{tabular}}
\end{table*}	

\begin{figure}[htbp]
	\centering
	\tiny
	\begin{minipage}[t]{0.11\textwidth}
		\centerline{\includegraphics[width=1.8cm,height=1.8cm]{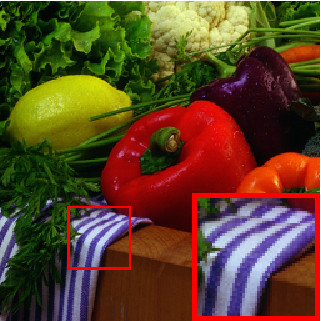}}
		\centerline{(a) Ground truth}   
	\end{minipage}

	\begin{minipage}[t]{0.11\textwidth}
	\centerline{\includegraphics[width=1.8cm,height=1.8cm]{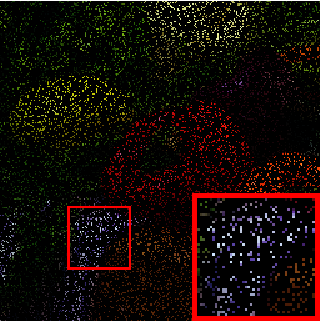}}
	\centerline{(b) Observation}
\end{minipage}
	\begin{minipage}[t]{0.11\textwidth}
		\centerline{\includegraphics[width=1.8cm,height=1.8cm]{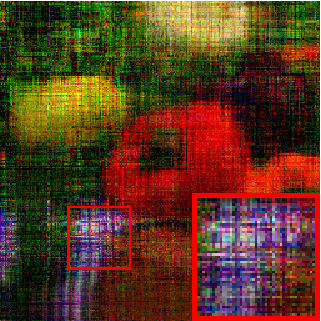}}
		\centerline{(c) WNNM}
	\end{minipage}
	\begin{minipage}[t]{0.11\textwidth}
		\centerline{\includegraphics[width=1.8cm,height=1.8cm]{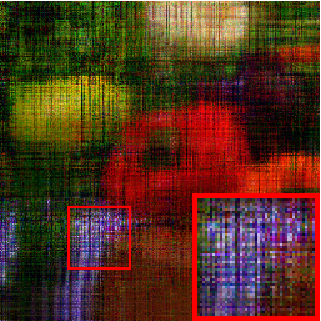}}
		\centerline{(d) MC-NC}
	\end{minipage}	
	\begin{minipage}[t]{0.11\textwidth}
		\centerline{\includegraphics[width=1.8cm,height=1.8cm]{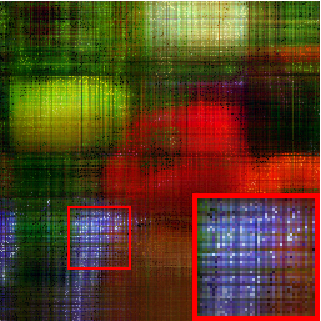}}
		\centerline{(e) IRLNM-QR}
	\end{minipage}
	\begin{minipage}[t]{0.11\textwidth}
		\centerline{\includegraphics[width=1.8cm,height=1.8cm]{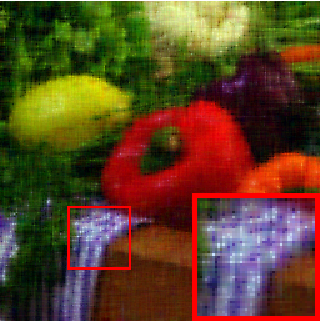}}
		\centerline{(f) TNN-SR}
	\end{minipage}
	\begin{minipage}[t]{0.11\textwidth}
		\centerline{\includegraphics[width=1.8cm,height=1.8cm]{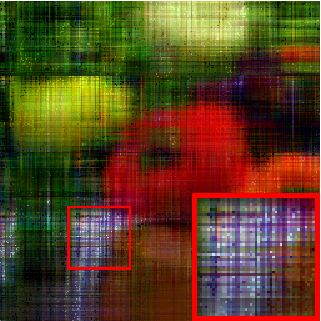}}
		\centerline{(g) QLNF}   
	\end{minipage}

	\begin{minipage}[t]{0.11\textwidth}
		\centerline{\includegraphics[width=1.8cm,height=1.8cm]{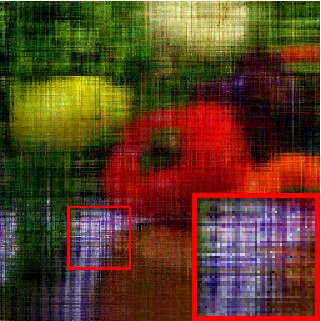}}
		\centerline{(h) TQLNA}
	\end{minipage}
	\begin{minipage}[t]{0.11\textwidth}
		%	\centering
		\centerline{\includegraphics[width=1.8cm,height=1.8cm]{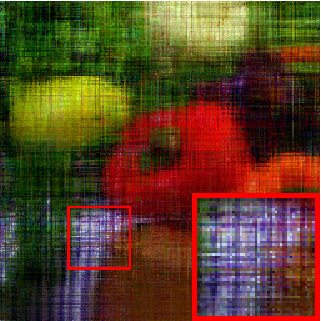}}
		\centerline{(i) LRQA-G}
	\end{minipage}
	\begin{minipage}[t]{0.11\textwidth}
		%	\centering
		\centerline{\includegraphics[width=1.8cm,height=1.8cm]{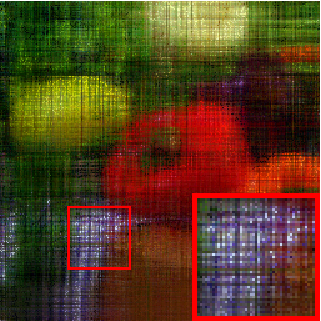}}
		\centerline{(j) LRQMC}
	\end{minipage}
	\begin{minipage}[t]{0.11\textwidth}
		%	\centering
		\centerline{\includegraphics[width=1.8cm,height=1.8cm]{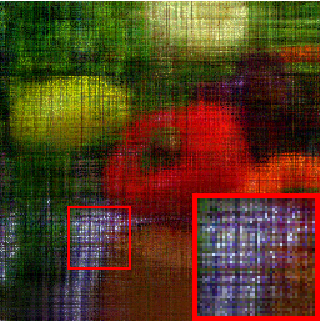}}
		\centerline{(k) QLNM-QQR}
	\end{minipage}
	\begin{minipage}[t]{0.11\textwidth}
		%	\centering
		\centerline{\includegraphics[width=1.8cm,height=1.8cm]{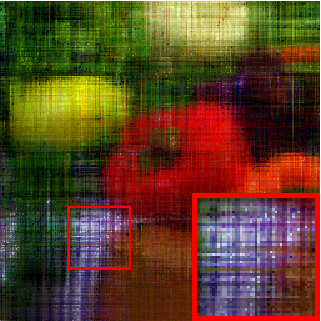}}
		\centerline{(l) IRQLNM-QQR}
	\end{minipage}
	\begin{minipage}[t]{0.11\textwidth}
		%  \centering
		\centerline{\includegraphics[width=1.8cm,height=1.8cm]{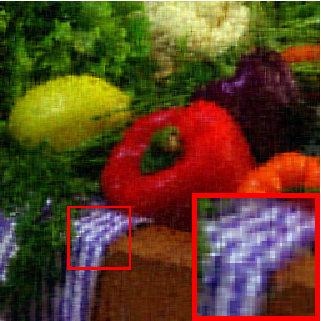}}
		\centerline{(m) QLNM-QQR-SR}
	\end{minipage}

	\begin{minipage}[t]{0.11\textwidth}
		\centering
		\centerline{\includegraphics[width=1.8cm,height=1.8cm]{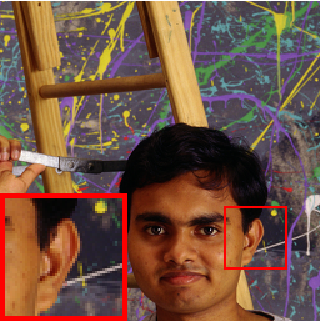}}
		\centerline{(a) Ground truth}   
	\end{minipage}

	\begin{minipage}[t]{0.11\textwidth}
	\centering
	\centerline{\includegraphics[width=1.8cm,height=1.8cm]{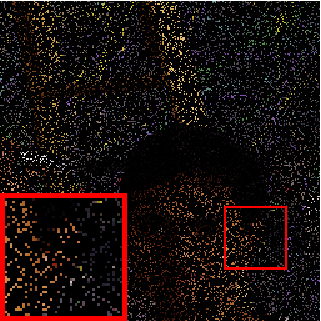}}
	\centerline{(b) Observation}
\end{minipage}
	\begin{minipage}[t]{0.11\textwidth}
		\centering
		\centerline{\includegraphics[width=1.8cm,height=1.8cm]{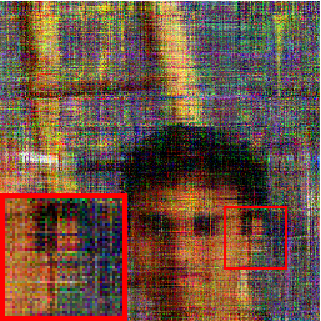}}
		\centerline{(c) WNNM}
	\end{minipage}
	\begin{minipage}[t]{0.11\textwidth}
		\centering
		\centerline{\includegraphics[width=1.8cm,height=1.8cm]{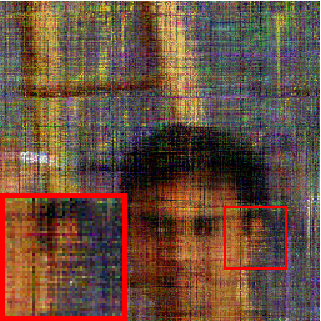}}
		\centerline{(d) MC-NC}
	\end{minipage}
	\begin{minipage}[t]{0.11\textwidth}
		\centering
		\centerline{\includegraphics[width=1.8cm,height=1.8cm]{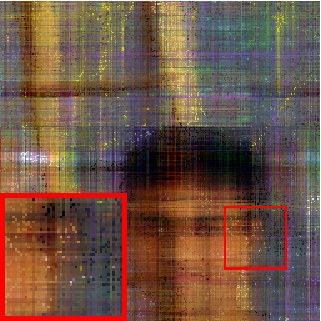}}
		\centerline{(e) IRLNM-QR}
	\end{minipage}
	\begin{minipage}[t]{0.11\textwidth}
		\centering
		\centerline{\includegraphics[width=1.8cm,height=1.8cm]{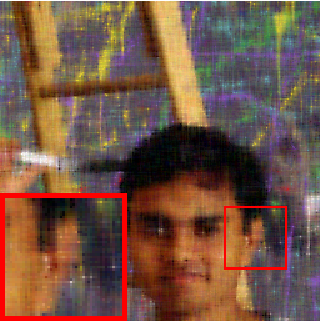}}
		\centerline{(f) TNN-SR}
	\end{minipage}
	\begin{minipage}[t]{0.11\textwidth}
		\centering
		\centerline{\includegraphics[width=1.8cm,height=1.8cm]{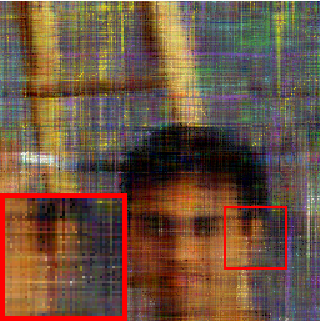}}
		\centerline{(g) QLNF}   
	\end{minipage}

	\begin{minipage}[t]{0.11\textwidth}
		\centering
		\centerline{\includegraphics[width=1.8cm,height=1.8cm]{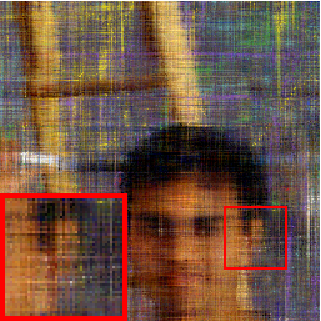}}
		\centerline{(h) TQLNA}
	\end{minipage}
	\begin{minipage}[t]{0.11\textwidth}
		\centering
		\centerline{\includegraphics[width=1.8cm,height=1.8cm]{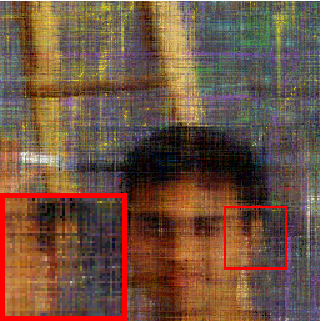}}
		\centerline{(i) LRQA-G}
	\end{minipage}
	\begin{minipage}[t]{0.11\textwidth}
		\centering
		\centerline{\includegraphics[width=1.8cm,height=1.8cm]{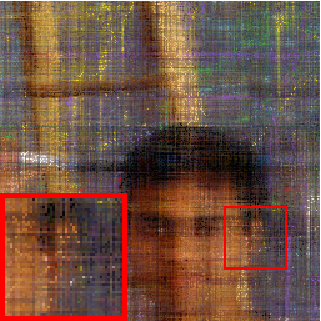}}
		\centerline{(j) LRQMC}
	\end{minipage}
	\begin{minipage}[t]{0.11\textwidth}
		\centering
		\centerline{\includegraphics[width=1.8cm,height=1.8cm]{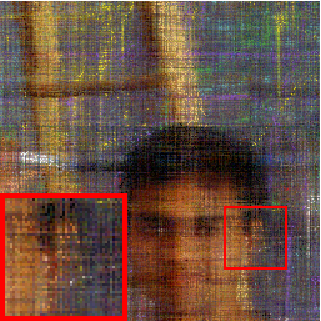}}
		\centerline{(k) QLNM-QQR}
	\end{minipage}
	\begin{minipage}[t]{0.11\textwidth}
		\centering
		\centerline{\includegraphics[width=1.8cm,height=1.8cm]{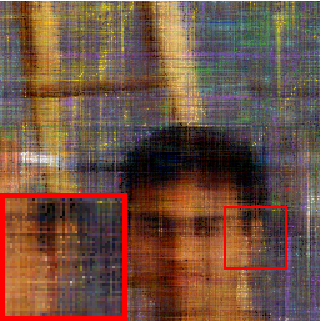}}
		\centerline{(l) IRQLNM-QQR}
	\end{minipage}
	\begin{minipage}[t]{0.11\textwidth}
		\centering
		\centerline{\includegraphics[width=1.8cm,height=1.8cm]{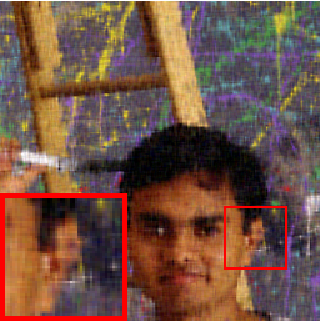}}
		\centerline{(m) QLNM-QQR-SR}
	\end{minipage}
	\caption{The image recovery outcomes of $\text{MR}=85\%$ on Image(2) and Image(5). (a) Ground truth. (b) Observation. (c)-(m) are the recovery outcomes of WNNM, MC-NC, IRLNM-QR, TNN-SR, QLNF, TQLNA, LRQA-G, LRQMC, QLNM-QQR, IRQLNM-QQR, and QLNM-QQR-SR, respectively.}
	\label{fig:IMAGE25}
\end{figure}

\indent
\textbf{Experimental results on the color medical images:} 
The continuous development of medical imaging equipment has revolutionized healthcare by providing accurate and detailed visual information about the human body. Magnetic resonance imaging (MRI) and positron emission tomography (PET) are widely used imaging techniques for capturing organs' structural and functional characteristics, respectively. Fusing these two types of data can greatly enhance the interpretation of tissue and organ behavior, thereby improving the accuracy of diagnoses. Therefore, the analysis of their overlay is of great importance in medical diagnosis. However, medical images may be incomplete for various reasons, such as equipment limitations or patient movements during examinations, which can negatively impact the accuracy of disease diagnosis. Quaternion completion-based models can be employed to address this issue.\\
\indent
Eight color medical images of size $256\times256$, obtained from ``The Whole Brain Atlas''$\footnote{http://www.med.harvard.edu/AANLIB/home.html}$ medical image database provided by Harvard Medical School, were used in the experiments. To leverage the information in the color medical images more effectively, we preprocessed the eight medical images by extracting sub-images with a size of $141\times141$ for subsequent experimental analysis. \cref{fig:medical_ori} shows the eight original medical images along with their corresponding sub-images obtained after preprocessing. 

\begin{figure}[htbp]
	\centering
	\tiny
	\begin{minipage}[t]{0.11\textwidth}
		\centering
		\centerline{\includegraphics[width=1.8cm,height=1.8cm]{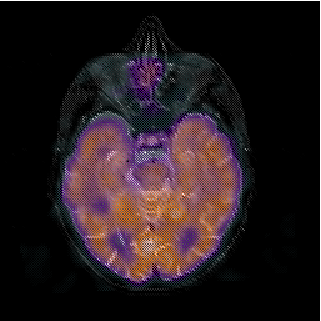}}
		\centerline{(a) Image(9)}    
	\end{minipage}
	\begin{minipage}[t]{0.11\textwidth}
		\centering
		\centerline{\includegraphics[width=1.8cm,height=1.8cm]{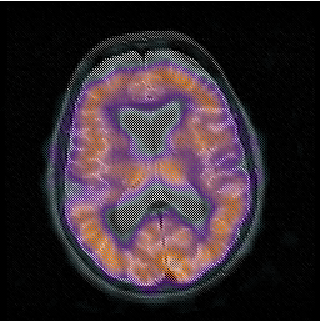}}
		\centerline{(b) Image(10)}
	\end{minipage}
	\begin{minipage}[t]{0.11\textwidth}
		\centering
		\centerline{\includegraphics[width=1.8cm,height=1.8cm]{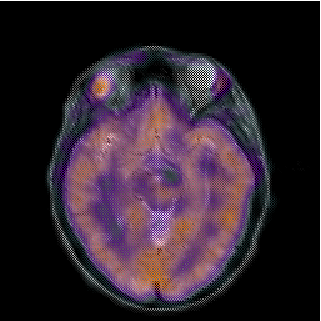}}
		\centerline{(c) Image(11)}
	\end{minipage}
	\begin{minipage}[t]{0.11\textwidth}
		\centering
		\centerline{\includegraphics[width=1.8cm,height=1.8cm]{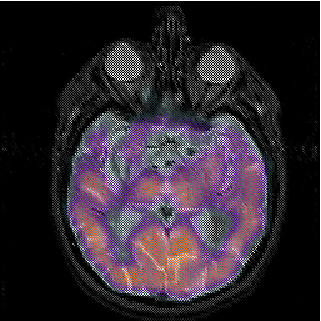}}
		\centerline{(d) Image(12)}
	\end{minipage}
	\begin{minipage}[t]{0.11\textwidth}
		\centering
		\centerline{\includegraphics[width=1.8cm,height=1.8cm]{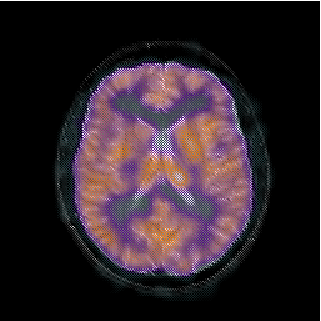}}
		\centerline{(e) Image(13)}
	\end{minipage}
	\begin{minipage}[t]{0.11\textwidth}
		\centering
		\centerline{\includegraphics[width=1.8cm,height=1.8cm]{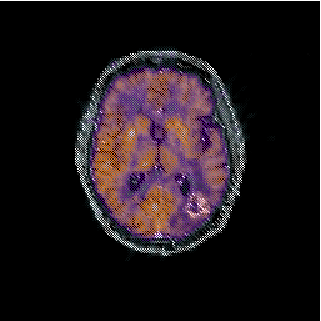}}
		\centerline{(f) Image(14)}
	\end{minipage}
	\begin{minipage}[t]{0.11\textwidth}
		\centering
		\centerline{\includegraphics[width=1.8cm,height=1.8cm]{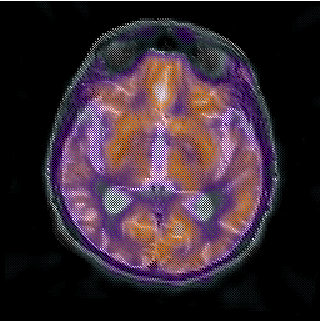}}
		\centerline{(g) Image(15)}
	\end{minipage}
	\begin{minipage}[t]{0.11\textwidth}
		\centering
		\centerline{\includegraphics[width=1.8cm,height=1.8cm]{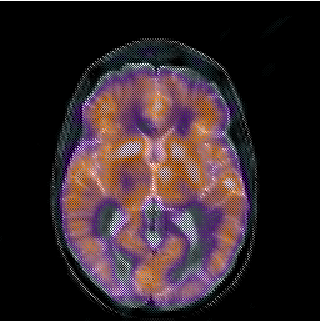}}
		\centerline{(h) Image(16)}
	\end{minipage}

	\begin{minipage}[t]{0.11\textwidth}
	\centering
	\centerline{\includegraphics[width=1.95cm,height=1.95cm]{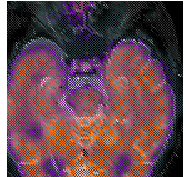}}
	\centerline{(i) Image(9)$\_$Sub}    
\end{minipage}
\begin{minipage}[t]{0.11\textwidth}
	\centering
	\centerline{\includegraphics[width=1.95cm,height=1.95cm]{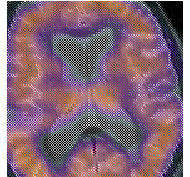}}
	\centerline{(j) Image(10)$\_$Sub}
\end{minipage}
\begin{minipage}[t]{0.11\textwidth}
	\centering
	\centerline{\includegraphics[width=1.95cm,height=1.95cm]{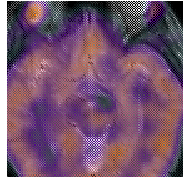}}
	\centerline{(k) Image(11)$\_$Sub}
\end{minipage}
\begin{minipage}[t]{0.11\textwidth}
	\centering
	\centerline{\includegraphics[width=1.95cm,height=1.95cm]{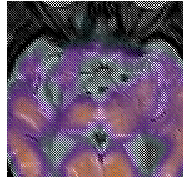}}
	\centerline{(l) Image(12)$\_$Sub}
\end{minipage}
\begin{minipage}[t]{0.11\textwidth}
	\centering
	\centerline{\includegraphics[width=1.95cm,height=1.95cm]{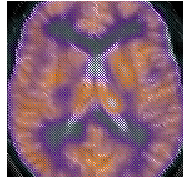}}
	\centerline{(m) Image(13)$\_$Sub}
\end{minipage}
\begin{minipage}[t]{0.11\textwidth}
	\centering
	\centerline{\includegraphics[width=1.95cm,height=1.95cm]{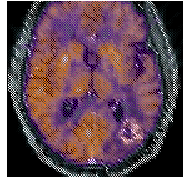}}
	\centerline{(n) Image(14)$\_$Sub}
\end{minipage}
\begin{minipage}[t]{0.11\textwidth}
	\centering
	\centerline{\includegraphics[width=1.95cm,height=1.95cm]{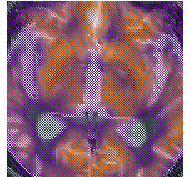}}
	\centerline{(o) Image(15)$\_$Sub}
\end{minipage}
\begin{minipage}[t]{0.11\textwidth}
	\centering
	\centerline{\includegraphics[width=1.95cm,height=1.95cm]{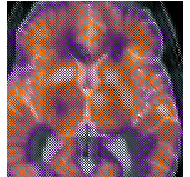}}
	\centerline{(p) Image(16)$\_$Sub}
\end{minipage}
	\caption{(a)-(h) are the eight original color medical images with size $256\times256\times3$. (i)-(p) Ground truth: Image(9)$\_$Sub-Image(16)$\_$Sub are the corresponding sub-images with size $141\times141\times3$.}
	\label{fig:medical_ori}
\end{figure}

\indent
For the experiment conducted at an MR=$85\%$, we used the same parameter settings for $\mu^{0}$, $\rho$, and $\beta$ in the QLNM-QQR, IRQLNM-QQR, and QLNM-QQR-SR methods as in the previous experiment. The $r$ values for QLNM-QQR, IRQLNM-QQR, and QLNM-QQR-SR were set to 55, 80, and 45, respectively. The IRLNM-QQR method was also implemented with the values of $\varsigma$ and V set to 10 and 3, respectively, in (\ref{IRQLNM_QQR_wei}).
\cref{fig:medicalrandom} displays the recovered color medical images using different methods for visual comparison. \cref{fig:medicalrandom} visually demonstrates the superiority of our proposed QLNM-QQR-SR approach over all other methods.
The numerical comparison of different methods in terms of PSNR and SSIM values for recovered medical images at MR=$85\%$ is shown in Table \ref{tablecolormedical}. 
The results show that the IRQLNM-QQR method is superior to the QLNM-QQR and LRQA-G methods regarding numerical and visual results.
It is challenging to improve the quality of recovered images at this level of MR. However, our proposed QLNM-QQR-SR method outperforms other methods in terms of PSNR and SSIM values, as demonstrated by Table \ref{tablecolormedical}. 

\begin{figure*}[htbp]
	\centering	
	\begin{minipage}[h]{0.065\textwidth}
		\centering
		\begin{minipage}{1\textwidth}
			\centering
			\includegraphics[width=1.2cm,height=1.2cm]{SR_0.15_medical/01_cut.eps}
		\end{minipage} 
		\hfill\\
		\begin{minipage}{1\textwidth}
			\centering
			\includegraphics[width=1.2cm,height=1.2cm]{SR_0.15_medical/02_cut.eps}
		\end{minipage} 
		\hfill\\
		\begin{minipage}{1\textwidth}
			\centering
			\includegraphics[width=1.2cm,height=1.2cm]{SR_0.15_medical/06_cut.eps}
		\end{minipage} 
		\hfill\\
		\begin{minipage}{1\textwidth}
			\centering
			\includegraphics[width=1.2cm,height=1.2cm]{SR_0.15_medical/07_cut.eps}
		\end{minipage} 
		\hfill\\	
		\begin{minipage}{1\textwidth}
			\centering
			\includegraphics[width=1.2cm,height=1.2cm]{SR_0.15_medical/08_cut.eps}
		\end{minipage} 
		\hfill\\
		\begin{minipage}{1\textwidth}
			\centering
			\includegraphics[width=1.2cm,height=1.2cm]{SR_0.15_medical/15_cut.eps}
		\end{minipage} 
		\hfill\\
		\begin{minipage}{1\textwidth}
			\centering
			\includegraphics[width=1.2cm,height=1.2cm]{SR_0.15_medical/16_cut.eps}
		\end{minipage} 
		\hfill\\
		\begin{minipage}{1\textwidth}
			\centering
			\includegraphics[width=1.2cm,height=1.2cm]{SR_0.15_medical/22_cut.eps}
		\end{minipage} 
		\hfill\\
		\caption*{(a)}
		%\label{a}
	\end{minipage}
	%\hfill
	\begin{minipage}[h]{0.065\textwidth}
		\centering
		\begin{minipage}{1\textwidth}
			\centering
			\includegraphics[width=1.2cm,height=1.2cm]{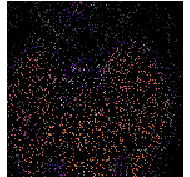}
		\end{minipage} 
		\hfill\\
		\begin{minipage}{1\textwidth}
			\centering
			\includegraphics[width=1.2cm,height=1.2cm]{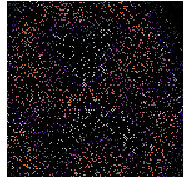}
		\end{minipage} 
		\hfill\\
		\begin{minipage}{1\textwidth}
			\centering
			\includegraphics[width=1.2cm,height=1.2cm]{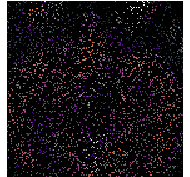}
		\end{minipage} 
		\hfill\\
		\begin{minipage}{1\textwidth}
			\centering
			\includegraphics[width=1.2cm,height=1.2cm]{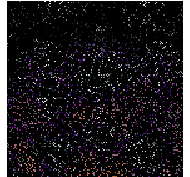}
		\end{minipage} 
		\hfill\\
		\begin{minipage}{1\textwidth}
			\centering
			\includegraphics[width=1.2cm,height=1.2cm]{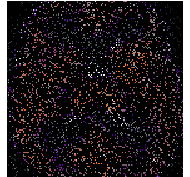}
		\end{minipage} 
		\hfill\\
		\begin{minipage}{1\textwidth}
			\centering
			\includegraphics[width=1.2cm,height=1.2cm]{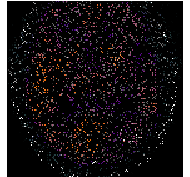}
		\end{minipage} 
		\hfill\\
		\begin{minipage}{1\textwidth}
			\centering
			\includegraphics[width=1.2cm,height=1.2cm]{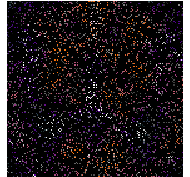}
		\end{minipage} 
		\hfill\\
		\begin{minipage}{1\textwidth}
			\centering
			\includegraphics[width=1.2cm,height=1.2cm]{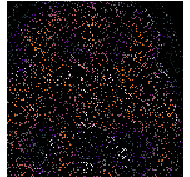}
		\end{minipage} 
		\hfill\\
		\caption*{(b)}
		%\label{a}
	\end{minipage}
	%\hfill
	\begin{minipage}[h]{0.065\textwidth}
		\centering
		\begin{minipage}{1\textwidth}
			\centering
			\includegraphics[width=1.2cm,height=1.2cm]{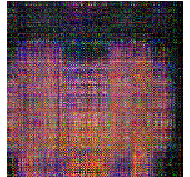}
		\end{minipage} 
		\hfill\\
		\begin{minipage}{1\textwidth}
			\centering
			\includegraphics[width=1.2cm,height=1.2cm]{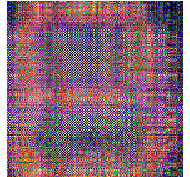}
		\end{minipage} 
		\hfill\\
		\begin{minipage}{1\textwidth}
			\centering
			\includegraphics[width=1.2cm,height=1.2cm]{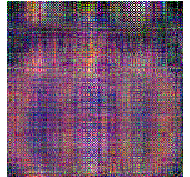}
		\end{minipage} 
		\hfill\\
		\begin{minipage}{1\textwidth}
			\centering
			\includegraphics[width=1.2cm,height=1.2cm]{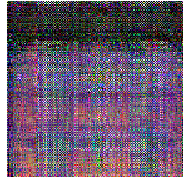}
		\end{minipage} 
		\hfill\\
		\begin{minipage}{1\textwidth}
			\centering
			\includegraphics[width=1.2cm,height=1.2cm]{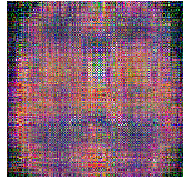}
		\end{minipage} 
		\hfill\\
		\begin{minipage}{1\textwidth}
			\centering
			\includegraphics[width=1.2cm,height=1.2cm]{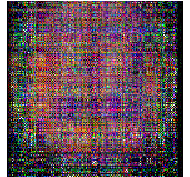}
		\end{minipage} 
		\hfill\\
		\begin{minipage}{1\textwidth}
			\centering
			\includegraphics[width=1.2cm,height=1.2cm]{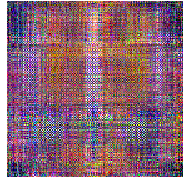}
		\end{minipage} 
		\hfill\\
		\begin{minipage}{1\textwidth}
			\centering
			\includegraphics[width=1.2cm,height=1.2cm]{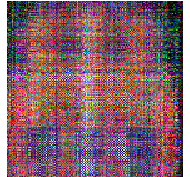}
		\end{minipage} 
		\hfill\\
		\caption*{(c)}
	\end{minipage}
	\begin{minipage}[h]{0.065\textwidth}
		\centering
		\begin{minipage}{1\textwidth}
			\centering
			\includegraphics[width=1.2cm,height=1.2cm]{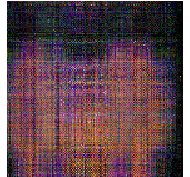}
		\end{minipage} 
		\hfill\\
		\begin{minipage}{1\textwidth}
			\centering
			\includegraphics[width=1.2cm,height=1.2cm]{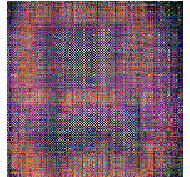}
		\end{minipage} 
		\hfill\\
		\begin{minipage}{1\textwidth}
			\centering
			\includegraphics[width=1.2cm,height=1.2cm]{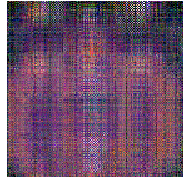}
		\end{minipage} 
		\hfill\\
		\begin{minipage}{1\textwidth}
			\centering
			\includegraphics[width=1.2cm,height=1.2cm]{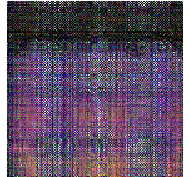}
		\end{minipage} 
		\hfill\\
		\begin{minipage}{1\textwidth}
			\centering
			\includegraphics[width=1.2cm,height=1.2cm]{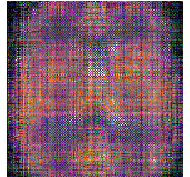}
		\end{minipage} 
		\hfill\\
		\begin{minipage}{1\textwidth}
			\centering
			\includegraphics[width=1.2cm,height=1.2cm]{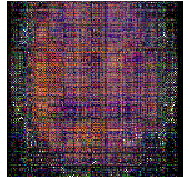}
		\end{minipage} 
		\hfill\\
		\begin{minipage}{1\textwidth}
			\centering
			\includegraphics[width=1.2cm,height=1.2cm]{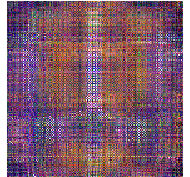}
		\end{minipage} 
		\hfill\\
		\begin{minipage}{1\textwidth}
			\centering
			\includegraphics[width=1.2cm,height=1.2cm]{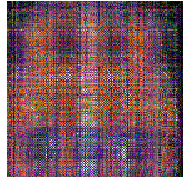}
		\end{minipage} 
		\hfill\\
		\caption*{(d)}
	\end{minipage}
	\begin{minipage}[h]{0.065\textwidth}
		\centering
		\begin{minipage}{1\textwidth}
			\centering
			\includegraphics[width=1.2cm,height=1.2cm]{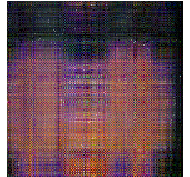}
		\end{minipage} 
		\hfill\\
		\begin{minipage}{1\textwidth}
			\centering
			\includegraphics[width=1.2cm,height=1.2cm]{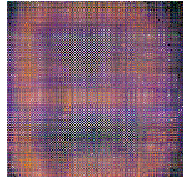}
		\end{minipage} 
		\hfill\\
		\begin{minipage}{1\textwidth}
			\centering
			\includegraphics[width=1.2cm,height=1.2cm]{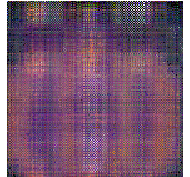}
		\end{minipage} 
		\hfill\\
		\begin{minipage}{1\textwidth}
			\centering
			\includegraphics[width=1.2cm,height=1.2cm]{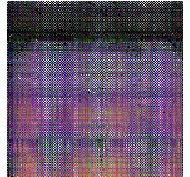}
		\end{minipage} 
		\hfill\\
		\begin{minipage}{1\textwidth}
			\centering
			\includegraphics[width=1.2cm,height=1.2cm]{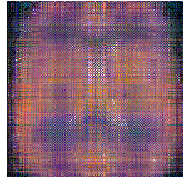}
		\end{minipage} 
		\hfill\\
		\begin{minipage}{1\textwidth}
			\centering
			\includegraphics[width=1.2cm,height=1.2cm]{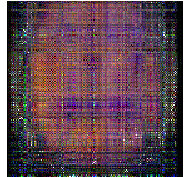}
		\end{minipage} 
		\hfill\\
		\begin{minipage}{1\textwidth}
			\centering
			\includegraphics[width=1.2cm,height=1.2cm]{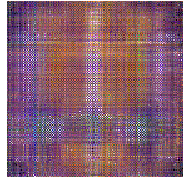}
		\end{minipage} 
		\hfill\\
		\begin{minipage}{1\textwidth}
			\centering
			\includegraphics[width=1.2cm,height=1.2cm]{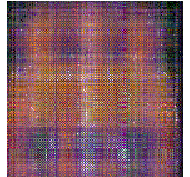}
		\end{minipage} 
		\hfill\\
		\caption*{(e)}
	\end{minipage}
	\begin{minipage}[h]{0.065\textwidth}
		\centering
		\begin{minipage}{1\textwidth}
			\centering
			\includegraphics[width=1.2cm,height=1.2cm]{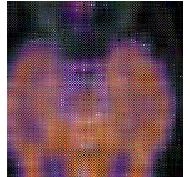}
		\end{minipage} 
		\hfill\\
		\begin{minipage}{1\textwidth}
			\centering
			\includegraphics[width=1.2cm,height=1.2cm]{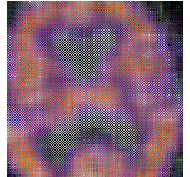}
		\end{minipage} 
		\hfill\\
		\begin{minipage}{1\textwidth}
			\centering
			\includegraphics[width=1.2cm,height=1.2cm]{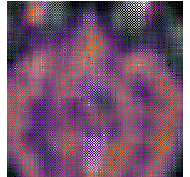}
		\end{minipage} 
		\hfill\\
		\begin{minipage}{1\textwidth}
			\centering
			\includegraphics[width=1.2cm,height=1.2cm]{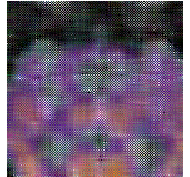}
		\end{minipage} 
		\hfill\\
		\begin{minipage}{1\textwidth}
			\centering
			\includegraphics[width=1.2cm,height=1.2cm]{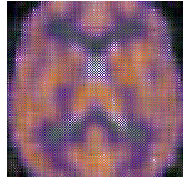}
		\end{minipage} 
		\hfill\\
		\begin{minipage}{1\textwidth}
			\centering
			\includegraphics[width=1.2cm,height=1.2cm]{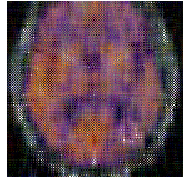}
		\end{minipage} 
		\hfill\\
		\begin{minipage}{1\textwidth}
			\centering
			\includegraphics[width=1.2cm,height=1.2cm]{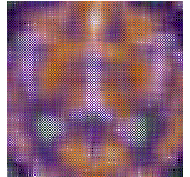}
		\end{minipage} 
		\hfill\\
		\begin{minipage}{1\textwidth}
			\centering
			\includegraphics[width=1.2cm,height=1.2cm]{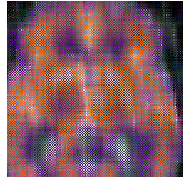}
		\end{minipage} 
		\hfill\\
		\caption*{(f)}
		%\label{a}
	\end{minipage}
	%	\hfill
	\begin{minipage}[h]{0.065\textwidth}
		\centering
		\begin{minipage}{1\textwidth}
			\centering
			\includegraphics[width=1.2cm,height=1.2cm]{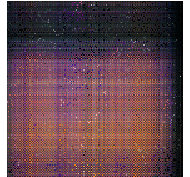}
		\end{minipage} 
		\hfill\\
		\begin{minipage}{1\textwidth}
			\centering
			\includegraphics[width=1.2cm,height=1.2cm]{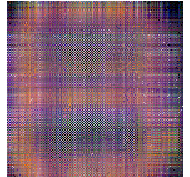}
		\end{minipage} 
		\hfill\\
		\begin{minipage}{1\textwidth}
			\centering
			\includegraphics[width=1.2cm,height=1.2cm]{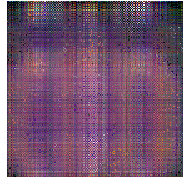}
		\end{minipage} 
		\hfill\\
		\begin{minipage}{1\textwidth}
			\centering
			\includegraphics[width=1.2cm,height=1.2cm]{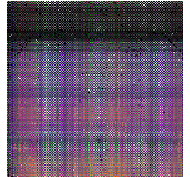}
		\end{minipage} 
		\hfill\\
		\begin{minipage}{1\textwidth}
			\centering
			\includegraphics[width=1.2cm,height=1.2cm]{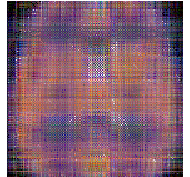}
		\end{minipage} 
		\hfill\\
		\begin{minipage}{1\textwidth}
			\centering
			\includegraphics[width=1.2cm,height=1.2cm]{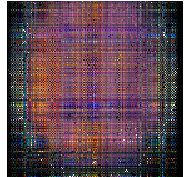}
		\end{minipage} 
		\hfill\\
		\begin{minipage}{1\textwidth}
			\centering
			\includegraphics[width=1.2cm,height=1.2cm]{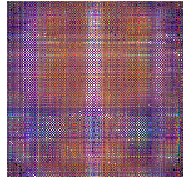}
		\end{minipage} 
		\hfill\\
		\begin{minipage}{1\textwidth}
			\centering
			\includegraphics[width=1.2cm,height=1.2cm]{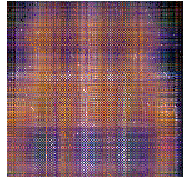}
		\end{minipage} 
		\hfill\\
		\caption*{(g)}
		%\label{a}
	\end{minipage}
	%	\hfill
	\begin{minipage}[h]{0.065\textwidth}
		\centering
		\begin{minipage}{1\textwidth}
			\centering
			\includegraphics[width=1.2cm,height=1.2cm]{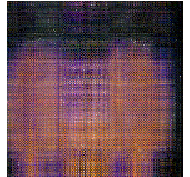}
		\end{minipage} 
		\hfill\\
		\begin{minipage}{1\textwidth}
			\centering
			\includegraphics[width=1.2cm,height=1.2cm]{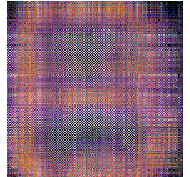}
		\end{minipage} 
		\hfill\\
		\begin{minipage}{1\textwidth}
			\centering
			\includegraphics[width=1.2cm,height=1.2cm]{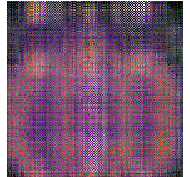}
		\end{minipage} 
		\hfill\\
		\begin{minipage}{1\textwidth}
			\centering
			\includegraphics[width=1.2cm,height=1.2cm]{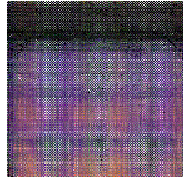}
		\end{minipage} 
		\hfill\\
		\begin{minipage}{1\textwidth}
			\centering
			\includegraphics[width=1.2cm,height=1.2cm]{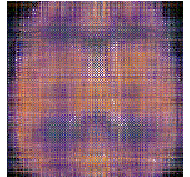}
		\end{minipage} 
		\hfill\\
		\begin{minipage}{1\textwidth}
			\centering
			\includegraphics[width=1.2cm,height=1.2cm]{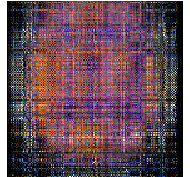}
		\end{minipage} 
		\hfill\\
		\begin{minipage}{1\textwidth}
			\centering
			\includegraphics[width=1.2cm,height=1.2cm]{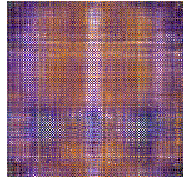}
		\end{minipage} 
		\hfill\\
		\begin{minipage}{1\textwidth}
			\centering
			\includegraphics[width=1.2cm,height=1.2cm]{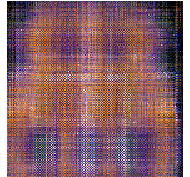}
		\end{minipage} 
		\hfill\\
		\caption*{(h)}
		%\label{a}
	\end{minipage}
	%	\hfill
	\begin{minipage}[h]{0.065\textwidth}
		\centering
		\begin{minipage}{1\textwidth}
			\centering
			\includegraphics[width=1.2cm,height=1.2cm]{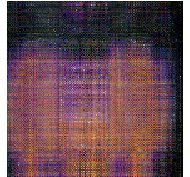}
		\end{minipage} 
		\hfill\\
		\begin{minipage}{1\textwidth}
			\centering
			\includegraphics[width=1.2cm,height=1.2cm]{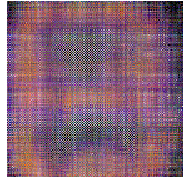}
		\end{minipage} 
		\hfill\\
		\begin{minipage}{1\textwidth}
			\centering
			\includegraphics[width=1.2cm,height=1.2cm]{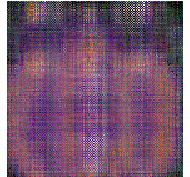}
		\end{minipage} 
		\hfill\\
		\begin{minipage}{1\textwidth}
			\centering
			\includegraphics[width=1.2cm,height=1.2cm]{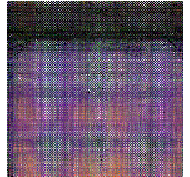}
		\end{minipage} 
		\hfill\\
		\begin{minipage}{1\textwidth}
			\centering
			\includegraphics[width=1.2cm,height=1.2cm]{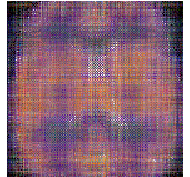}
		\end{minipage} 
		\hfill\\
		\begin{minipage}{1\textwidth}
			\centering
			\includegraphics[width=1.2cm,height=1.2cm]{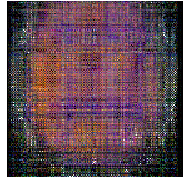}
		\end{minipage} 
		\hfill\\
		\begin{minipage}{1\textwidth}
			\centering
			\includegraphics[width=1.2cm,height=1.2cm]{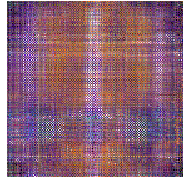}
		\end{minipage} 
		\hfill\\
		\begin{minipage}{1\textwidth}
			\centering
			\includegraphics[width=1.2cm,height=1.2cm]{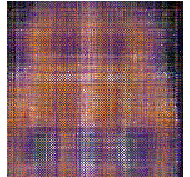}
		\end{minipage} 
		\hfill\\
		\caption*{(i)}
		%\label{a}
	\end{minipage}
	%	\hfill
	\begin{minipage}[h]{0.065\textwidth}
		\centering
		\begin{minipage}{1\textwidth}
			\centering
			\includegraphics[width=1.2cm,height=1.2cm]{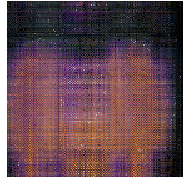}
		\end{minipage} 
		\hfill\\
		\begin{minipage}{1\textwidth}
			\centering
			\includegraphics[width=1.2cm,height=1.2cm]{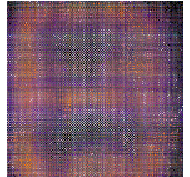}
		\end{minipage} 
		\hfill\\
		\begin{minipage}{1\textwidth}
			\centering
			\includegraphics[width=1.2cm,height=1.2cm]{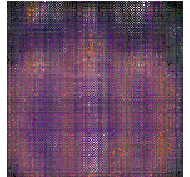}
		\end{minipage} 
		\hfill\\
		\begin{minipage}{1\textwidth}
			\centering
			\includegraphics[width=1.2cm,height=1.2cm]{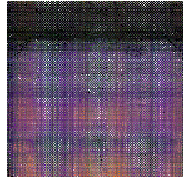}
		\end{minipage} 
		\hfill\\
		\begin{minipage}{1\textwidth}
			\centering
			\includegraphics[width=1.2cm,height=1.2cm]{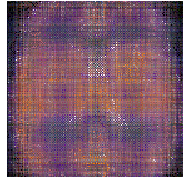}
		\end{minipage} 
		\hfill\\
		\begin{minipage}{1\textwidth}
			\centering
			\includegraphics[width=1.2cm,height=1.2cm]{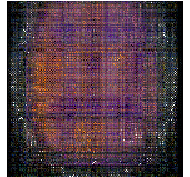}
		\end{minipage} 
		\hfill\\
		\begin{minipage}{1\textwidth}
			\centering
			\includegraphics[width=1.2cm,height=1.2cm]{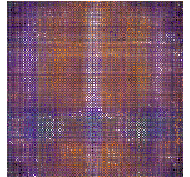}
		\end{minipage} 
		\hfill\\
		\begin{minipage}{1\textwidth}
			\centering
			\includegraphics[width=1.2cm,height=1.2cm]{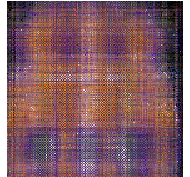}
		\end{minipage} 
		\hfill\\
		\caption*{(j)}
		%\label{a}
	\end{minipage}
	%	\hfill
	\begin{minipage}[h]{0.065\textwidth}
		\centering
		\begin{minipage}{1\textwidth}
			\centering
			\includegraphics[width=1.2cm,height=1.2cm]{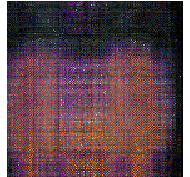}
		\end{minipage} 
		\hfill\\
		\begin{minipage}{1\textwidth}
			\centering
			\includegraphics[width=1.2cm,height=1.2cm]{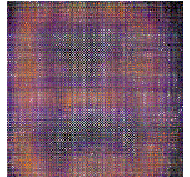}
		\end{minipage} 
		\hfill\\
		\begin{minipage}{1\textwidth}
			\centering
			\includegraphics[width=1.2cm,height=1.2cm]{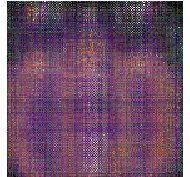}
		\end{minipage} 
		\hfill\\
		\begin{minipage}{1\textwidth}
			\centering
			\includegraphics[width=1.2cm,height=1.2cm]{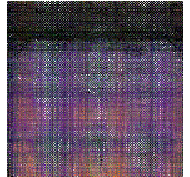}
		\end{minipage} 
		\hfill\\
		\begin{minipage}{1\textwidth}
			\centering
			\includegraphics[width=1.2cm,height=1.2cm]{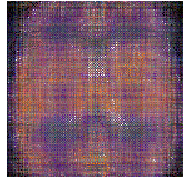}
		\end{minipage} 
		\hfill\\
		\begin{minipage}{1\textwidth}
			\centering
			\includegraphics[width=1.2cm,height=1.2cm]{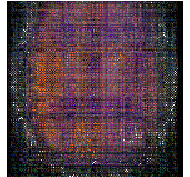}
		\end{minipage} 
		\hfill\\
		\begin{minipage}{1\textwidth}
			\centering
			\includegraphics[width=1.2cm,height=1.2cm]{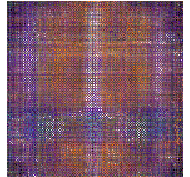}
		\end{minipage} 
		\hfill\\
		\begin{minipage}{1\textwidth}
			\centering
			\includegraphics[width=1.2cm,height=1.2cm]{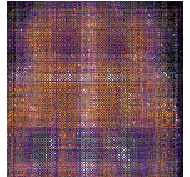}
		\end{minipage} 
		\hfill\\
		\caption*{(k)}
		%\label{a}
	\end{minipage}
	%	\hfill
	\begin{minipage}[h]{0.065\textwidth}
		\centering
		\begin{minipage}{1\textwidth}
			\centering
			\includegraphics[width=1.2cm,height=1.2cm]{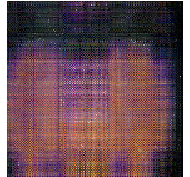}
		\end{minipage} 
		\hfill\\
		\begin{minipage}{1\textwidth}
			\centering
			\includegraphics[width=1.2cm,height=1.2cm]{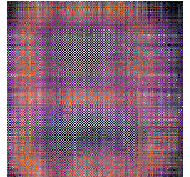}
		\end{minipage} 
		\hfill\\
		\begin{minipage}{1\textwidth}
			\centering
			\includegraphics[width=1.2cm,height=1.2cm]{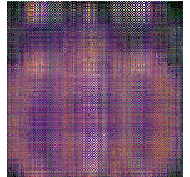}
		\end{minipage} 
		\hfill\\
		\begin{minipage}{1\textwidth}
			\centering
			\includegraphics[width=1.2cm,height=1.2cm]{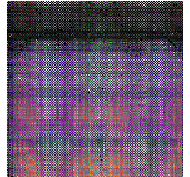}
		\end{minipage} 
		\hfill\\
		\begin{minipage}{1\textwidth}
			\centering
			\includegraphics[width=1.2cm,height=1.2cm]{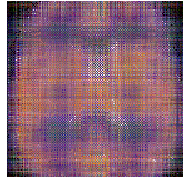}
		\end{minipage} 
		\hfill\\
		\begin{minipage}{1\textwidth}
			\centering
			\includegraphics[width=1.2cm,height=1.2cm]{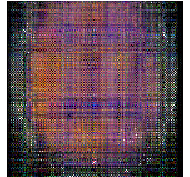}
		\end{minipage} 
		\hfill\\
		\begin{minipage}{1\textwidth}
			\centering
			\includegraphics[width=1.2cm,height=1.2cm]{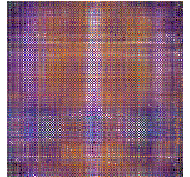}
		\end{minipage} 
		\hfill\\
		\begin{minipage}{1\textwidth}
			\centering
			\includegraphics[width=1.2cm,height=1.2cm]{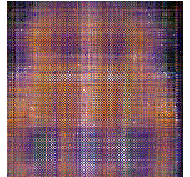}
		\end{minipage} 
		\hfill\\
		\caption*{(l)}
		%\label{a}
	\end{minipage}
	\begin{minipage}[h]{0.065\textwidth}
		\centering
		\begin{minipage}{1\textwidth}
			\centering
			\includegraphics[width=1.2cm,height=1.2cm]{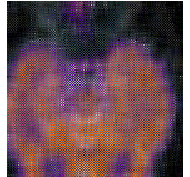}
		\end{minipage} 
		\hfill\\
		\begin{minipage}{1\textwidth}
			\centering
			\includegraphics[width=1.2cm,height=1.2cm]{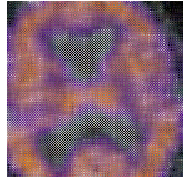}
		\end{minipage} 
		\hfill\\
		\begin{minipage}{1\textwidth}
			\centering
			\includegraphics[width=1.2cm,height=1.2cm]{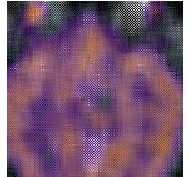}
		\end{minipage} 
		\hfill\\
		\begin{minipage}{1\textwidth}
			\centering
			\includegraphics[width=1.2cm,height=1.2cm]{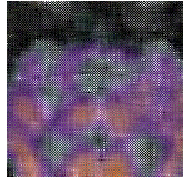}
		\end{minipage} 
		\hfill\\
		\begin{minipage}{1\textwidth}
			\centering
			\includegraphics[width=1.2cm,height=1.2cm]{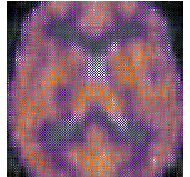}
		\end{minipage} 
		\hfill\\
		\begin{minipage}{1\textwidth}
			\centering
			\includegraphics[width=1.2cm,height=1.2cm]{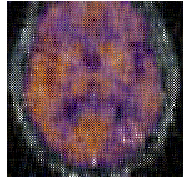}
		\end{minipage} 
		\hfill\\
		\begin{minipage}{1\textwidth}
			\centering
			\includegraphics[width=1.2cm,height=1.2cm]{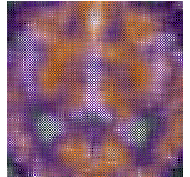}
		\end{minipage} 
		\hfill\\
		\begin{minipage}{1\textwidth}
			\centering
			\includegraphics[width=1.2cm,height=1.2cm]{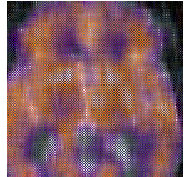}
		\end{minipage} 
		\hfill\\
		\caption*{(m)}
		%\label{a}
	\end{minipage}
	\hfill\\
	\caption{(a) Ground truth. From top to bottom: Image(9)$\_$Sub-Image(16)$\_$Sub. (b) Observation (MR=$75\%$). (c)-(m) are the restored results of WNNM, MC-NC, IRLNM-QR, TNN-SR, QLNF, TQLNA, LRQA-G, LRQMC, QLNM-QQR, IRQLNM-QQR, and QLNM-QQR-SR, respectively.}
	\label{fig:medicalrandom} 
\end{figure*}

\begin{table*}[t]
	\caption{A comparison of quantitative assessment indices (PSNR/SSIM) across different methods on the set of eight color medical images.}
	\label{tablecolormedical}
	\centering
	\resizebox{\textwidth}{!}{
		\begin{tabular}{|c|c|c|c|c|c|c|c|c|c|c|c|}		
			\hline
			Methods:& WNNM  &  MC-NC  &IRLNM-QR  & TNN-SR & QLNF  &TQLNA & LRQA-G & LRQMC   & QLNM-QQR &IRQLNM-QQR & QLNM-QQR-SR	
			\\ \toprule
			\hline
			Images:  &\multicolumn{11}{c|}{${\rm{MR}}=85\%$}\\
		    \hline
		     Image(9)$\_$Sub &	14.693/0.432	&	16.221/0.494	&	17.768/0.588	&	21.499/0.774	&	17.462/0.538	&	18.169/0.588	&	18.026/0.589	&	18.418/0.611	&	17.311/0.550	&	18.188/0.613	&	\textbf{21.729}/\textbf{0.781}\\
		     Image(10)$\_$Sub &	13.239/0.447	&	14.673/0.501	&	16.528/0.627	&	21.003/0.838	&	16.470/0.600	&	16.905/0.635	&	16.637/0.624	&	15.985/0.590	&	15.712/0.568	&	16.848/0.641	&	\textbf{21.494}/\textbf{0.854}\\
		     Image(11)$\_$Sub &    14.949/0.480	&	16.422/0.543	&	17.675/0.637	&	22.786/0.860	&	17.476/0.605	&	18.336/0.661	&	18.116/0.645	&	17.639/0.621	&	17.366/0.598	&	18.191/0.667	&	\textbf{23.031}/\textbf{0.870}\\
		     Image(12)$\_$Sub &	12.798/0.436	&	14.124/0.476	&	15.855/0.612	&	18.968/0.783	&	16.053/0.613	&	15.949/0.614	&	15.880/0.604	&	16.119/0.611	&	15.499/0.567	&	16.039/0.628	&	\textbf{19.073}/\textbf{0.787}\\
		     Image(13)$\_$Sub &	13.445/0.415	&	15.158/0.495	&	16.171/0.576	&	21.103/0.830	&	16.836/0.595	&	17.166/0.623	&	16.883/0.605	&	16.541/0.590	&	16.144/0.555	&	17.003/0.627	&	\textbf{21.437}/\textbf{0.843}\\
		     Image(14)$\_$Sub &	10.679/0.258	&	12.762/0.321	&	13.536/0.407	&	16.670/0.597	&	14.028/0.399	&	13.577/0.391	&	13.993/0.406	&	14.108/0.401	&	13.827/0.374	&	13.997/0.431	&	\textbf{16.844}/\textbf{0.611}\\
		     Image(15)$\_$Sub	& 12.684/0.468	&	14.198/0.527	&	15.830/0.640	&	20.442/0.861	&	15.531/0.617	&	16.031/0.652	&	15.858/0.635	&	15.541/0.614	&	15.265/0.585	&	15.986/0.656	&	\textbf{20.726}/\textbf{0.873}\\
		     Image(16)$\_$Sub	& 12.481/0.420	&	13.977/0.466	&	15.403/0.576	&	20.481/0.829	&	15.772/0.591	&	16.138/0.620	&	15.990/0.604	&	16.329/0.622	&	15.285/0.546	&	16.123/0.625	&	\textbf{20.933}/\textbf{0.843}\\
			\hline
			Aver. &  14.949		&	16.422		&	17.768		&	22.786		&	17.476		&	18.336		&	18.116		&	18.418		&	17.366	&		18.191	&		\textbf{23.031}	\\ \toprule
	\end{tabular}}
\end{table*}

\section{Conclusions}
In this study, we created a novel method called QLNM-QQR for completing color images using the tool of quaternion representation of color images. The method is based on the quaternion $L_{2,1}$-norm and a Tri-Factorization of a quaternion matrix called CQSVD-QQR. The coupling between color channels can be naturally handled with this approach and representing color pixels as vector units rather than scalars in the quaternion representation results in better retention of color information. The method avoids the need to calculate the QSVD of large quaternion matrices, which reduces the computational complexity compared to traditional LRQMC methods. According to theoretical analysis, the quaternion $L_{2,1}$-norm of a submatrix in QLNM-QQR is capable of converging to its QNN. To enhance its performance, we introduce an improved version called IRQLNM-QQR that uses iteratively reweighted quaternion $L_{2,1}$-norm minimization. According to theoretical analysis, IRQLNM-QQR is equally precise as an LRQA-W minimization method. Additionally, we incorporate sparse regularization into the QLNM-QQR method to develop QLNM-QQR-SR. The experimental results obtained from both natural color images and color medical images indicate that IRLNM-QQR achieves almost comparable accuracy to the LRQA-G method and outperforms QLNM-QQR in precision. The experimental results also demonstrate that the QLNM-QQR-SR method proposed in this research displays better performance in both numerical accuracy and visual quality compared to several state-of-the-art techniques. Besides, we have proven that the quaternion $L_{2,1}$-norm is an upper bound for the quaternion nuclear norm of a quaternion matrix. As a result, the proposed methods have broad applicability and can enhance the performance of a variety of techniques, including multiview data analysis, quaternion matrix/tensor completion, and low-rank representation based on quaternion nuclear norm.

\section*{Acknowledgments}
This work was supported by University of Macau (MYRG2019-00039-FST), Science and Technology Development Fund, Macao S.A.R (FDCT/0036/2021/AGJ), and Science and Technology Planning Project of Guangzhou City, China (Grant No. 201907010043).

\section*{Declaration of competing interest}

The authors declare that they have no known competing financial interests or personal relationships that could have appeared to influence the work reported in this paper.

 \appendix
 \section{Proof of Theorem \ref{L2,1mp}}
 \label{appendixA}
 \begin{proof}
 	The quaternion $L_{2,1}$-norm of $\dot{\mathbf{X}}$ can be rewritten as follows:
 	\begin{equation}
 		\|\dot{\mathbf{X}}\|_{2,1}=\sum_{n=1}^{N}\sqrt{\sum_{m=1}^{M}|\dot{\mathbf{X}}_{mn}|^{2}}=\sum_{n=1}^{N}\|\dot{\mathbf{X}}_{\cdot n}\|_{2},
 	\end{equation}
 	where $\|\dot{\mathbf{X}}_{\cdot n}\|_{2}=\sqrt{\sum_{m=1}^{M}|\dot{\mathbf{X}}_{mn}|^{2}}$. Since two of the terms in (\ref{L2,1_norm}) are convex, there is only one optimal solution. With the use of the associated theories for quaternion matrix derivatives in \cite{xu2015theory}, we get 
 	\begin{equation}
 		\begin{aligned}
 			\frac{\partial\mathcal{J}(\dot{\mathbf{X}}_{\cdot n})}{\partial\dot{\mathbf{X}}_{\cdot n}}&=\beta\frac{\partial \|\dot{\mathbf{X}}_{\cdot n}\|_{2}}{\partial\dot{\mathbf{X}}_{\cdot n}}+\frac{1}{2}\frac{\partial\text{Tr}\left[(\dot{\mathbf{X}}_{\cdot n}-\dot{\mathbf{Y}}_{\cdot n})^{H}(\dot{\mathbf{X}}_{\cdot n}-\dot{\mathbf{Y}}_{\cdot n})\right]}{\partial\dot{\mathbf{X}}_{\cdot n}}\\
 			&=\beta \frac{\dot{\mathbf{X}}_{\cdot n}}{\|\dot{\mathbf{X}}_{\cdot n}\|_{2}}+\frac{1}{4}\Big(\dot{\mathbf{X}}_{\cdot n}-\dot{\mathbf{Y}}_{\cdot n}\Big)\\
 			&=\beta \frac{\dot{\mathbf{Y}}_{\cdot n}}{\|\dot{\mathbf{Y}}_{\cdot n}\|_{2}}+\frac{1}{4}\Big(\dot{\mathbf{X}}_{\cdot n}-\dot{\mathbf{Y}}_{\cdot n}\Big) \: \Big(\|\dot{\mathbf{Y}}_{\cdot n}\|_{2}>4\beta\Big).
 		\end{aligned}
 		\label{L21norm_gra}
 	\end{equation}
 	We can find the unique solution to problem (\ref{L2,1_norm}) by setting (\ref{L21norm_gra}) to zero as follows: 
 	\begin{equation}
 		\begin{aligned}
 			\dot{\widetilde{\mathbf{X}}}_{\cdot n}&=\dot{\mathbf{Y}}_{\cdot n}-4\beta\frac{\dot{\mathbf{Y}}_{\cdot n}}{\|\dot{\mathbf{Y}}_{\cdot n}\|_{2}} \: \Big(\|\dot{\mathbf{Y}}_{\cdot n}\|_{2}>4\beta\Big)\\
 			&=\frac{\dot{\mathbf{Y}}_{\cdot n}}{\|\dot{\mathbf{Y}}_{\cdot n}\|_{2}}\Big(\|\dot{\mathbf{Y}}_{\cdot n}\|_{2}-4\beta\Big) \: \Big(\|\dot{\mathbf{Y}}_{\cdot n}\|_{2}>4\beta\Big)\\
 			&=\frac{(\|\dot{\mathbf{Y}}_{\cdot n}\|_{2}-4\beta)_{+}}{\|\dot{\mathbf{Y}}_{\cdot n}\|_{2}}\dot{\mathbf{Y}}_{\cdot n},
 		\end{aligned}
 	\end{equation}
 	where $(y)_{+}=\text{max}\{y,0\}$. 
 \end{proof}

  \section{Proof of Theorem \ref{L2,1mp2}}
   \label{appendixB}
  \begin{proof}
  	In line with (\ref{l21_REWI}), the optimal solution $\dot{\mathbf{X}}_{\text{opt}}$ of the problem in (\ref{L2,1_norm_wei}) is given by
  	\begin{equation}
  		\dot{\mathbf{X}}_{\text{opt}}^{m}=\mathop{\text{min}}\limits_{\dot{\mathbf{X}}_{\text{opt}}^{m}}\frac{\omega_m}{\mu}\|\dot{\mathbf{X}}^{m}\|_{\ast}+\frac{1}{2}\|\dot{\mathbf{X}}^{m}-\dot{\mathbf{Y}}^{m}\|_{F}^{2},   \quad (m=1, \dots, M)
  		\label{L2,1_norm_wei_rew}
  	\end{equation}
  	where $\dot{\mathbf{X}}_{\text{opt}}=\sum_{m=1}^{M}\dot{\mathbf{X}}_{\text{opt}}^{m}$. 
  	Assume that $\dot{\mathbf{Y}}^{m}=\dot{\mathbf{U}}{\mathbf{\Sigma}}\dot{\mathbf{V}}^{H}$ is the QSVD of $\dot{\mathbf{Y}}^{m}$. We represent $\dot{\mathbf{U}}$ and $\dot{\mathbf{V}}$ as two partitioned matrices:
  	$\dot{\mathbf{U}}=[\dot{\mathbf{u}}_1, \dots, \dot{\mathbf{u}}_M]$, and $\dot{\mathbf{V}}=[\dot{\mathbf{v}}_1, \dots, \dot{\mathbf{v}}_M]$, where $\dot{\mathbf{u}}_m$ and $\dot{\mathbf{v}}_m \in \mathbb{H}^{M}$ ($m=1, \dots, M$). Based on Lemma \ref{QSVT}, we can give the closed-form solution of (\ref{L2,1_norm_wei_rew}) as follows:
  	\begin{equation}
  		\begin{aligned}
  			\dot{\mathbf{X}}_{\text{opt}}^{m}&=\dot{\mathbf{U}}\mathit{S}_{\frac{\omega_m}{\mu}}(\Sigma)\dot{\mathbf{V}}^{H}\\
  			&=(\|\dot{\mathbf{Y}}^{m}\|_{F}-\frac{\omega_m}{\mu})_{+}\dot{\mathbf{u}}_1\dot{\mathbf{v}}_1^{H}\\
  			&=\frac{(\|\dot{\mathbf{Y}}^{m}\|_{F}-\frac{\omega_m}{\mu})_{+}}{\|\dot{\mathbf{Y}}^{m}\|_{F}}\dot{\mathbf{Y}}^{m},
  		\end{aligned}
  		\label{wei_rew_SO}
  	\end{equation}
  	where $\|\dot{\mathbf{Y}}^{m}\|_{F}=\sigma_{m}$ is the only singular value of $\dot{\mathbf{Y}}^{m}$.
  	Based on the above discusses, the optimal solution $\dot{\mathbf{X}}_{\text{opt}}$ is given by
  	\begin{equation}
  		\dot{\mathbf{X}}_{\text{opt}}=\dot{\mathbf{Y}}\mathbf{A},
  	\end{equation}
  	where $\mathbf{A}=\text{diag}(a_1,\dots,a_M)$, and
  	\begin{equation}
  		a_m=\frac{({\sigma_{m}-\frac{\omega_{m}}{\mu}})_{+}}{\sigma_{m}},  \quad (m=1, \dots, M)
  	\end{equation}
  \end{proof}

  \section{Proof of Theorem \ref{L2,1cov}}
   \label{appendixC}
  	\begin{proof}
  	The IRQLNM-QQR algorithm produces a sequence of quaternion matrices $\{\dot{\mathbf{D}}^{\tau}\}$ that can converge to a diagonal matrix $\dot{\mathbf{D}}$ satisfying (\ref{DTAUCOV}). As a result, the $\dot{\hat{\mathbf{D}}}$ in (\ref{QLNM_QQR_wei_UPD}) also converge to a diagonal matrix $\dot{{\mathbf{T}}}$. Therefore, we can reformulate the problem in (\ref{QLNM_QQR_wei_UPD}) as follows:
  	\begin{equation}
  		\mathop{\text{min}}\limits_{\dot{\mathbf{D}}}\sum_{l=1}^{r}\partial g(\|\dot{\mathbf{D}}^{l}\|_{\ast})\|\dot{\mathbf{D}}^{l}\|_{\ast}+\frac{\mu^{\tau}}{2}\|\dot{\mathbf{D}}-\dot{{\mathbf{T}}}\|_{F}^{2},   
  		\label{IRQLNM_QQR_wei_UPD1}
  	\end{equation}
  	Because $\|\dot{\mathbf{D}}^{l}\|_{\ast}=\sigma_l(\dot{\mathbf{D}})$, we can rewrite the above problem as follows:
  	\begin{equation}
  		\mathop{\text{min}}\limits_{\dot{\mathbf{D}}}\sum_{l=1}^{r}\partial g(\|\dot{\mathbf{D}}^{l}\|_{\ast})\|\sigma_l(\dot{\mathbf{D}})+\frac{\mu^{\tau}}{2}\|\dot{\mathbf{D}}-\dot{{\mathbf{T}}}\|_{F}^{2},.  
  		\label{IRQLNM_QQR_wei_UPD2}
  	\end{equation}
  	Therefore, based on Lemma \ref{LRQA-W}, we can solve the problem (\ref{IRQLNM_QQR_wei_UPD2}).
  \end{proof}
  \bibliographystyle{elsarticle-num} 
  \bibliography{sample}

\end{document}